\documentclass[10pt]{article}

\usepackage[left=2.2cm,right=2.2cm,top=2.5cm,bottom=2.5cm]{geometry}

\usepackage{booktabs}
\usepackage{multirow}

\usepackage{amsmath}
\usepackage{amssymb}
\usepackage{url}
\usepackage{subfigure}
\usepackage{graphicx}
\usepackage{amsfonts}
\usepackage{algorithm}
\usepackage{algpseudocode}
\usepackage{algorithmicx}
\usepackage{amsthm}

\usepackage{natbib}
\setcitestyle{authoryear,open={[},close={]}}
\usepackage{wrapfig}
\usepackage{color}
\usepackage{epstopdf}
\usepackage[titletoc,toc,title]{appendix}
\usepackage{enumitem}
\usepackage{authblk}
\usepackage{multirow}
\usepackage{longtable}

\newcommand{\E}{\mathbb{E}}
\newcommand{\R}{\mathbb{R}}
\newcommand{\normal}{\mathcal{N}}
\newcommand{\inr}[2]{\langle #1, #2 \rangle}

\DeclareMathOperator{\spn}{span}
\DeclareMathOperator{\polylog}{polylog}
\DeclareMathOperator{\dist}{dist}

\newtheorem{theorem}{Theorem}
\newtheorem{lemma}[theorem]{Lemma}

\title{The Search Problem in Mixture Models}

\author[1]{Avik Ray \thanks{avik@utexas.edu}}
\author[2]{Joe Neeman \thanks{neeman@iam.uni-bonn.de}}
\author[1]{Sujay Sanghavi \thanks{sanghavi@mail.utexas.edu}}
\author[1]{Sanjay Shakkottai \thanks{shakkott@austin.utexas.edu}}
\affil[1]{{Department of Electrical and Computer Engineering} \authorcr
{University of Texas, Austin} \authorcr
{Austin, TX 78712}}
\affil[2]{{Department of Mathematics} \authorcr
{Rheinische Friedrich-Wilhelms-Universit\"at Bonn} \authorcr
{D-53115 Bonn, Germany}}

\begin{document}

\maketitle

\begin{abstract}
We consider the task of learning the parameters of a {\em single} component of a mixture model, for the case when we are given {\em side information} about that component; we call this the ``search problem" in mixture models. We would like to solve this with computational and sample complexity lower than solving the overall original problem, where one learns parameters of all components. 

Our main contributions are the development of a simple but general model for the notion of side information, and a corresponding simple matrix-based algorithm for solving the search problem in this general setting. We then specialize this model and algorithm to four common scenarios: Gaussian mixture models, LDA topic models, subspace clustering, and mixed linear regression. For each one of these we show that if (and only if) the side information is informative, we obtain parameter estimates with greater accuracy, and also improved computation complexity than existing moment based mixture model algorithms (e.g. tensor methods). We also illustrate several natural ways one can obtain such side information, for specific problem instances. Our experiments on real datasets (NY Times, Yelp, BSDS500) further demonstrate the practicality of our algorithms showing significant improvement in runtime and accuracy.

\end{abstract}

\section{Introduction}

Mixture models denote the statistical setting where observed samples can come from one of several distinct underlying populations -- each typically with its own probability distribution -- but are not labeled as separate in the data presented. They have been used to model a wide variety of phenomena, and have seen great success in practice, going back as far as~\cite{Pearson:1894}. In this paper we consider (what we call) the {\bf search problem} in the mixture model setting: given some {\em special side information} about one of the mixture components, is it possible to efficiently learn the parameters of that component only?
Given that there are known methods for learning the entire set of
parameters of various mixture models, ``efficient'' here means
more efficient (statistically and/or computationally) than existing
methods for learning all the parameters.

As an example, we consider the ``latent Dirichlet allocation'' model
for document generation. In this model, ``underlying population''
means the set of topics in a document, which determines the
frequencies of different words in the document.
``Side information'' could be a word that is more common in the topic
of interest than it is in any other topic: for example, the word
``semi-supervised'' might work if the topic of interest is machine learning.

Side information could also consist of a small number of labelled
examples. We might have a small collection of documents about
machine learning and also a much larger corpus that includes documents
from many topics. Our methods will allow us to leverage the large,
unlabelled corpus to obtain good estimates for word frequencies in
machine learning articles -- and these estimates will be much better
than anything that could be learned from the small labelled sample.

{\bf Main contributions:}
We propose a general setting for side information in mixture models,
and show how to solve the search problem by estimating certain
matrices of moments. We prove error bounds on the resulting estimates;
our rates have a sharp dependence on the sample size (although they are possibly
not sharp in the other parameters).

We then specialize our approach to four popular families of mixture models:
Gaussian mixture models with spherical covariances, latent Dirichlet
allocation for topic models, mixed linear regression, and subspace
clustering. We give concrete algorithms for these four families. Our results also include new moment derivations for mixed linear regression and subspace clustering models.

Finally, we simulate our algorithm on both real and synthetic data sets for the Gaussian mixture model, topic model, and subspace clustering applications. For synthetic data set we compare its performance to the tensor decomposition methods discussed by~\cite{AGHKT:14} in both GMM and LDA models, and k-means for subspace clustering. We show that our methods outperform the baseline when the side information is informative. We also demonstrate the practical applicability of our algorithms on three real data sets -- the NY Times data set of news articles, Yelp data set of business reviews, and BSDS500 data set of images. In the first two text corpus, we show our algorithm recovers more coherent topics than topic modeling algorithm by~\cite{AroGeHalMimMoi:12}. In the BSDS500 data set, we demonstrate how our algorithm can be used for parallel image segmentation. In all three cases, our algorithm also exhibits significant computational gains over competing unsupervised and semi-supervised algorithms.

\subsection{Related Work}
There is a vast literature on mixture models; too much to even summarize
here. We will therefore focus this section on two more closely related
areas: method of moments estimators for mixture models, and learning
with side information.

{\bf Mixture models and method of moments:}
A common method for learning mixture models is the EM
algorithm of \cite{DeLaRu:77}, which outputs a complete set
of model parameters. However, EM may converge slowly (or
not at all) [\citealt{RednerWalker:84}];
this weakness of EM has spurred a resurgence in method-of-moments estimators
for mixture models. Although these methods go back to the pioneering work
of \cite{Pearson:1894} on Gaussian mixture models, the last several years have seen
important advances. \cite{MoitraValiant:10}, and
\cite{HardtPrice:14}
showed that Gaussian mixture models with two components can be learned in polynomial time.
\cite{HsuKakade:13} considered mixtures of more Gaussians, but constrained to have spherical
covariances. They gave a method based on third-order tensor decompositions, which was later generalized to other models in~\cite{AGHKT:14}.

{\bf Learning with side information:}
As has been observed many times, often in practice one has access to a set
of data that is somewhat richer than standard models of data in learning
theory. The term \emph{side information} is used as a catch-all for extra data that doesn't fit into pre-existing models; as such, the literature contains many
incomparable models of side information.

\cite{XJRN:02} and \cite{YaJiJa:10} took unsupervised clustering as their starting point. For them, side information
arrived as pairs of points that were known to belong to the same cluster; they showed
how this extra information could substantially improve the performance of the $k$-means algorithm.

\cite{KuuselaOcone:04} developed a framework for side information in the PAC learning model, in which extra samples with a particular
dependence on the original samples could sometimes give a substantial benefit.

Many different types of metadata have been proposed for the
\emph{latent Dirichlet allocation} (LDA) model of document generation. 
\cite{BleiMcAuliffe:08} introduced the \emph{supervised LDA} model,
in which each document comes with an additional response variable from a generalized
linear model.
On the other hand \cite{RCSS:04}
proposed the \emph{author-topic model}, in which the metadata (author names)
affects the distribution of the documents themselves.
From a more experimental point of view, \cite{LuZhai:08} used long,
detailed product reviews as side information for categorizing short snippets
and blog entries.

The notion of \emph{semi-supervised learning} (see the book by \cite{ChScZi:06})
is also related to our framework of side information. In semi-supervised learning, the
learner has access to a small number of labelled examples and a large number of unlabelled examples. This setting is useful for us too, although
our general method does not strictly require data of this form.

\section{Basic Idea and Algorithm}
\label{sec:basic}

We now first briefly describe the basic mixture model setting, and then describe our method. These descriptions cover several popular specific examples for mixture models, and we detail the application to each of them in
Section~\ref{sec:specific_models}.

{\bf Setting:}
We are interested in the standard statistical setting of (parametric) mixture models: that is, samples are drawn i.i.d. from a distribution $f$ given by
\[
f(x) ~ = ~ \sum_{i=1}^k \, \alpha_i \, g(x;\mu_i).
\]
Here $g$ corresponds to a known parametric class of distributions, and $k$ is the number of mixture components. The corresponding parameter vectors are $\mu_1,\ldots,\mu_k$, and their mixture weights / probabilities are $\alpha_1,\ldots,\alpha_k$. So, for example, in the case of the standard (spherical) Gaussian mixture model, $g(x;\mu_i)$ is the Gaussian pdf $\mathcal{N}(\mu_i,I)$. Thus each sample can be considered to be drawn by first selecting a mixture component $\mu_i$ with probability $\alpha_i$, and then drawing the sample $x$ according to $g(x;\mu_i)$. We assume all the $\mu_i$'s are {\em linearly independent}. This is a common assumption for learning mixture models using spectral methods.

{\bf Search problem:} The standard parameter estimation problem is to find all the $\mu_i$ vectors given samples. In this paper we are interested in the search problem: we are given {\em side information} about one of the vectors -- say $\mu_1$, without loss of generality -- and we would like to recover {\em only} $\mu_1$. Of course, we would like to do this with sample and computational complexity lower than what would be required to estimate all parameter vectors (i.e., lower complexity than the standard case). 

{\bf Side information:} Our general procedure requires the following model for side information: we assume that we have access to a vector $v$ such that the inner product with the parameter vector $\mu_1$ -- the special one we are searching for -- is higher than the inner product with any of the other $\mu_i$; i.e. there exists $\delta>0$ such that;
\[
  \inr{\mu_1}{v} \geq ~ (1+\delta) \inr{\mu_i}{v} \quad \text{for all $i\neq 1$}
\]
Section~\ref{sec:specific_models} shows how to obtain such side information in some specific models of interest: spherical Gaussian mixture models, mixed linear regression, subspace clustering and the LDA topic model.

We remark that it's also possible (and perhaps more intuitive in some situations) to ask for side information satisfying $|\inr{\mu_1}{v}| \geq (1 + \delta) |\inr{\mu_i}{v}|$. However, our assumption above is slightly weaker, since for any $v$ satisfying the latter assumption, either $v$ or $-v$ satisfies the former assumption. Later, we show the above condition is sufficient for uniquely identifying the required parameter $\mu_1$ (but it may not be necessary). We refer side information vector $v$ as {\em informative} about $\mu_1$ if it satisfies the above condition.

\subsection{General Procedure}
\label{sec:mAB}

The main idea behind method of moments is to use samples to estimate certain moments of the distribution $f(x),$ using which we can recover the parameters of interest. For many mixture models (including the four common examples we detail), it is possible to easily and directly estimate using first and second order moments, given sufficient samples, the vector
\begin{equation}\label{eq:m}
 m := \sum_{i=1}^k \alpha_i \mu_i.
\end{equation}
and the matrix
\begin{equation}\label{eq:A}
 A := \sum_{i=1}^k \alpha_i \mu_i \mu_i^T.
\end{equation}
For example, in many models the estimate of vector $m$ is simply the sample mean, and matrix $A$ can be derived from the sample covariance matrix. The exact procedure for estimating $m$ and $A$ varies according to the particular parametric model $g$. The fact that $m$ and $A$ (and also higher-order tensors) can be estimated from samples is well known for many models, see~\cite{AGHKT:14} for a treatment of several different models,
and for other pointers to the literature.

Typically, all mixture model components cannot be identified from just the first and second order moments (or $m$ and $A$). It is often necessary to compute even higher order moment terms. In our search problem, given the side information, {\bf we develop}  procedures to estimate an alternative matrix $B$, using higher order moments, given by
\begin{equation}\label{eq:B}
  B := \sum_{i=1}^k \alpha_i \inr{\mu_i}{v} \mu_i \mu_i^T
\end{equation}
Again, the exact procedure for estimating $B$ from samples depends on the particular parametric model $g$. 

For this section, we assume we are able to estimate $A,B,m$ to within some accuracy. We will use the notation $\hat A, \hat B, \hat m$ to denote these finite sample estimates of $A,B,m$ respectively, and $n$ denotes the number of samples used to compute these estimates. With this in hand, we outline two general procedures for estimating $\mu_1$ (i.e.\ the component that we are interested in).
The first procedure is based on a whitening step, much like the one
that is used in the spectral algorithms
in~\cite{HsuKakade:13,AnaFosHsuKak:12LDA}, and tensor decomposition
methods of~\cite{AGHKT:14} (please see remarks in
Section~\ref{sec:specific_models} for the differences for specific models).
The second procedure uses a line search instead, and may be
computationally favorable when $k$ is large, because it avoids
the need to invert a $k \times k$ matrix. Both Algorithms \ref{alg:meta-whitening} and \ref{alg:meta-cancel} take as input the estimates $\hat A, \hat B, \hat m$ (where $\hat B$ is constructed using side information vector $v$) and they output estimates of the first mixture component $\hat \mu_1,$ and also the proportion of the first component $\hat \alpha_1.$

\subsubsection{The Whitening Method} \label{sec:whitening}

\begin{algorithm}[ht]
 \caption{Extracting a mixture component from side information:
 the whitening method.}
 \label{alg:meta-whitening}
\begin{algorithmic}[1]
 \Require $\hat A, \hat B, \hat m$
 \Ensure $\hat \mu_1, \hat \alpha_1$
\State let $\{\sigma_j, v_j\}$ be the singular values and singular vectors of $\hat A$, in non-increasing order\;
\State let $V$ be the $d \times k$ matrix whose $j$th column is $v_j$\;
\State let $D$ be the $k \times k$ diagonal matrix with $D_{jj} = \sigma_j$\;
\State let $u$ be the largest eigenvector of $D^{-1/2} V^T \hat B V D^{-1/2}$
\State let $w = V D^{1/2} u$\;
\State let $E$ be the span of $\{V D^{1/2} v: v \perp u\}$\;
\State write $V V^T \hat m$ (uniquely) as $a w + y$, where $y \in E$\;
\State return $w/a$ and $a^2$\;
\end{algorithmic}
\end{algorithm}

Our main result about Algorithm~\ref{alg:meta-whitening}
is that if $\hat A$ and $\hat B$ are good estimates
of $A$ and $B$ then Algorithm~\ref{alg:meta-whitening}
outputs good estimates for $\mu_1$ and $\alpha_1$.
In order to interpret Theorem~\ref{thm:meta-whitening} as
an error rate, note that if all parameters but $\epsilon$ are fixed
then the error is $O(\epsilon)$. Since standard concentration results
yield $\epsilon = O(n^{-1/2})$, where $n$ is the number of samples; our error rate in terms of $n$ is also
$O(n^{-1/2})$. This rate is sharp, since it is also the rate for estimating
the mean of a single Gaussian vector (i.e. a GMM with only one component).

\begin{theorem}\label{thm:meta-whitening}
 Suppose that $\mu_1, \dots, \mu_k$ are linearly independent, and that $\hat{A}$ is positive semi-definite. Also suppose that $\inr{\mu_1}{v} \ge (1+\delta)\inr{\mu_i}{v}$ for
 all $i \ne 1$.
 Assume that
 \[\max\{\|A - \hat A\|, \|B - \hat B\|, \|m - \hat m\|\} \le \epsilon < \sigma_k(A)/4,\]
 and that the right hand side of~\eqref{eq:alpha-bound} is at most $\alpha_1$. Then
 \begin{align}
   \|\mu_1 - \hat \mu_1\| &\le C R |\alpha_1^{-1/2} - \hat \alpha_1^{-1/2}|
  + C \frac{\sqrt{\sigma_1(A)}}{\sqrt{\alpha_1}} \eta \quad \text{, and} \notag \\
  |\alpha_1 - \hat \alpha_1| &\le \frac{C\sqrt{\alpha_1}(\alpha_1 R + \eta)}{\sigma_k(A)} \left(\eta + R \frac{\epsilon}{\sigma_k(A)} + \epsilon\right) \label{eq:alpha-bound}
  \end{align}
 where
 $
 \eta = \frac{\epsilon \sigma_1}{\delta \sigma_k^{5/2}}$,
  $R = \max_i \|\mu_i\|,$
 $\sigma_1(A) \ge \cdots \ge \sigma_k(A) > 0$ are the non-zero singular values of $A = \sum_i \alpha_i \mu_i \mu_i^T$, and $C$ is a universal constant.
\end{theorem}

Our error bounds are somewhat complicated, and depend on many different parameters, so let us elaborate on them slightly. First of all, the dependence on $\sigma_1(A)$ and $\sigma_k(A)$ is of the order $\|\mu_1 - \hat \mu_1\| \lesssim \sigma_1(A)^{3/2} / \sigma_k(A)^{5/2}$, which is probably an artifact of the analysis, and not the true behavior of the algorithm. On the other hand, our dependence on $\epsilon$ is optimal: we have $|\alpha_1 - \hat \alpha_1| \lesssim \epsilon$ and $\|\mu_1 - \hat \mu_1\| \lesssim \epsilon$.
Note also that our bound has no explicit dependence on $k$; this feature comes from the fact that our method is targeted at a single mixture component. By comparison, other methods typically give bounds in which the \emph{averaged} per-mixture-component error does not depend on $k$. In terms of dependence on $k$, therefore, our bounds are better than previous bounds if there is only one component of interest.

Finally, let us remark on the assumption that the right hand side of~\eqref{eq:alpha-bound} is at most $\alpha_1$. This amounts to an assumption that $\epsilon$ is sufficiently small compared to all the other parameters. Without this assumption, the bound in~\eqref{eq:alpha-bound} would not be very interesting, since $|\alpha_1 - \hat \alpha_1| \le \alpha_1$ is too weak to give useful information about $\hat \alpha_1$ (it could even be zero).

We defer the actual analysis of Algorithm~\ref{alg:meta-whitening} to the appendix, but we will motivate the algorithm and give the basic idea of the proof by showing that if $\hat A, \hat B$, and $\hat m$ are equal to $A, B$ and $m$ respectively then Algorithm~\ref{alg:meta-whitening} outputs $\mu_1$ and $\alpha_1$ exactly.

\begin{lemma}\label{lem:whitening-infinite}
  Let $m$, $A$, and $B$ be defined by in~\eqref{eq:m}, \eqref{eq:A}, and~\eqref{eq:B},
where $\mu_1, \dots, \mu_k$ are linearly independent.
If $\inr{\mu_1}{v} > \inr{\mu_i}{v}$ for all $i \ne 1$
and we apply Algorithm~\ref{alg:meta-whitening}
to $A$, $B$, and $m$, then it returns $\mu_1$ and $\alpha_1$.
\end{lemma}
\begin{proof}
Let $V$ and $D$ be as defined in Algorithm~\ref{alg:meta-whitening}.
Since $A$ has rank $k$,
\[
\sum_{i=1}^k \alpha_i D^{-1/2} V^T \mu_i \mu_i^T V D^{-1/2} = D^{-1/2} V^T A V D^{-1/2} = I_k.
\]
Defining $u_i := \sqrt{\alpha_i} D^{-1/2} V^T \mu_i$, we have
$\sum_i u_i u_i^T = I_k$, which implies that the $u_i$
are orthonormal in $\R^k$.
Now,
\begin{equation*}
 D^{-1/2} V^T B V D^{-1/2}
= \sum_{i=1}^k \alpha_i \inr{\mu_i}{v} D^{-1/2} V^T \mu_i \mu_i^T V D^{-1/2} = \sum_{i=1}^k \inr{\mu_i}{v} u_i u_i^T.
\end{equation*}
Since $\inr{\mu_1}{v}$ was assumed to be larger than all other $\inr{\mu_i}{v}$, it follows
that $u_1$ is the largest eigenvector of
$D^{-1/2} V^T B V D^{-1/2}$. Now, if $w = V D^{1/2} u_1$ then $w = \sqrt{\alpha_1} \mu_1$.

Now, note that since the $\mu_i$ are linearly independent, there is a unique way to write
$m = V V^T m = \sum_i \alpha_i \mu_1$ as $a w + y$, where $y$ belongs to the span of $\{\mu_2, \dots, \mu_k\}$
(which is the same as the span of $\{V D^{1/2} u_i: i \ge 2\}$. Moreover,
the unique choice of $a$ that allows this representation must satisfy $aw = \alpha_1 \mu_1$, which implies
that $a = \sqrt{\alpha_1}$. Therefore, $w/a = \mu_1$ and $a^2 = \alpha_1$.
\end{proof}

The proof of Lemma~\ref{lem:whitening-infinite} is crucial to understanding the algorithm, and also the broader message of this article: if we can get hold of two different normalizations of something, then we can learn something about it. In the proof of Lemma~\ref{lem:whitening-infinite}, this happens twice: first, we use the fact that $A$ and $B$ contain the same components (but with differing normalizations) to extract the span of a single component of interest. The differing normalization is crucial, because $A$ by itself does not uniquely determine the set $\{\mu_1, \dots, \mu_k\}$, much less single out a specific component of interest.

In the second step of Lemma~\ref{lem:whitening-infinite}, we know $\sqrt{\alpha_1} \mu_1$, which is not enough to determine either $\alpha_1$ or $\mu_1$. However, we also have access to $m$, which involves a contribution of $\alpha_1 \mu_1$. Exploiting the difference between these two normalizations, we recover both $\alpha_1$ and $\mu_1$.

\subsubsection{The Cancellation Method} \label{sec:cancellation}

Our second method avoids the matrix inversion in
Algorithm~\ref{alg:meta-whitening}, preferring a line search instead.

\begin{algorithm}[ht]
 \caption{Extracting a mixture component from side information:
 the cancellation method.}
\label{alg:meta-cancel}
\noindent\begin{minipage}{\textwidth}
\renewcommand\footnoterule{}        
\begin{algorithmic}[1]
\Require $\hat A, \hat B, \hat m$
\Ensure $\hat \mu_1, \hat \alpha_1$
\State let $\widehat V$ be the $d \times k$ matrix of $k$ largest eigenvectors of $\hat A$;
\State search over $\lambda$ to find the largest $\lambda = \lambda^*$ such that $\widehat{V}\widehat{V}^T(\widehat{A} - \lambda \widehat{B})\widehat{V}\widehat{V}^T$ is PSD;
\State let $\widehat{Z}_{\lambda^*} = \widehat{A} - \lambda^* \widehat{B},$ and let $\{v_2 , \hdots, v_k\}$ be the top $k-1$ singular vectors of $\widehat{Z}_{\lambda^*}$\;
\State let $V_{1:(k-1)}$ be the $d \times (k-1)$ matrix with columns $\{v_2 , \hdots, v_k\}$\;
\State let $x_1 = \hat{m} - V_{1:(k-1)}V_{1:(k-1)}^T \hat{m}$\;
\State let $v_1 = x_1 / \|x_1\|$\;
\State compute $c_i = v_1^T \widehat{A} v_i$ for $i=1$ to $k$\;
\State let $a_i = c_i/\|x_1\|$ for $i=1$ to $k$\;
\State return $\hat{\mu}_1 = \sum_{i=1}^k a_i v_i$ and $\hat{\alpha}_1 = c_1/a_1^2$\;
\end{algorithmic}
\end{minipage}
\end{algorithm}  

In the above Algorithm \ref{alg:meta-cancel}, we assume $\langle \mu_1,v \rangle>0.$ When this is not the case and $B$ is a negative semi-definite matrix, we simply have to change the line search step to search for the smallest $\lambda<0$ such that $\widehat{V}\widehat{V}^T(\widehat{A} - \lambda \widehat{B})\widehat{V}\widehat{V}^T$ is PSD. Theorem \ref{thm:meta-cancel} shows that with $m,A,B$ estimated up to $O(\epsilon)$ error, the parameter estimation error in Algorithm~\ref{alg:meta-cancel} is also bounded as $O(\epsilon).$

\begin{theorem}\label{thm:meta-cancel}
  Suppose $\{\mu_1 , \hdots , \mu_k\}$
  are linearly independent and $v$ satisfies $\inr{\mu_1}{v} \ge (1+\delta)\inr{\mu_i}{v}$ for
 all $i \ne 1$. Suppose that $\max
  \{\|\widehat{A}-A\|,\|\widehat{B}-B\|, \|\hat{m}-m\|\} < \epsilon,$ and $\lambda_1 := 1/\langle \mu_1, v \rangle.$ Then Algorithm
  \ref{alg:meta-cancel} returns $\hat{\mu}_1 , \hat{\alpha}_1$ with
\begin{eqnarray*}
\|\hat{\mu}_1 - \mu_1\| &<& \frac{C\epsilon}{\alpha_1^2 a_1^2} \left( \sigma_1(A)  \left( 1 + \frac{\alpha_1 a_1}{\sigma_{k-1}(Z_{\lambda_1})}\right)+ \frac{\sigma_1(A) \eta_3 R}{\sigma_{k-1}(Z_{\lambda_1})}\right) \\
|\hat{\alpha}_1 - \alpha_1| &<& \frac{C \sigma_1(A) \epsilon}{\alpha_1 a_1^3} \left( \eta_1 + \frac{\eta_2 R \eta_3}{\sigma_{k-1}(Z_{\lambda_1})}\right)
\end{eqnarray*}
where $\eta_1 := \max\{\alpha_1 a_1 (2 a_1 + 1),20\},$ $\eta_2 := \max\{\alpha_1 a_1^2,  10\},$ $\eta_3 = \max \left\{1,\lambda_1,\sigma_1(B)\right\},$ $R = \max \|\mu_i\|,$ $a_1 = \|\mu_1 - \prod_{\mathcal{V}} \mu_1\|,$ where $\mathcal{V}=\text{span}\{\mu_2 , \hdots , \mu_k\},$ and $C$ is an universal constant.
\end{theorem}

Again, we will defer the actual analysis to the appendix, and instead
show that Algorithm~\ref{alg:meta-cancel} returns the exact answer
when fed exact initial data. We will do this in two lemmas:
Lemmas~\ref{lem:z_psd} and \ref{lem:coord_algo_exact_stat}.

\begin{lemma}
Let $Z = \sum_{i=1}^k \gamma_i \mu_i \mu_i^T$ where $\{\mu_1, \hdots , \mu_k\}$ are linearly independent, $\mu_i \in \mathbb{R}^d, \gamma_i \in \mathbb{R}$ and $d>k$. If $\gamma_1 < 0$ and $\gamma_i >0$ for all $i \neq 1$ then $Z$ is not positive semi-definite. \label{lem:z_psd}
\end{lemma}
\begin{proof}
  Let $\Pi$ denote the projection onto the orthogonal complement of
  $\text{span} \{\mu_2, \hdots , \mu_k\}$. Let $x = \Pi \mu_1$,
  and note that $\inr{x}{\mu_1} > 0$ but $\inr{x}{\mu_i} = 0$
  for all $i \ne 1$. Hence, $x^T Z x = \gamma_1 \inr{x}{\mu_1}^2 < 0$
  and so $Z$ is not positive semi-definite.
\end{proof}

\begin{lemma}
Let $m$, $A$, and $B$ be defined by in~\eqref{eq:m}, \eqref{eq:A}, and~\eqref{eq:B},
where $\mu_1, \dots, \mu_k$ are linearly independent.
If $\inr{\mu_1}{v} > \inr{\mu_i}{v}$ for all $i \ne 1$
and we apply Algorithm~\ref{alg:meta-cancel}
to $A$, $B$, and $m$, then it returns $\mu_1$ and $\alpha_1$.
\label{lem:coord_algo_exact_stat}
\end{lemma}
\begin{proof}
  Define $w_i = \inr{\mu_i}{v}$ and let $\gamma_i = \alpha_i(1-\lambda w_i)$,
  so that
  \[
    Z_\lambda = A - \lambda B = \sum_{i=1}^k \gamma_i \mu_i \mu_i^T.
  \]
Note that, in our case where $\widehat A = A,$ and $\widehat B = B,$ columns of $\widehat V$ simply form a common orthonormal bases of the row/column space of both matrices $A,B.$ Therefore the matrix $\widehat{V}\widehat{V}^T(A - \lambda B)\widehat{V}\widehat{V}^T = A - \lambda B = Z_{\lambda}.$
  Now for $\lambda >
  \frac{1}{w_1},$ $\gamma_1 <0$ and for all $\lambda \leq \frac{1}{w_1},$
  $\gamma_i \geq 0$ for all $i$ since $w_1 > w_i,$ for every $i \neq 1.$
  By Lemma \ref{lem:z_psd}, $\lambda^* = \frac{1}{w_1}$ is the largest
  $\lambda$ such that $Z_\lambda$ is PSD; hence,
  \[
  Z_{\lambda^*} = \sum_{i=2}^k \alpha_i (1-\lambda^* w_i) \mu_i \mu_i^T.
  \]
  From Lemma \ref{lem:cancel_exact_supporting} in Appendix \ref{sec:cancel_useful_proof} it follows that $k-1$ singular vectors $\{v_2 , \hdots , v_k\}$
  of $Z_{\lambda^*}$ form a basis of the subspace $\mathcal{V} =
  \text{span} \{\mu_2 , \hdots , \mu_k\}.$ Let $\mathcal{V}_{\perp}$ be the
  perpendicular space of $\mathcal{V}$, and write
  $\Pi = I - V_{1:(k-1)} V_{1:(k-1)}^T$ for the orthogonal projection
  onto $\mathcal{V}_\perp$.
  Since $\Pi \mu_i = 0$ for $i \ne 1$, we have
  $x_1 = \Pi m = \alpha \Pi \mu_1$.
  
  Now define $b_1, \dots, b_k$ by $\mu_1 = \sum_{i=1}^k b_i v_i$.
  In order to prove that the algorithm returns $\mu_1$ correctly,
  we need to show that $b_i = a_i := c_i / \|x_1\|$.
  Indeed,
  \[
    c_i := v_1^T A v_i = \sum_{j=1}^k\alpha_j v_1^T \mu_j \mu_j^T v_i =
    \alpha_1 b_1 b_i,
  \]
  since $v_1^T \mu_j = 0$ for $j \ne 1$.
  On the other hand, $\|x_1\| = \alpha \|\Pi \mu_1\| = \alpha b_1$,
  and so $b_i = a_i$, as claimed. Moreover,
  $\hat{\alpha}_1 = \frac{c_1}{a_1^2} = \alpha_1$, as claimed.
\end{proof}

\noindent {\bf Optimization for $\lambda^*$:} The first step of Algorithm \ref{alg:meta-cancel} involves finding a smallest $\lambda^*$ such that $\widehat{Z}'_{\lambda^*} = \widehat{V}\widehat{V}^T(\widehat{A} - \lambda^* \widehat{B})\widehat{V}\widehat{V}^T$ is PSD using line search. Although $\widehat{Z}'_{\lambda}$ is a $d \times d$ matrix, this step can be performed efficiently as follows. Instead of searching for $\lambda$ directly for $\widehat{Z}'_{\lambda},$ we do this for a smaller $k \times k$ matrix $\widehat{V}^T\widehat{Z}_{\lambda}'\widehat{V}=\widehat{V}^T(\widehat{A} - \lambda^* \widehat{B})\widehat{V}.$ This optimization step using line search can be performed in just $O(k^3 \log |\lambda^*|)$ time.

\section{Specific Models}\label{sec:specific_models}

In this section we discuss how the search algorithms can be applied in four specific mixture models.

\subsection{Gaussian Mixture Model with Spherical Covariance}

{\bf The model:} Besides the mixture parameters
$\alpha_1, \dots, \alpha_k$, the Gaussian mixture model (GMM) has mean parameters
$\mu_1, \dots, \mu_k \in \R^d$ and variance parameters $\sigma_1, \dots, \sigma_k \in \R$.
The conditional densities $g(\cdot; \mu_i, \sigma_i)$ are Gaussian,
with mean $\mu_i$ and covariance $\sigma_i^2 I_d$. Explicitly,
\[
  g(x; \mu_i, \sigma_i) = \frac{1}{(2\pi\sigma_i^2)^{d/2}}
  e^{-\frac{\|x-\mu_i\|^2}{2\sigma_i^2}}.
\]

{\bf Matrices $A$ and $B$:}
We fix a vector $v \in \R^d$, with the assumption that
$\inr{v}{\mu_1} > \inr{v}{\mu_i}$ for $i \ne 1$. Recall (from
Section~\ref{sec:mAB}) that $m = \E[x] = \sum_i \alpha_i \mu_i,$ $A =
\sum_{i=1}^k \alpha_i \mu_i \mu_i^T,$ and $B = \sum_{i=1}^k \alpha_i
\inr{\mu_i}{v} \mu_i \mu_i^T.$ To compute these quantities, we first
define $\sigma^2$ to be the $(k+1)$th-largest eigenvalue of the mixture
covariance
matrix $\E[(x - m)(x - m)^T],$ and let $u$ be a corresponding
eigenvector. Then let $\widetilde{m} = \E[x(u^T (x - m))^2].$
Then it follows from moment computations (see \cite{HsuKakade:13}) that:
\begin{eqnarray*}
  A &=& \E [x x^T] - \sigma^2 I_d \\
  B &=& \E [\inr x v x x^T] - \widetilde{m} v^T - v \widetilde{m}^T -
      \inr {\widetilde{m}} v I_d,
 \end{eqnarray*}
\vspace{-.1in}

%
%
%
%
%
%
Given the samples $\{\hat{x}_i\},$ we can now empirically evaluate
these quantities (denoted by $\hat{m}, \hat{A}, \hat{B}$ respectively) by
replacing expectations above by the corresponding sample averages; for
instance we replace $\E[x
x^T]$ by $\widehat{\E}[x x^T] \stackrel{.}{=} (1/n) \sum_{j=1}^n \hat{x}_j
\hat{x}_j^T$. 


{\bf Examples of $v$:}
Assuming that $\|\mu_1\|^2 > \inr{\mu_1}{\mu_i}$ for all $i \ne 1$
-- this will be true, for example, if $\|\mu_i\|$ are all the same -- one can
find a suitable vector $v$ given a relatively small number of samples
from the first mixture component. Specifically, if
$\|\mu_1\|^2 \ge \inr{\mu_1}{\mu_i} + \delta$ and $\|\mu_i\| \le R$
for all $i \ne 1$ then standard Gaussian tail bounds imply the following:
if $v := \ell^{-1} \sum_{j=1}^\ell x_j$ where $\ell =
\smash{\Omega(R^2 \delta^{-2} \log k)}$ and $x_1, \dots, x_m$ are drawn independently from
the distribution $g(\cdot; \mu_1, \sigma_1)$ then with high probability
$v$ satisfies $\inr{v}{\mu_1} > \inr{v}{\mu_i}$ for all $i \ne 1$.
Here, ``high probability'' means probability converging to 1 as
the hidden constant in $\ell = \Omega(\cdot)$ grows. Note here that
the number of tagged samples is nowhere near sufficient to estimate
$\mu_1$ by direct averaging; indeed to do so would require the number
of samples to grow with the size of the underlying dimension.

{\bf Remarks:} We note that spectral algorithms which uses the whitening procedure has been proposed before in the context of GMM e.g. \cite{HsuKakade:13}. The primary difference between the algorithm in \cite{HsuKakade:13} and Algorithm \ref{alg:meta-whitening} is that the former, in absence of side information, takes a projection of the third order moment tensor $M_3$ on a random unit vector to obtain the second matrix, where as our matrix $B$ can be viewed as a projection of $M_3$ on the side information vector $v$. The main advantage of projecting onto $v$ is that, when we have reliable side information, this will give a good singular value separation resulting in better empirical performance. The Cancellation algorithm however is distinctly different from both and has not been studied before.

\subsection{Latent Dirichlet Allocation} \label{sec:LDA}

{\bf The model:} In the LDA model with $k$ topics and a dictionary of size $d$, the parameters $\mu_1 , \hdots, \mu_k \in \Delta_{d-1}$ are the probability distributions corresponding to each topic ($\Delta_{d-1}$ denotes the probability simplex $\{y \in \R^d: \sum_i y_i = 1, \min_i y_i \ge 0\}$). The LDA model introduced in \cite{BlNgJo:03} differs slightly from the other models as the mixture distribution cannot be expressed exactly in the parametric form in Section~\ref{sec:basic}. Instead we have a two level hierarchy as follows. Given $\bar{\alpha} = (\alpha_1 , \hdots , \alpha_k)$, we first draw a topic distribution $\theta$ from the Dirichlet($\bar{\alpha}$) distribution. Given this $\theta = (\theta_1 , \hdots , \theta_k)$ each word in the document
is drawn i.i.d.~from the distribution $\sum_{i=1}^k \theta_i \mu_i$.
However still we can compute the vector $m$ and the matrices $A,B$ as shown below. Then with an appropriate $v$ our algorithms can recover the topic distribution $\mu_1.$

{\bf Matrices $A$ and $B$:} Let $x_1$ denote the random vector with $x_1(w) = 1$ if the first word is $w,$ and $0$ otherwise. Similarly define vectors $x_2,x_3$ corresponding to the second and third word respectively, and let $\alpha_0 = \sum_{i=1}^k \alpha_i.$ Then, moment computations under the LDA distribution yields the following expressions for $(m, A, B),$ defined in (\ref{eq:m}), (\ref{eq:A}), (\ref{eq:B}):
%
%
%
\begin{align*}
m &= \alpha_0 \mathbb{E}[x_1], \ \ \ A =  \alpha_0 (\alpha_0+1) \mathbb{E}[x_1 x_2^T] - m m^T \\
B &= \frac{\alpha_0 (\alpha_0+1) (\alpha_0+2)}{2} \mathbb{E}[\langle x_3,v \rangle x_1 x_2^T]  - \frac{\alpha_0 (\alpha_0+1)}{2} \left( \langle m,v\rangle \mathbb{E}[x_1 x_2^T] + \mathbb{E}[\langle x_3,v \rangle x_1 m^T] \right. \\
& + \left. \mathbb{E}[\langle x_3 , v \rangle m x_2^T] \right) + \langle m,v\rangle m m^T. 
\end{align*}
With the given document samples, let $\hat{x}_i$ denote the normalized empirical word frequencies in the document $i.$ Then, $\hat{m} = \frac{\alpha_0}{n} \sum_{i=1}^n \hat{x}_i,$ and  $\widehat{A}, \widehat{B}$ can be immediately estimated using the above expressions by replacing expectations with sample averages.



{\bf Using labeled words to find $v$:} In order to recover the topic distribution $\mu_1$ we now require a vector $v$ which satisfies $ \langle \mu_1, v \rangle > \langle \mu_i , v\rangle$ for $i\neq 1.$ Now suppose we are given a \textit{labeled word} $\ell$ such that its occurrence probability in topic $1$ is the highest, i.e., $\mu_1(\ell) > \mu_i(\ell)$ for $i \neq 1$ (note that this does not mean $\ell$ is the most frequent word in topic $1,$ there may be words with higher occurrence probability in this topic). Then we can simply choose $v=e_\ell$ (the standard basis element with $1$ in the $\ell$-th coordinate).
For most topics of practical interest it is possible to find such labeled words. For example the word ``ball'' can be a labeled word for topic sport, ``party'' is a labeled word for topic politics and so on. However, a labeled word is merely indicative of a topic and is not exclusive to a topic (e.g. the word ``ball'' can occur in other contexts as well). In this sense, the labelled word is quite different from the ``anchor word'' described in \cite{AroGeHalMimMoi:12}.  Note however that anchor words are also labeled words (but {\em not} vice-versa) since for an anchor word $\ell,$ $\mu_1(\ell)>0$ and $\mu_i(\ell)=0$ for $i \neq 1.$ 

{\bf Using labeled documents to find $v$:} If the different topics
are not too similar, then we can estimate a suitable vector $v$ from
a small collection of documents that are mostly about the topic of
interest. For example, if $\inr{\mu_i}{\mu_j} \le \eta \|\mu_i\|\|\mu_j\|$
for all $i \ne j$, and if we observe a total of $m$ words from
some collection of documents with $\theta_1 \ge (1+\delta) (1/2 + \eta)$
then about $m = \Omega(\delta^{-2} \log k)$ words will suffice to find
a suitable vector $v$.

{\bf Remarks:} Similar to the case of GMM, a spectral algorithm using whitening procedure to estimate LDA components have been presented before in \cite{AnaFosHsuKak:12LDA}. Again the main difference with our Whitening algorithm being the fact that in \cite{AnaFosHsuKak:12LDA} the second matrix is constructed by taking a random projection of the third order moment tensor $Triples,$ and in Algorithm \ref{alg:meta-whitening} this is constructed as a projection onto $v.$ As mentioned before empirically this results is a more stable algorithm due to guaranteed singular value separation. The Cancellation algorithm has not been previously studied in LDA model.

\subsection{Mixed Regression} \label{sec:mixed_regression}

{\bf The model:} In mixed linear regression the mixture samples generated are of the form $y = \langle x,\mu_i \rangle + \xi,$ where $x \sim \mathcal{N}(0,I)$ and noise $\xi \sim \mathcal{N}(0,\sigma^2).$ As before, a sample is generated using the $i$-th linear component $\mu_i,$ with probability $\alpha_i.$ We have access to the observations $(y,x)$ but the particular $\mu_i$ and $\xi$ are unknown. Hence the conditional density $g(x, y;\mu_i,\sigma)$ is a multivariate Gaussian 
where $x \sim \normal(0, I)$, $y \sim \mathcal{N}(0,\|\mu_i\|^2+\sigma^2)$,
and $\mathrm{Cov}(x,y) = \mu_i$.

{\bf Matrices $A$ and $B$:} To compute $A$ and $B$, we consider the following moments (for more detailed derivations, see Appendix~\ref{app:moments}):
\begin{align*}
  M_{1,1} &= \mathbb{E}[yx] = \sum_{i=1}^k \alpha_i \mu_i \\
  M_{2,2} &= \mathbb{E}[y^2 x x^T] =  2 \sum_{i=1}^k \alpha_i \mu_i \mu_i^T + \sum_{i=1}^k \alpha_i (\sigma^2 +\|\mu_i\|^2 ) I \\
  M_{3,1} &= \mathbb{E}[y^3 x] = 3 \sum_{i=1}^k \alpha_i (\sigma^2 + \|\mu_i\|^2) \mu_i \\
  M_{3,3} &= \mathbb{E}[y^3 \langle x,v\rangle x x^T] = 6 \sum_{i=1}^k \alpha_i \langle \mu_i, v\rangle \mu_i \mu_i^T + \left( M_{3,1} v^T + v M_{3,1}^T + \langle M_{3,1}, v\rangle I \right) 
\end{align*}

Let $\tau^2$ be the smallest singular value of the matrix $M_{2,2}.$ Then we can compute $m,A,B$ as follows.
\begin{eqnarray*}
m &=& M_{1,1}, \ \ \ A = \frac{1}{2} (M_{2,2} - \tau^2 I) \\
B &=& \frac{1}{6} (M_{3,3} - (M_{3,1}v^T + v M_{3,1}^T + \langle M_{3,1},v\rangle I))
\end{eqnarray*}

As in the previous cases with finite samples the estimates $\hat{m},\widehat{A},\widehat{B}$ can be computed by taking their empirical expectations e.g., $\widehat{M}_{1,1} = \widehat{\E}[yx] = \frac{1}{n} \sum_{i=1}^n \hat{y}_i \hat{x}_i$ and so on, where $(\hat{y}_i,\hat{x}_i)$ denote the $i$-th sample.

{\bf Examples of $v$:} Suppose we are given a few random labeled examples from the first component. Then assuming $\|\mu_1\|^2 > \langle \mu_1,\mu_i\rangle + \delta,$ $\|\mu_i\|^2 \leq R,$ similar to the GMM case we can estimate a $v := \frac{1}{\ell}\sum_{j=1}^\ell \hat{y}_j \hat{x}_j$ using only $\ell=\Omega\left( R^4 \delta^{-2} \log k\right)$ labeled samples so that $\langle \mu_1,v\rangle > \langle \mu_i,v \rangle$ holds with high probability. 

{\bf Remarks:} Our construction of the second matrix $B$ is a consequence of some new moment results for the mixed linear regression model. We present these detailed moment derivations in Appendix \ref{app:mixed_regression}. This also results in improved sample complexity bounds over previous moment based algorithms (discussed in Section \ref{sec:comparison}).

\subsection{Subspace Clustering} \label{sec:subspace_model}

{\bf The model:} Besides the mixture parameters $\alpha_1, \dots, \alpha_k$,
the subspace clustering model has parameters $U_1, \dots, U_k \in \R^{d \times m}$
and $\sigma \in \R$, where the matrices
$U_1, \dots, U_k$ have orthonormal columns. 
The conditional distribution $g(\cdot; U_i)$ is a standard Gaussian
variable supported on the column space of $U_i$, plus independent Gaussian noise.
More precisely, we sample $y \sim \normal(0, I_d)$
and set $x = U_i U_i^T y + \xi$, where $\xi \sim \normal(0, \sigma^2 I_d)$ is independent of $y$.

{\bf Matrices $A$ and $B$:}
The subspace clustering model does not quite fit
into the basic method of Section~\ref{sec:basic}; one motivation
for presenting it is
to show that the basic ideas in Section~\ref{sec:basic} are
more flexible than they first appear.
Suppose $v \in \R^d$ satisfies $\|U_1^T v\| > \|U_i^T v\|$
for all $i \ne 1$.
We consider
\begin{eqnarray*}
 A &:=& \E[x x^T] - \sigma^2 I_d = \sum_{i=1}^k \alpha_i U_i U_i^T \\
  B &:=& \E[\inr{x}{v}^2 x x^T] - \sigma^2 v^T A v I_d - \sigma^2 \|v\|^2 A - \sigma^4 (\|v\|^2 I_d + v v^T) - 2 \sigma^2 (A v v^T + v v^T A) \\
  & & = \sum_{i=1}^k \alpha_i \|U_i^T v\|^2 U_i U_i^T  + 2 \sum_{i=1}^k \alpha_i U_i U_i^T v v^T U_i U_i^T
\end{eqnarray*}
and their empirical versions $\hat A$ and $\hat B$ (the computation giving the claimed formula for $B$ is carried out in Appendix~\ref{app:moments}). Now with these $\hat A$ and $\hat B,$ we can recover the subspace $U_1$ using Algorithm~\ref{alg:subspace}. This algorithm uses the same principle behind the whitening method in Section \ref{sec:whitening}, the key difference is that here we pick the top $m$ eigenvectors of the whitened $B$ matrix.  

\begin{algorithm}[ht]
\caption{Subspace clustering algorithm}
  \label{alg:subspace}
\begin{algorithmic}[1]  
\Require $\hat A, \hat B$
\Ensure $\hat U$
\State let $\{\sigma_j, v_j\}$ be the singular values and singular vectors of $\hat A$, in non-increasing order\;
\State let $V$ be the $d \times mk$ matrix whose $j$th column is $v_j$\;
\State let $D$ be the $mk \times mk$ diagonal matrix with $D_{jj} = \sigma_j$\;
\State let $Y=[u_1, \dots, u_{m}]$ be the matrix of $m$ largest eigenvectors of $D^{-1/2} V^T \hat B V D^{-1/2}$\;
\State let $Z = V D^{1/2} Y$\;
\State let the columns of $\hat U$ be the $m$ eigenvectors of the matrix $ZZ^T$\;
\end{algorithmic}  
\end{algorithm}

The following perturbation theorem guarantees that if the side information vector $v$ is substantially more aligned with
the subspace spanned by $U_1$ than it is with any other subspace, and the matrices $A,B$ are estimated within $\epsilon$ accuracy, then Algorithm \ref{alg:subspace} can recover the required subspace with a small error.

\begin{theorem} \label{thm:subspace_perturbation}
Suppose that $\|\hat{A} - A\| \leq \epsilon$
and $\|\hat{B} - B\| \leq \epsilon.$ Suppose that the side information vector $v$ satisfies $\|U_i v\|^2 \le (1/3 - \delta) \|U_1 v\|^2$.
  Then output $\hat U$ of
  Algorithm~\ref{alg:subspace} satisfies
  \[
    \|\hat U \hat U^T - U_1 U_1^T\| \le C \epsilon \alpha_1^{-1} \sigma_1(A)^{2} \sigma_{mk}(A)^{-2} \delta^{-1}.
  \]
\end{theorem}

We prove Theorem \ref{thm:subspace_perturbation} in Appendix \ref{app:subspace_proof}.
Note that the conditions on $v$ can be satisfied if the spaces $U_i$ satisfy a certain affinity condition and we
have a few labelled samples from $U_1$. Specifically, suppose that $\inr{u}{w} < (\frac 1{\sqrt 3} - \eta) \|u\|\|w\|$ for every
$u \in U_1$ and $w \in U_i$, $i \ne 1$. Then any $v \in U_1$ will satisfy the assumption of Theorem~\ref{thm:subspace_perturbation}.
Hence, a single labelled sample from $U_1$ (or several -- depending on $\eta$ -- noisy samples) is enough to find a
suitable $v$.

{\bf Remarks:} To the best of our knowledge Algorithm \ref{alg:subspace} is the first moment based algorithm for the subspace clustering model. The detailed moment derivations are presented in Appendix \ref{app:subspace_clustering}. Also our generative model allows samples to be noisy, hence they do not lie exactly on the subspace but close to it. Such a setting has not been considered in most subspace clustering literature.

\subsection{Comparison} \label{sec:comparison}

In this section we compare the theoretical performance of the Whitening and Cancellation algorithms with other algorithms. Both Whitening and Cancellation algorithms require estimating the quantities $m,A,B$ by computing moments from the samples. Therefore the sample complexity primarily depends on how well these quantities concentrate. We compute the specific sample complexities for each model in Appendix \ref{app:concentration}.

For Gaussian mixture model the sample complexity of our algorithm scales as $\tilde{\Omega}(d \epsilon^{-2}\log d)$ similar to moment based algorithm by \cite{HsuKakade:13} and tensor decomposition based algorithm by \cite{AGHKT:14}. In terms of runtime the Whitening algorithm is faster than the tensor decomposition based algorithm by \cite{AGHKT:14}. This can be viewed as follows. The first step in both the algorithms take $O(d^2 k)$ time to compute the whitening matrix and in subsequent whitening steps. However, computing the largest eigenvector in Algorithm \ref{alg:meta-whitening} takes only $O(k^2)$ time, faster than $O(k^5 \log k)$ time required for rank-$k$ tensor power iteration (we also verify this in our experiments in Section \ref{sec:experiments}).

In LDA topic model our algorithms have a sample complexity of $\tilde{\Omega}(\epsilon^{-2}\log d),$ again similar to tensor decomposition based algorithm by \cite{AGHKT:14}, and non-negative matrix factorization (NMF) based algorithm by \cite{AroGeHalMimMoi:12}. The Whitening algorithm again is faster than tensor decomposition as argued for GMM case. The NMF based algorithm using optimization based RecoverKL/RecoverL2 procedures also has a runtime of $O(d^2 k)$ similar to our algorithms (in Section \ref{sec:experiments} again we observe our algorithm to be faster in practice). The spectral topic modeling algorithm in \cite{AnaFosHsuKak:12LDA} also has a computation complexity $O(d^2k)$ similar to our algorithms. However, its sample complexity has a high $\Omega(k^5)$ dependence on the number of components. This spectral algorithm also suffer from instability in practice due to the random projection step (as noted in \citealt{AGHKT:14}).

In the case of mixed linear regression again our method has a sample complexity of $\tilde{\Omega}(d \epsilon^{-2}\log d)$ similar (upto log factors) to the convex optimization based approach by \cite{ChYiCa:14}, alternating minimization based approach by \cite{YiCaSa:13}, but better than tensor decomposition based method of \cite{SeJaAn:14} which has a sample complexity of $\tilde{\Omega}(d^3 \epsilon^{-2}).$ However unlike the convex optimization and alternating minimization based techniques our method is also applicable when the number of components $k>2.$ As argued in GMM case the Whitening algorithm is again faster than the tensor algorithm by \cite{SeJaAn:14}.

Subspace clustering algorithms like greedy subspace clustering by \cite{ParkCarSan:14greedy}, optimization based algorithms by \cite{ElhVid:09SSC}, \cite{SolCan:12GSC}, requires the samples to exactly lie on a subspace. In contrast our moment based algorithm works even when the samples are noisy and perturbed from the actual subspace. Our subspace clustering algorithm also has a sample complexity of $\tilde{\Omega}(m \epsilon^{-2}\log d)$ which is similar (up to log factors) to greedy subspace clustering algorithm by \cite{ParkCarSan:14greedy}.

We note that it is possible to use approximation methods like randomized svd to further speed up the Whitening, Cancellation and tensor decomposition based algorithms by \cite{AGHKT:14}, however this will result in decreased accuracy in both algorithms. We refer to \cite{HuaNirHakAna:15online} for such stochastic optimization, and parallelization techniques used to speed up the tensor algorithms.

In a setting where side information is provided on each of the $k$ components, observe that we can run the Whitening algorithm independently for each of the $k$ components, possibly in parallel. Hence we can recover all $k$ components, without loosing the runtime advantage of the Whitening algorithm. We demonstrate this application on real dataset in Section \ref{sec:real_data_experiments}. In terms of the overall computation time, it can be shown that running the Whitening algorithm for all $k$ components is still faster than the tensor decomposition based algorithm by \cite{AGHKT:14}, when $k = \Omega(n^\frac{1}{3}d^\frac{1}{3}).$

\section{Experiments} \label{sec:experiments}

In this section we present the empirical performance of our Whitening, Cancellation, and Subspace clustering algorithms. We consider three of the settings: the Gaussian Mixture Model (GMM), and Latent Dirichlet Allocation (LDA), and Subspace clustering, and validate our algorithms on both real and synthetic data sets.

\subsection{Synthetic Data Set}

First we compare the sample complexity and runtime of our algorithms with the robust tensor decomposition algorithm by \cite{AGHKT:14}, which is based on tensor power iteration, for learning mixture models (we refer to this as the TPM algorithm). Our second baseline algorithm is a faster heuristic of TPM where we start the tensor power iterations initialized with side information vector $v,$ and recover just the first component. We refer this as the Fast-TPM algorithm. For the Cancellation algorithm we compute the optimum $\lambda$ for cancellation using two different techniques as follows. First, let $\widehat{Z}_{\lambda}' = V^T \widehat{Z}_{\lambda} V,$ where $V$ is the matrix of top $k$ singular vectors of $\smash{\widehat{A}}.$ In the first method, we perform a line search over positive $\lambda$ to find the minimum $\lambda$ such that $\sigma_k(\widehat{Z}_{\lambda}')$ falls below certain threshold. This method works well in GMM case. In a second method we minimize the convex function $\|\widehat{Z}_{\lambda}'\|_{*} + \lambda,$ subject to $\lambda \geq 0$. This method performs better in the case of LDA. Note that for the Cancellation algorithm after estimating $\lambda,$ instead of using $m$ and $A$ to find $\mu_1$ we can follow the same steps using $m'=Av$ and $B$ to recover $\mu_1.$ Theoretically it has the same performance, however empirically we observe this to work slightly better and we use this version for our experiments. We implement all algorithms for our synthetic data experiments using MATLAB.

{\bf Performance metric:} We compute the estimation error of parameter $\mu_1$ as $\mathcal{E}=\|\hat{\mu}_1-\mu_1\|.$ In our figures we plot the quantity ``percentage relative error gain'' which is defined as $G = 100 (\mathcal{E}_T - \mathcal{E}_A)/\mathcal{E}_T,$ where $\mathcal{E}_T$ is the TPM error and $\mathcal{E}_A$ is the error for Whitening / Cancellation / Fast-TPM algorithm. Note that a positive error gain implies that the TPM error is greater than that of the competing algorithm. In the subspace clustering model we plot similar percentage relative error gain over the baseline k-means algorithm.

\begin{figure}[ht]
\centering
\subfigure[]{\includegraphics[height=1.6in,width=1.95in]{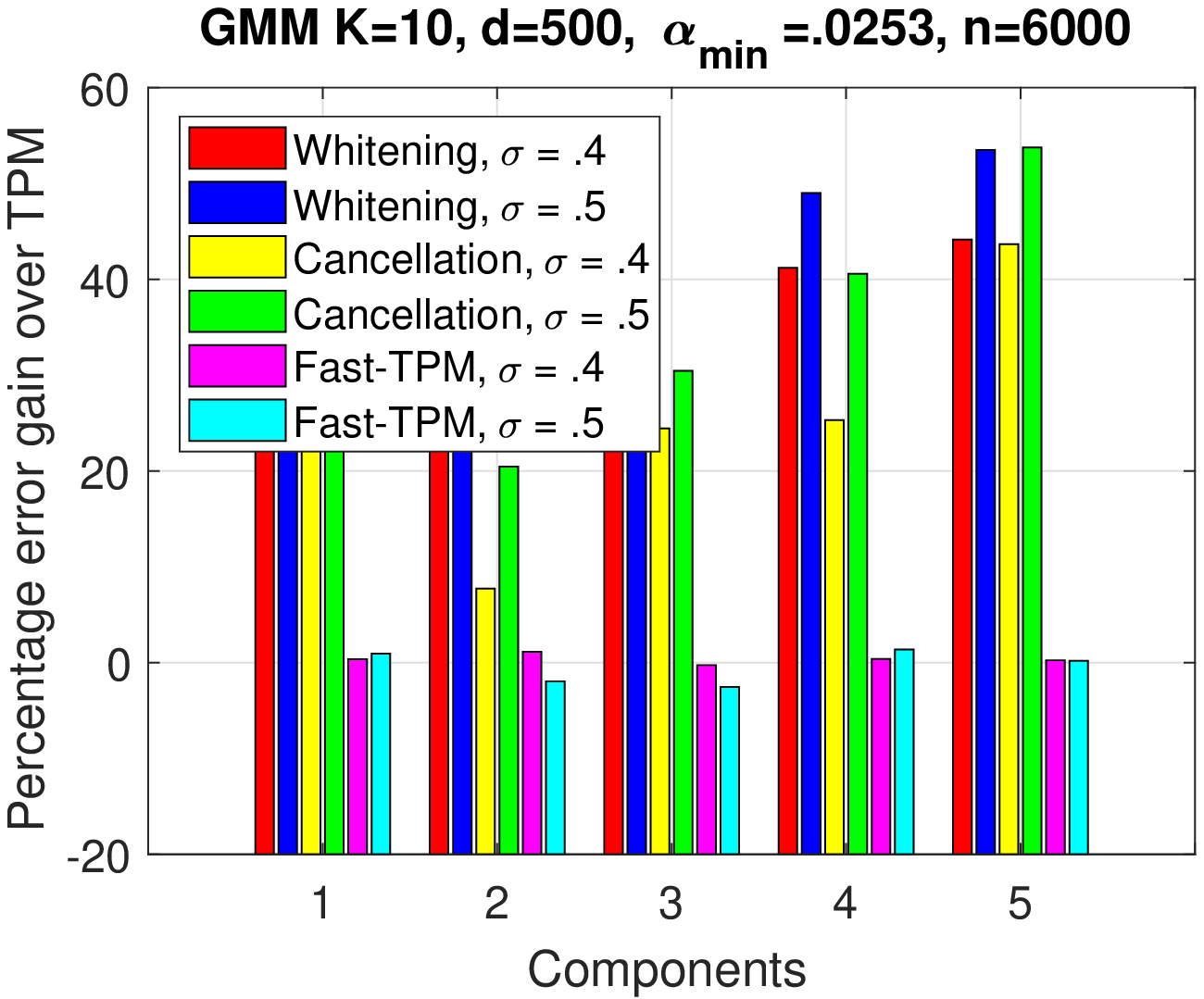}}
\subfigure[]{\includegraphics[height=1.6in,width=1.95in]{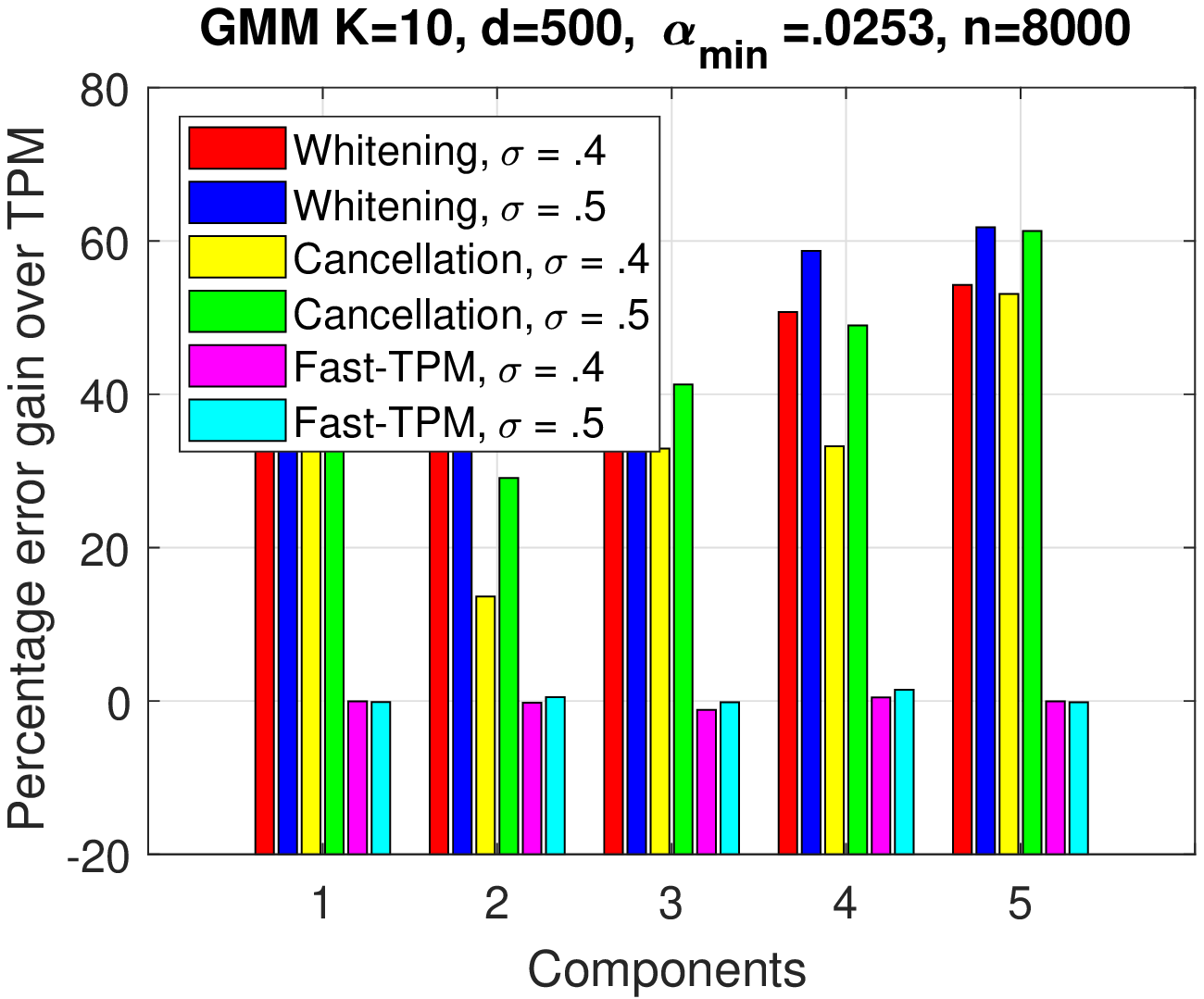}}
\subfigure[]{\includegraphics[height=1.6in,width=1.95in]{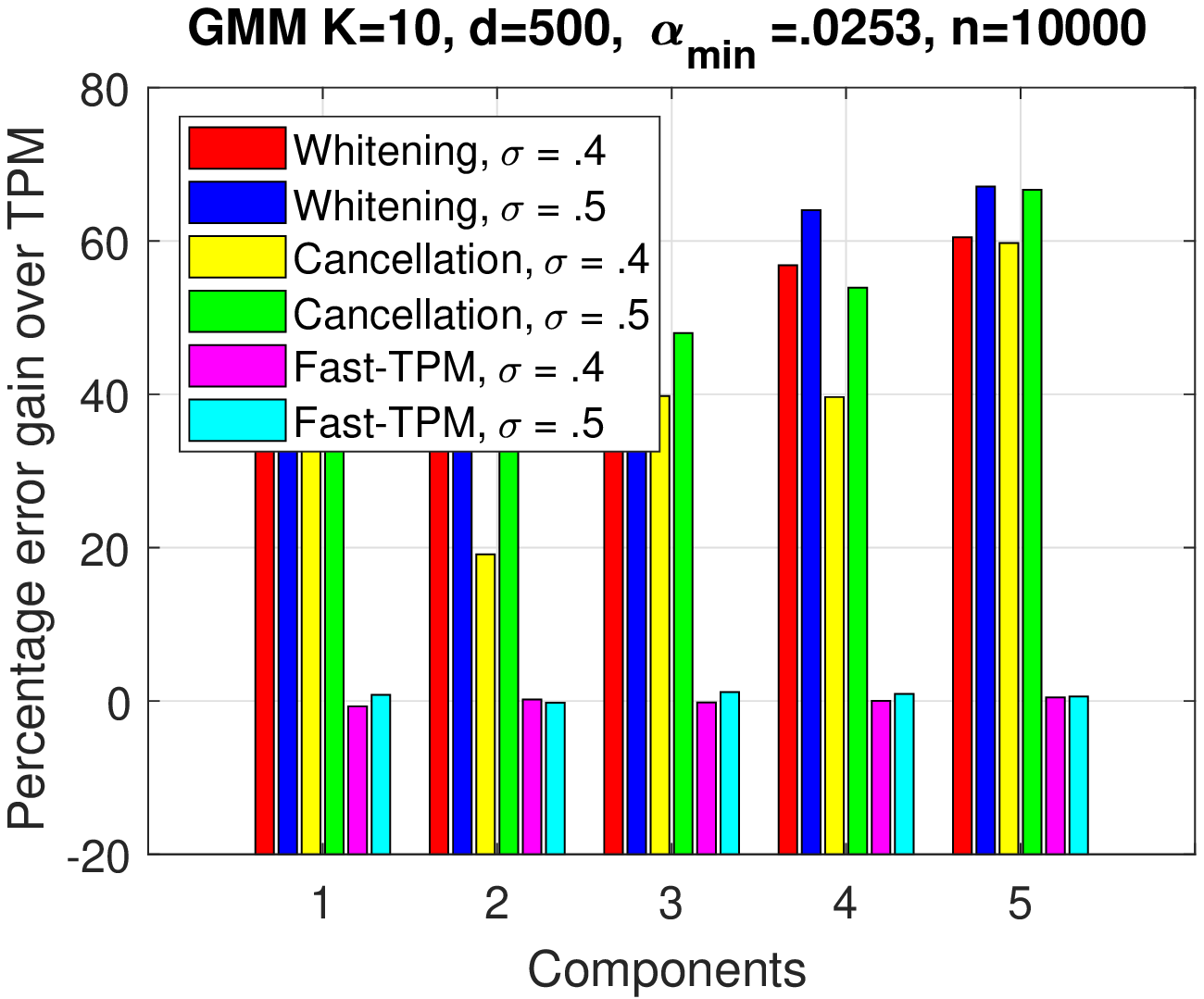}}
\caption{Figure showing the percentage relative error gain by the Whitening, Cancellation, and Fast-TPM algorithm over the TPM algorithm for $5$ components of increasing size, in a GMM with $k=10,d=500,\sigma \in \{.4,.5\},$ and three different sample complexities (a) $n=6000$ (b) $n=8000$ (c) $n=10000.$
Our algorithms shows increasingly better gain over TPM and Fast-TPM as $\alpha_i, \sigma$ and $n$ increase.}
\label{fig:GMM_bar_error_gains}
\end{figure}

{\bf Gaussian mixture model:} We generate synthetic data sets for GMM with different $k,$ $d,$ $\alpha_i,$ $\sigma,$ and $v.$
Figure \ref{fig:GMM_bar_error_gains} shows the percentage relative error gains of the Whitening, Cancellation, and Fast-TPM algorithms over the TPM algorithm in a GMM with various values of $k, d, \alpha_i, \sigma$, and $n$.
The $\mu_i$ were generated randomly over the sphere of norm $r=10.$ We define $\alpha_{min}:=\min_i \alpha_i.$ The side information vector $v$ was chosen as follows. Let $\{v_1, \hdots , v_k \}$ be a orthonormal basis of $\text{span}\{\mu_1 , \hdots , \mu_k\},$ such that $\{v_2 , \hdots , v_k\} \in \text{span}\{\mu_2 , \hdots , \mu_k\}.$ Then we choose $v = \sqrt{\gamma} v_1 + \sqrt{(1-\gamma)/(k-1)} \sum_{i=2}^k v_i$ for some $\gamma \in (0,1)$ such that the condition $\langle \mu_1, v \rangle > \langle \mu_i,v\rangle$ is satisfied. We observe that in all the cases, our algorithms have lower error (positive error gain) than both the tensor algorithms. Moreover, our methods' advantage increases with increasing proportion $\alpha_i,$ increasing sample size $n,$ and increasing variance $\sigma.$ We also observe that the Fast-TPM algorithm has the same error performance as TPM (error gain close to zero).       


Figure \ref{fig:GMM_bar_rare} gives an example where the Whitening algorithm can successfully recover even rare components. Here we consider a GMM with $k=10,d=500$ with the rarest component having probability $\alpha_{min}=.0037.$ Again we observe positive relative error gains over TPM algorithm for increasing number of samples $n.$
 
\begin{figure}[ht]
\centering
\subfigure[]{\includegraphics[height=1.6in,width=1.95in]{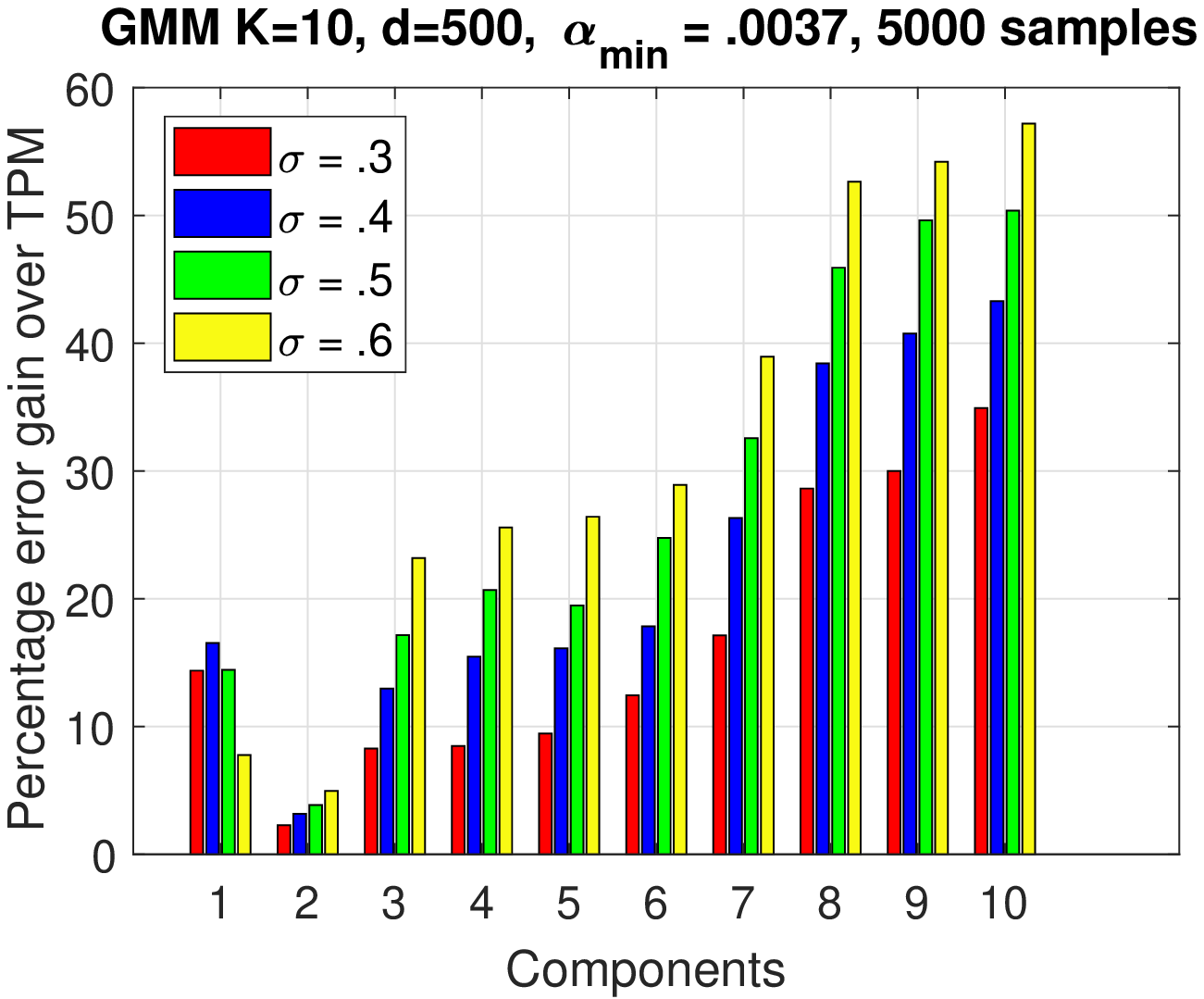}}
\subfigure[]{\includegraphics[height=1.6in,width=1.95in]{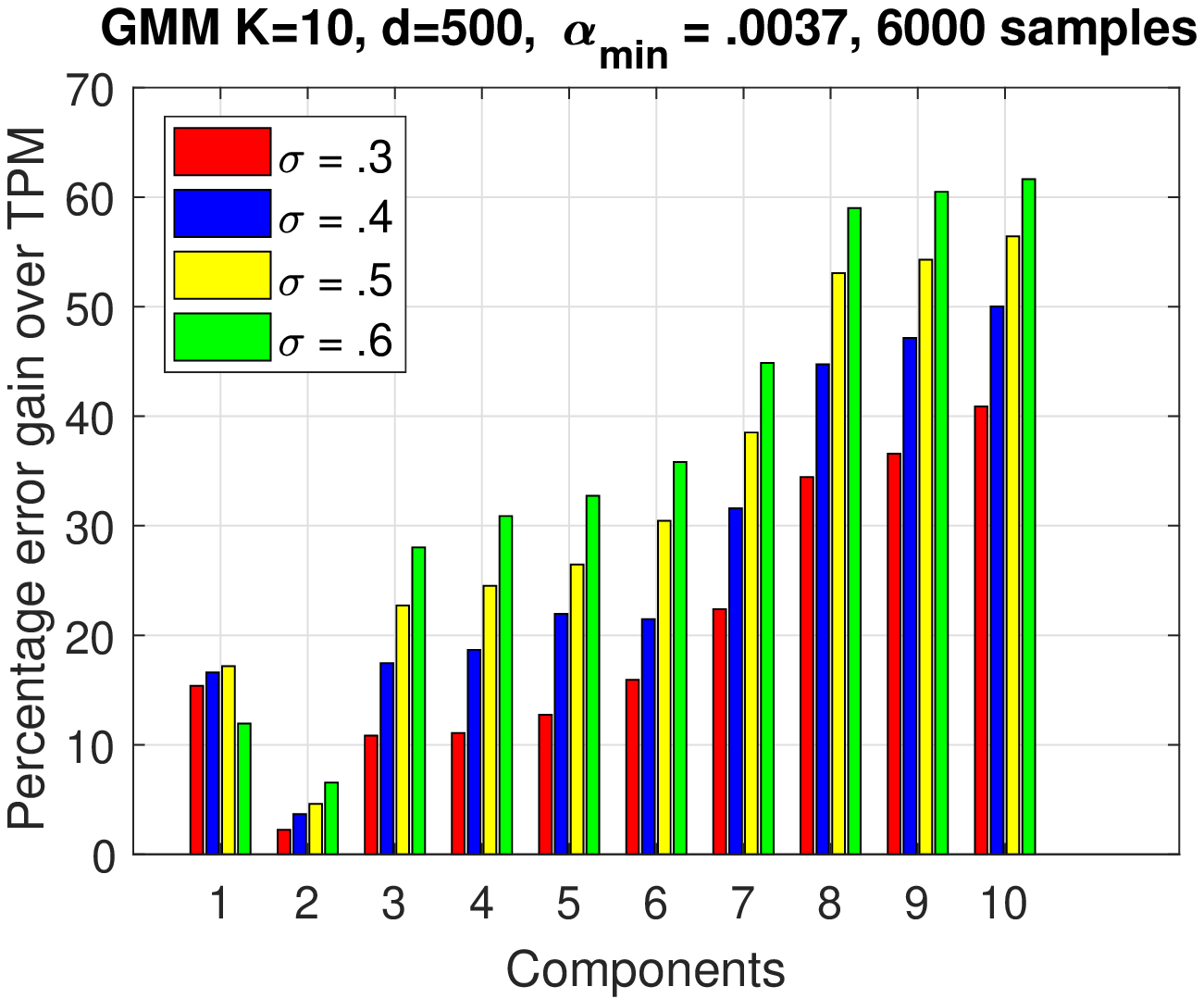}}
\subfigure[]{\includegraphics[height=1.6in,width=1.95in]{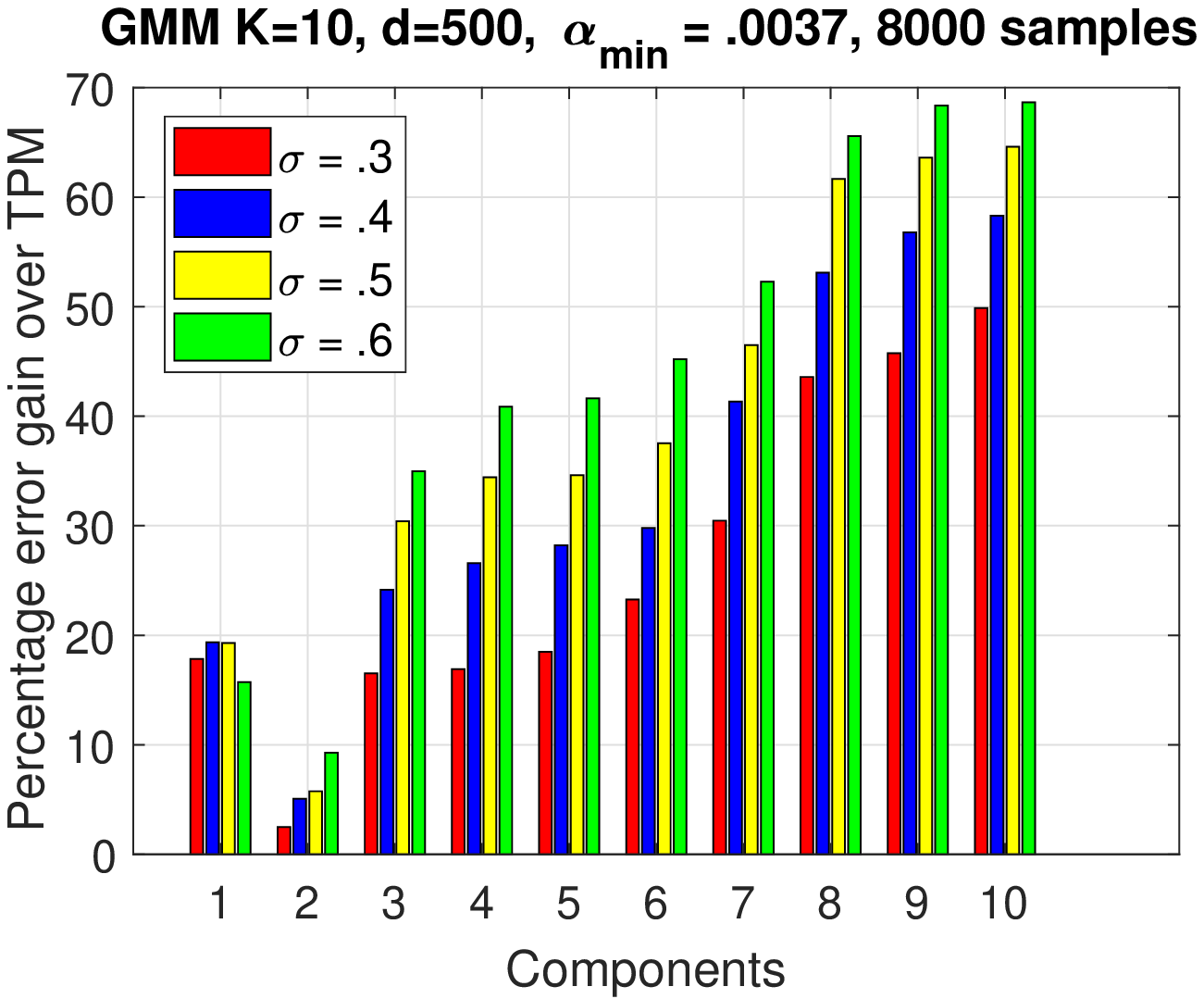}}
\caption{Figure showing the percentage relative error gain of the Whitening algorithm over the TPM algorithm in presence of rare components ($\alpha_{min} = .0037$), for a GMM with $k=10,d=500, \sigma \in \{.3,.4,.5,.6\},$ and number of samples (a) $n=5000$ (b) $n=6000$ (c) $n=8000.$ The Whitening algorithm recovers even the rarest component with increasing error gain over TPM as the number of samples increase.}
\label{fig:GMM_bar_rare}
\end{figure}

In Figure \ref{fig:GMM_runtime} we plot the speedup of the algorithms over TPM, and observe that the Whitening and Cancellation algorithms are much faster (high speedup) than the TPM algorithm. We also observe that the Fast-TPM algorithm is faster than TPM and Cancellation algorithms, but slower than Whitening algorithm. Note that, while it is also possible to speed up the basic TPM algorithm compared here using techniques such as randomized svd and stochastic tensor gradient descent [\citealt{HuaNirHakAna:15online}], such approximate methods will reduce the overall accuracy. Moreover the randomized svd techniques can also be applied to the search algorithms presented in this paper, to obtain further speedups.

\begin{figure}[ht]
\centering
\subfigure[]{\includegraphics[height=1.6in,width=1.95in]{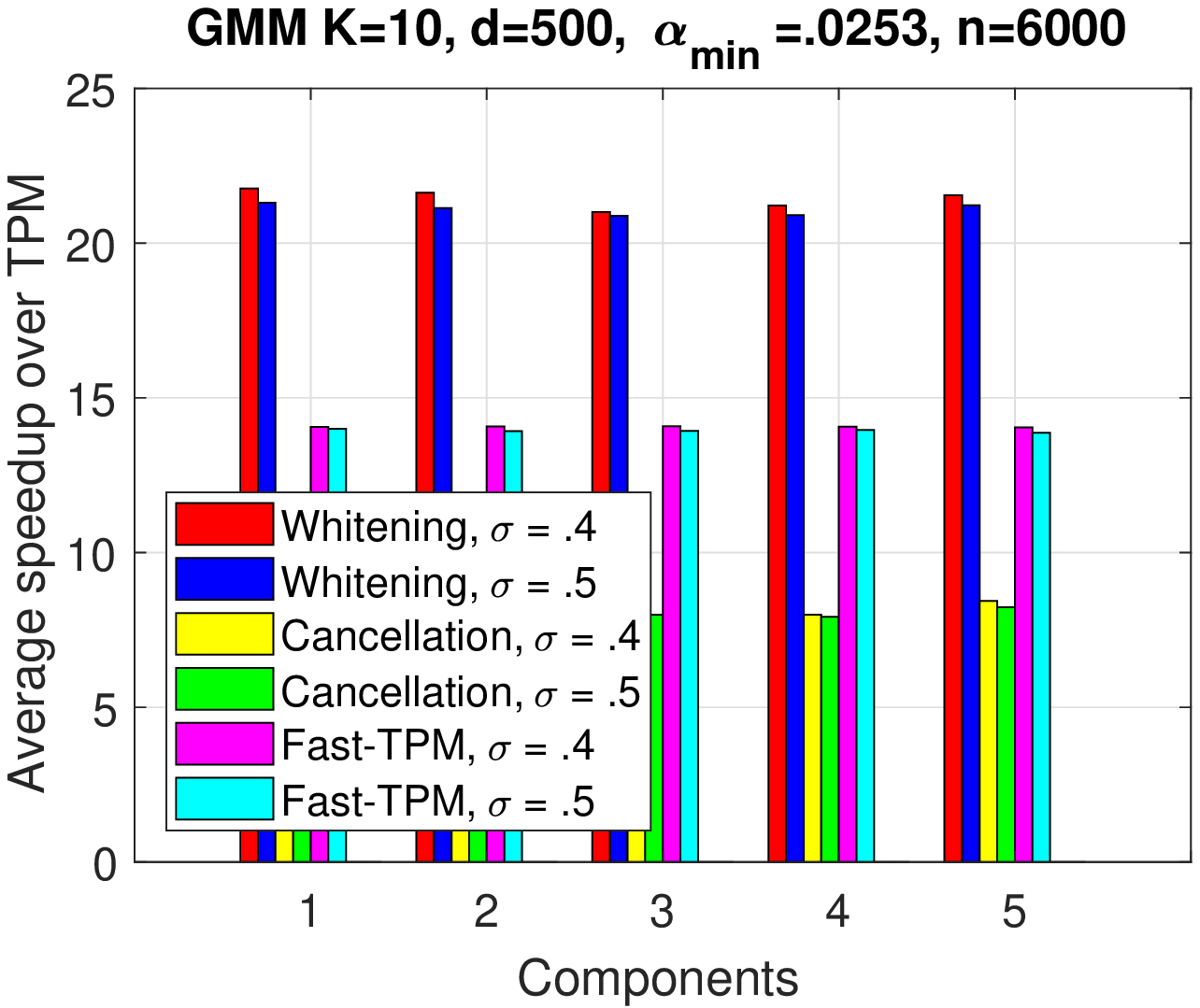}}
\subfigure[]{\includegraphics[height=1.6in,width=1.95in]{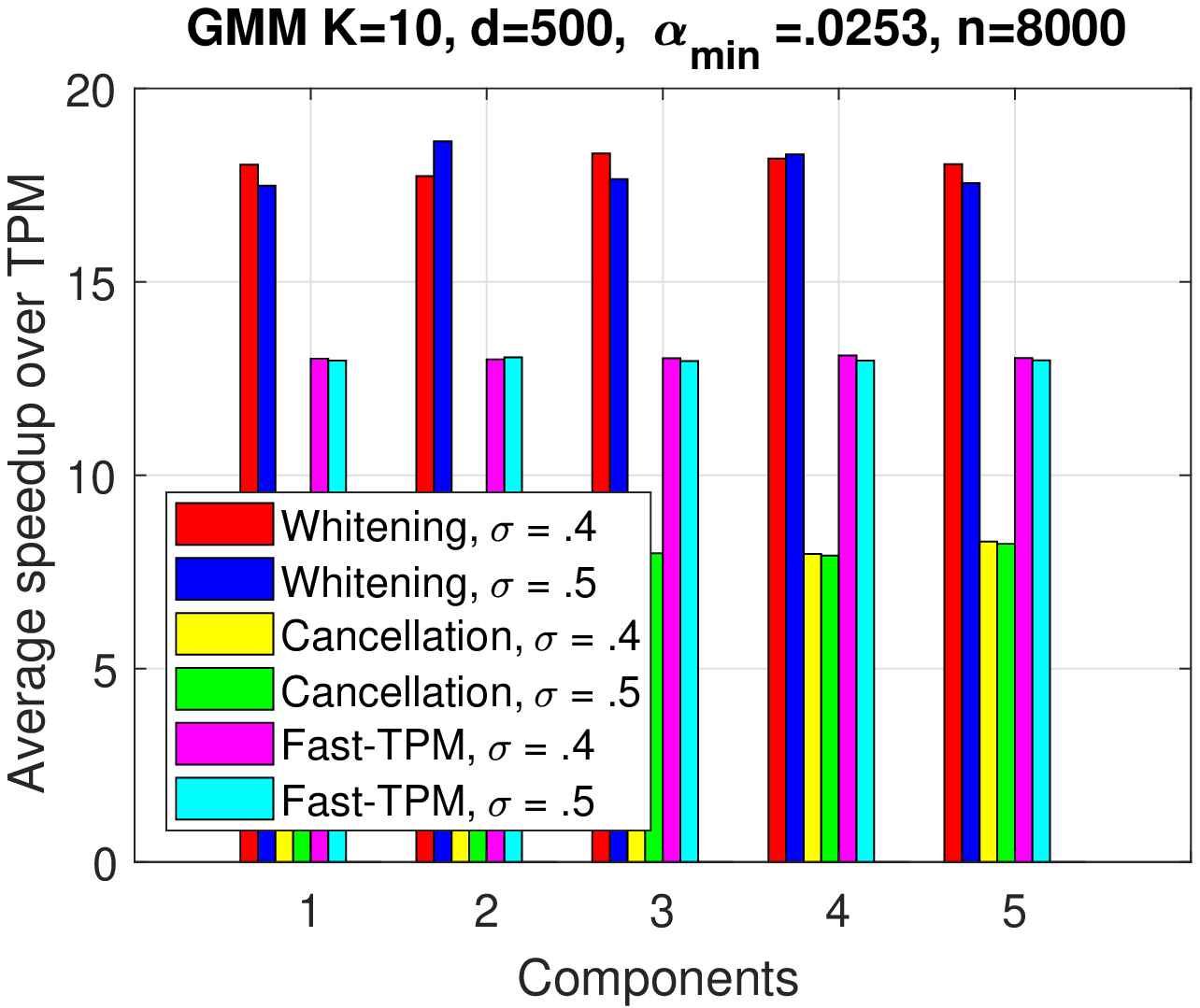}}
\subfigure[]{\includegraphics[height=1.6in,width=1.95in]{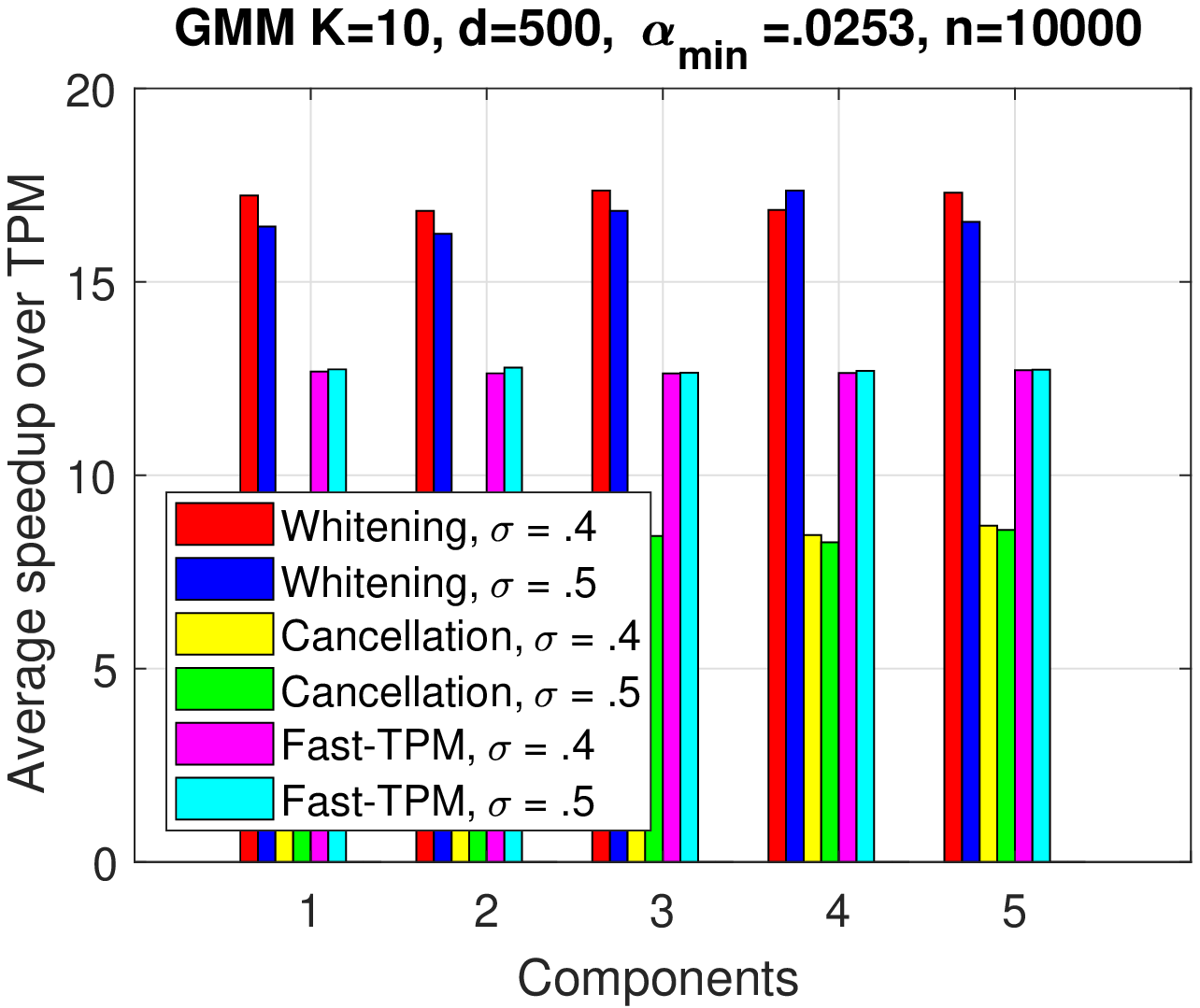}}
\caption{Figure showing the average speedup of Whitening, Cancellation, and Fast-TPM algorithms over TPM, for $5$ components of increasing size, in a GMM with $k=10,d=500,\sigma \in \{.4,.5\},$ and three different sample complexities (a) $n=6000$ (b) $n=8000$ (c) $n=10000.$ The Whitening algorithm is the fastest.}
\label{fig:GMM_runtime}
\end{figure}

{\bf Topic Modeling:} We generate a synthetic LDA document corpus according to the model in \cite{BlNgJo:03}. The lengths of the documents are generated using a Poission($L$) distribution where $L$ is the mean document length. In Figure \ref{fig:LDA_bar_error_gains} we plot the percentage relative error gain of the Whitening, Cancellation, and Fast-TPM algorithms over the TPM algorithm.
Our side information was a labeled word $w$ satisfying $\mu_1(w) > \mu_i(w)$ for $i \neq 1.$ Again we observe positive error gains over the TPM algorithm. Although the Fast-TPM algorithm sometimes perform better than TPM for more frequent topics, the Whitening algorithm still outperforms it. Note that the performance varies across topics since the probability of the labeled word is different for each topic.

\begin{figure}[ht]
\centering
\subfigure[]{\includegraphics[height=1.6in,width=1.95in]{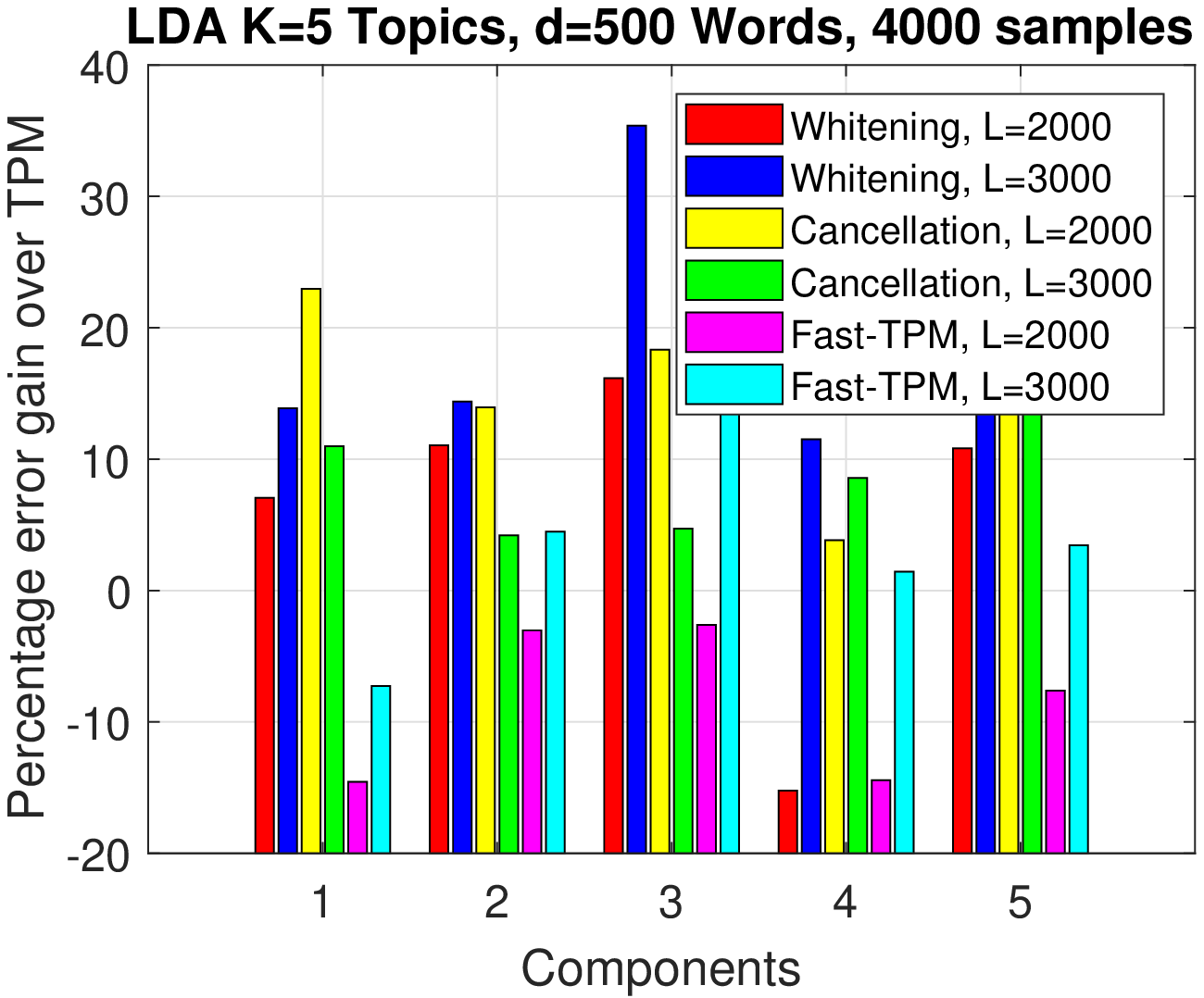}}
\subfigure[]{\includegraphics[height=1.6in,width=1.95in]{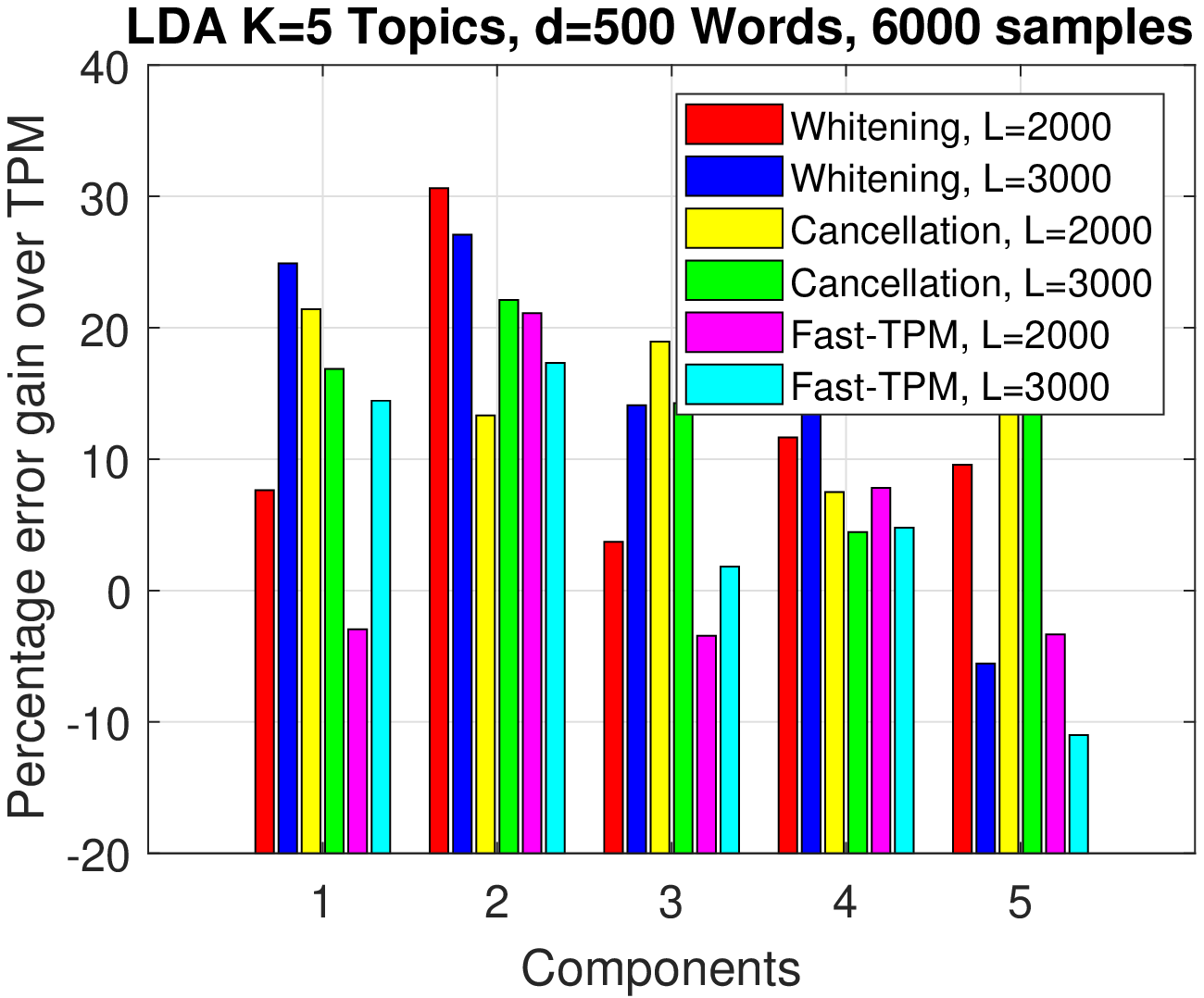}}
\subfigure[]{\includegraphics[height=1.6in,width=1.95in]{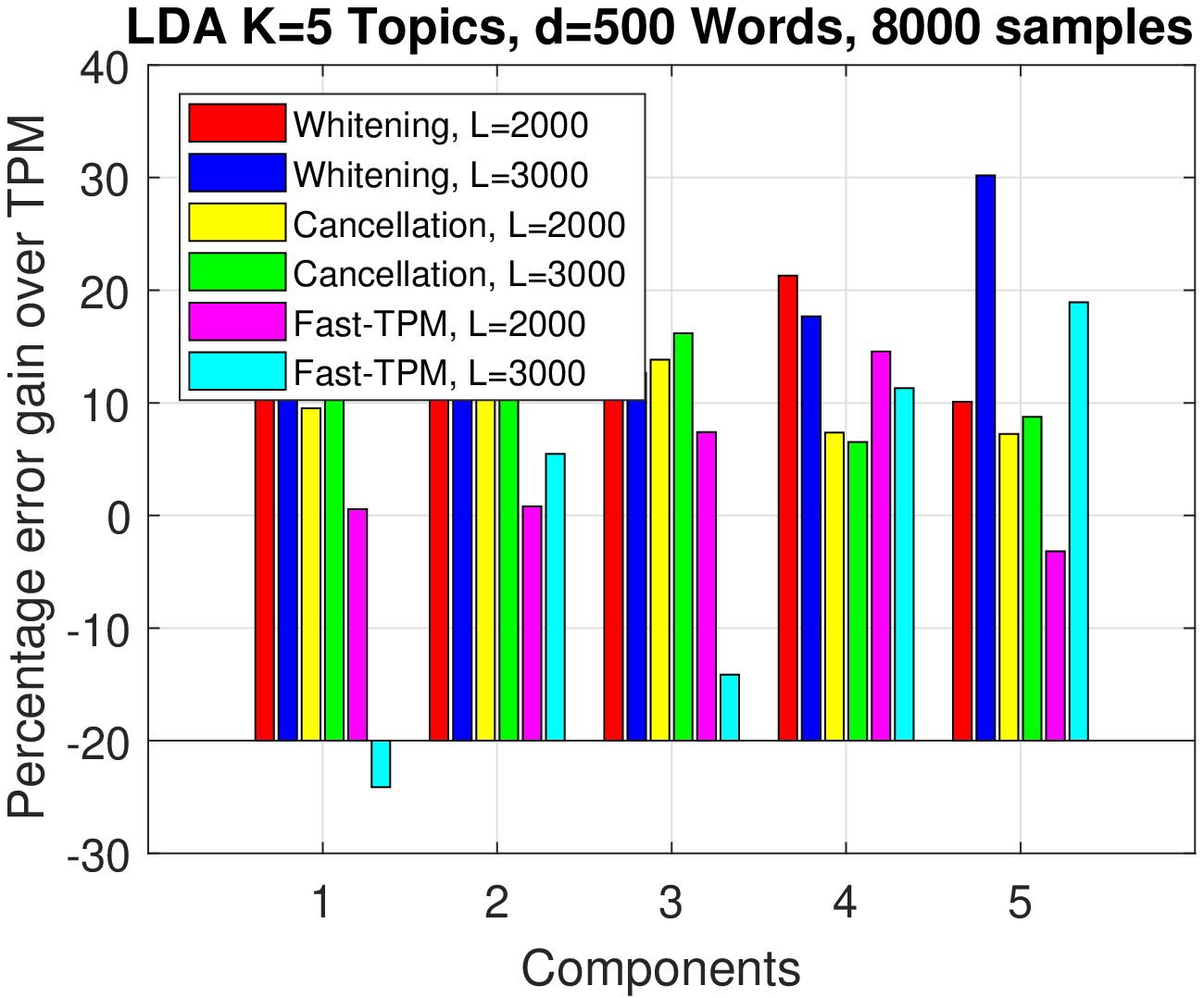}}
\caption{Figure showing the percentage relative error gain in each component of the Whitening, Cancellation, and Fast-TPM algorithms over the TPM algorithm in an LDA model with $k=5,d=500,$ mean document length $L \in \{2000,3000\},$ and number of documents (a) $n=4000$ (b) $n=6000$ (c) $n=8000.$
The Whitening algorithm show an improvement over TPM and Fast-TPM with increasing samples.}
\label{fig:LDA_bar_error_gains}
\end{figure}

{\bf Subspace Clustering:} We generate synthetic data for the subspace clustering model described in section \ref{sec:subspace_model} using parameters $d=500,k=5,m=10,$ and $\alpha_i \in [.1,.3].$ First we generate $k=5$ random subspaces with orthonormal basis $\{U_i\}_{i=1}^k,$ each of dimension $m=10.$ Then we generate random points on these subspaces, and add white Gaussian perturbations with $\sigma \in \{.1,.2\}.$ We choose the side information vector $v$ similar to the sensitivity experiment in GMM, and ensuring $\|U_1^Tv\| > \|U_i^T v\|,$ for $i \neq 1.$ Note that due to the added Gaussian noise, our samples do not lie exactly on the subspaces $\{U_i\}_{i=1}^k,$ but close to it. Traditional subspace clustering algorithms, which assume points to lie exactly on the subspace, may not perform well. The TPM algorithm is also not well suited for this model since (a) the required moment tensor will be of $4^{th}$ order resulting in high computation cost (b) even if $mk$ basis of the tensor are recovered, finding the target subspace will involve a further combinatorial search of $\binom{mk}{m}$ subspaces and finding the one having the strongest projection of $v.$ Therefore we choose the k-means algorithm as our baseline for this model and compare with Algorithm \ref{alg:subspace}. First we compute $k$ clusters using k-means, then we find an $m$ dimensional basis for each cluster using svd, finally we choose the target subspace as the one having the largest projection of $v.$ If $\widehat{U}_1$ is the estimated orthonormal basis for the target subspace $U_1,$ we compute the error as $\mathcal{E}=\|\widehat{U}_1\widehat{U}_1^T-U_1U_1^T\|/\|U_1U_1^T\|.$

\begin{figure}[ht]
\centering
\subfigure[]{\includegraphics[height=1.6in,width=1.95in]{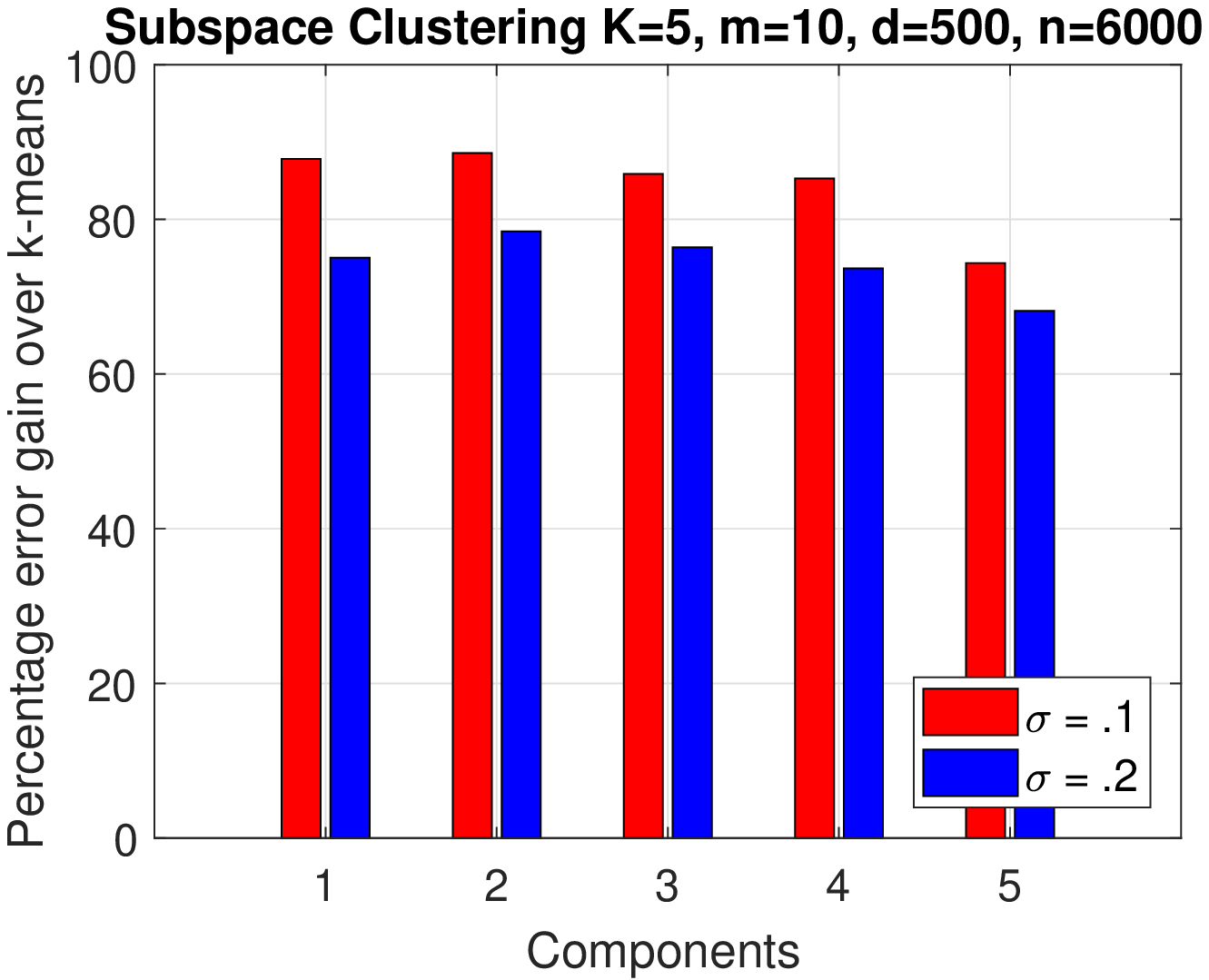}}
\subfigure[]{\includegraphics[height=1.6in,width=1.95in]{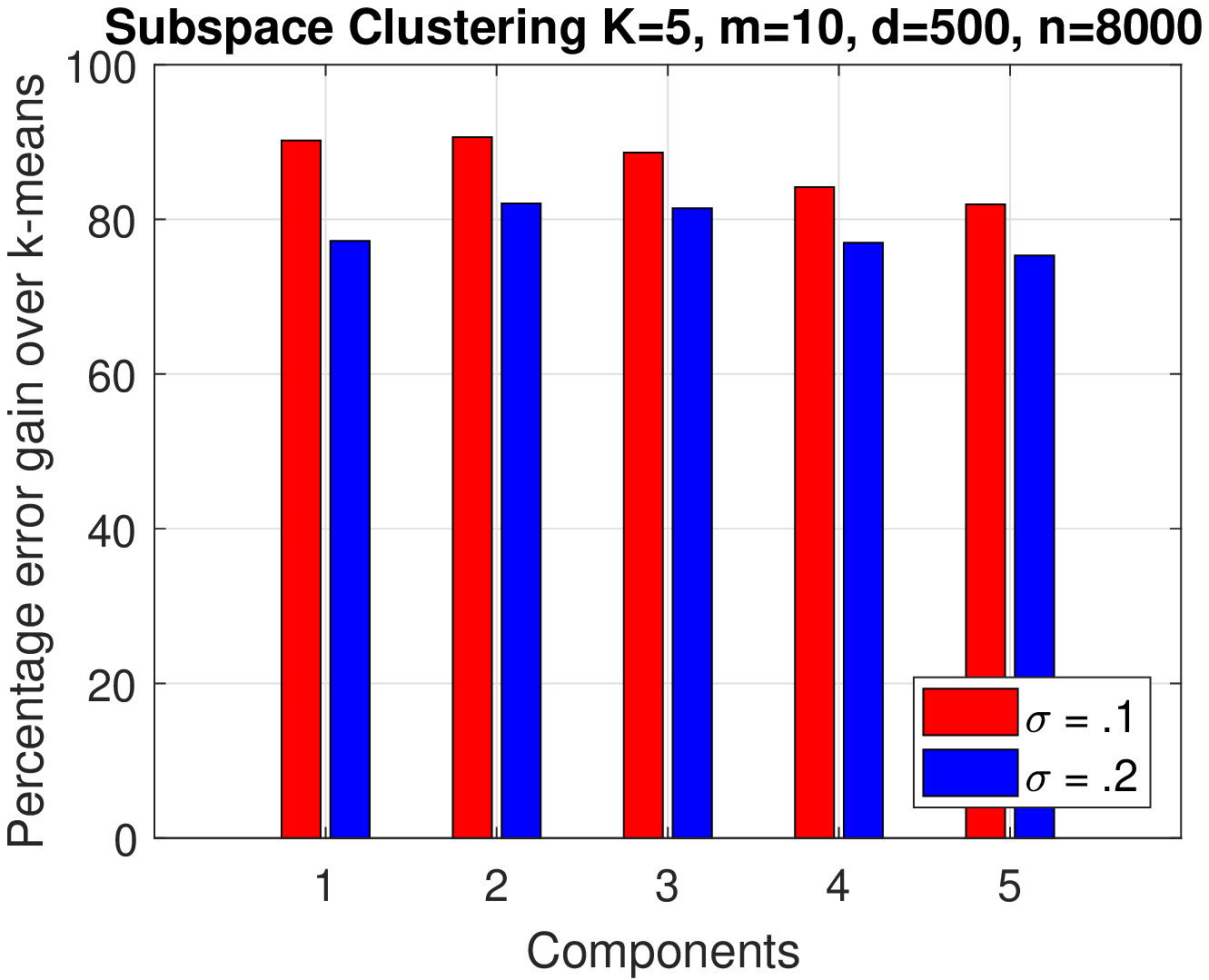}}
\subfigure[]{\includegraphics[height=1.6in,width=1.95in]{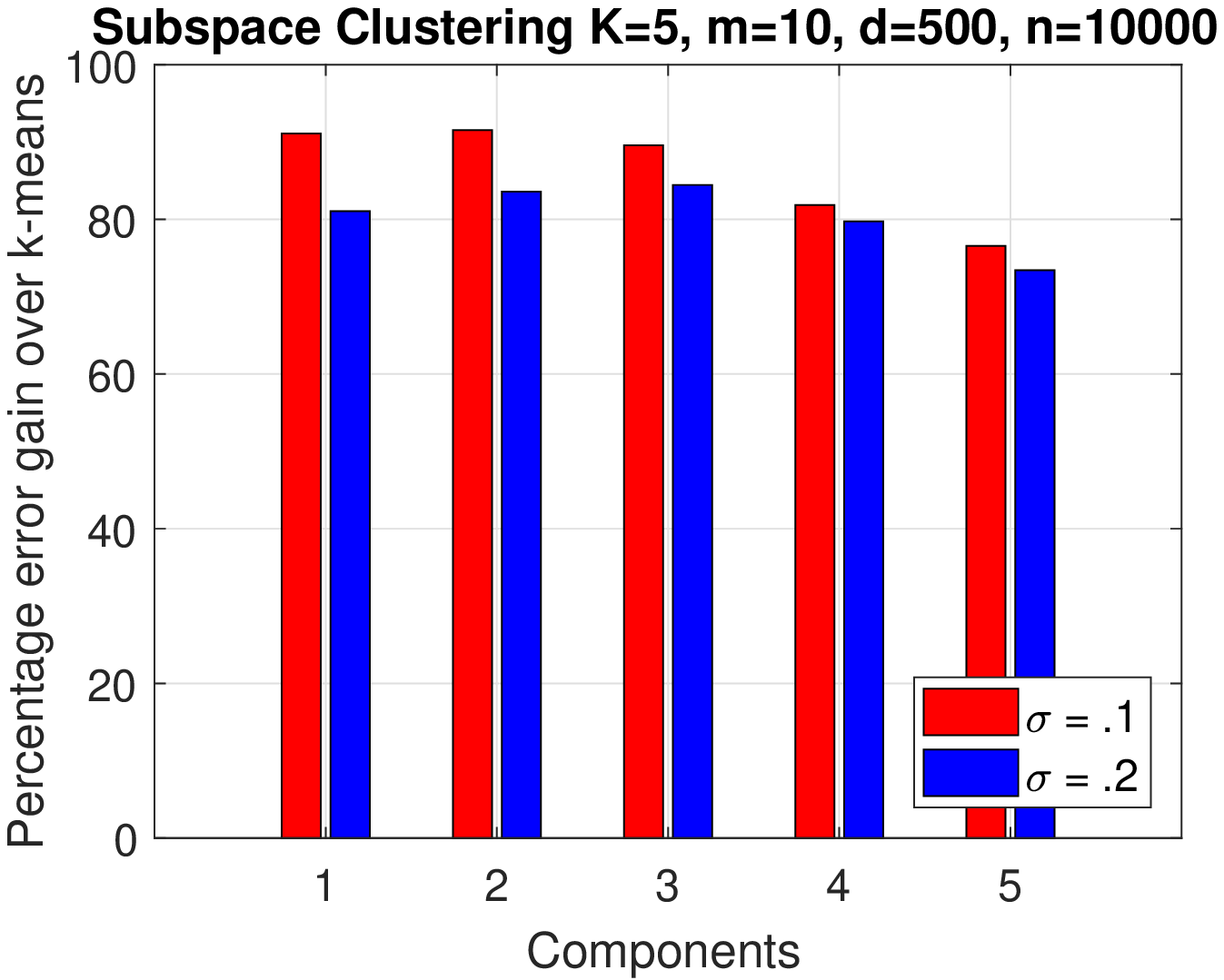}}
\caption{Figure showing the percentage relative error gain by our subspace search algorithm (Algorithm \ref{alg:subspace}) over k-means for $5$ components of increasing size, in a subspace clustering model with $k=5,m=10,d=500,\sigma \in \{.1,.2\},$ and three different sample complexities (a) $n=6000$ (b) $n=8000$ (c) $n=10000.$ Our algorithm shows much better error performance than k-means.}
\label{fig:Subspace_error_gains}
\end{figure}

Figure \ref{fig:Subspace_error_gains} shows that Algorithm \ref{alg:subspace} has a much better error performance over k-means. In the speedup plots in Figure \ref{fig:Subspace_speedup} we also observe that our subspace search algorithm is over $4X$ times faster than k-means.

\begin{figure}[ht]
\centering
\subfigure[]{\includegraphics[height=1.6in,width=1.95in]{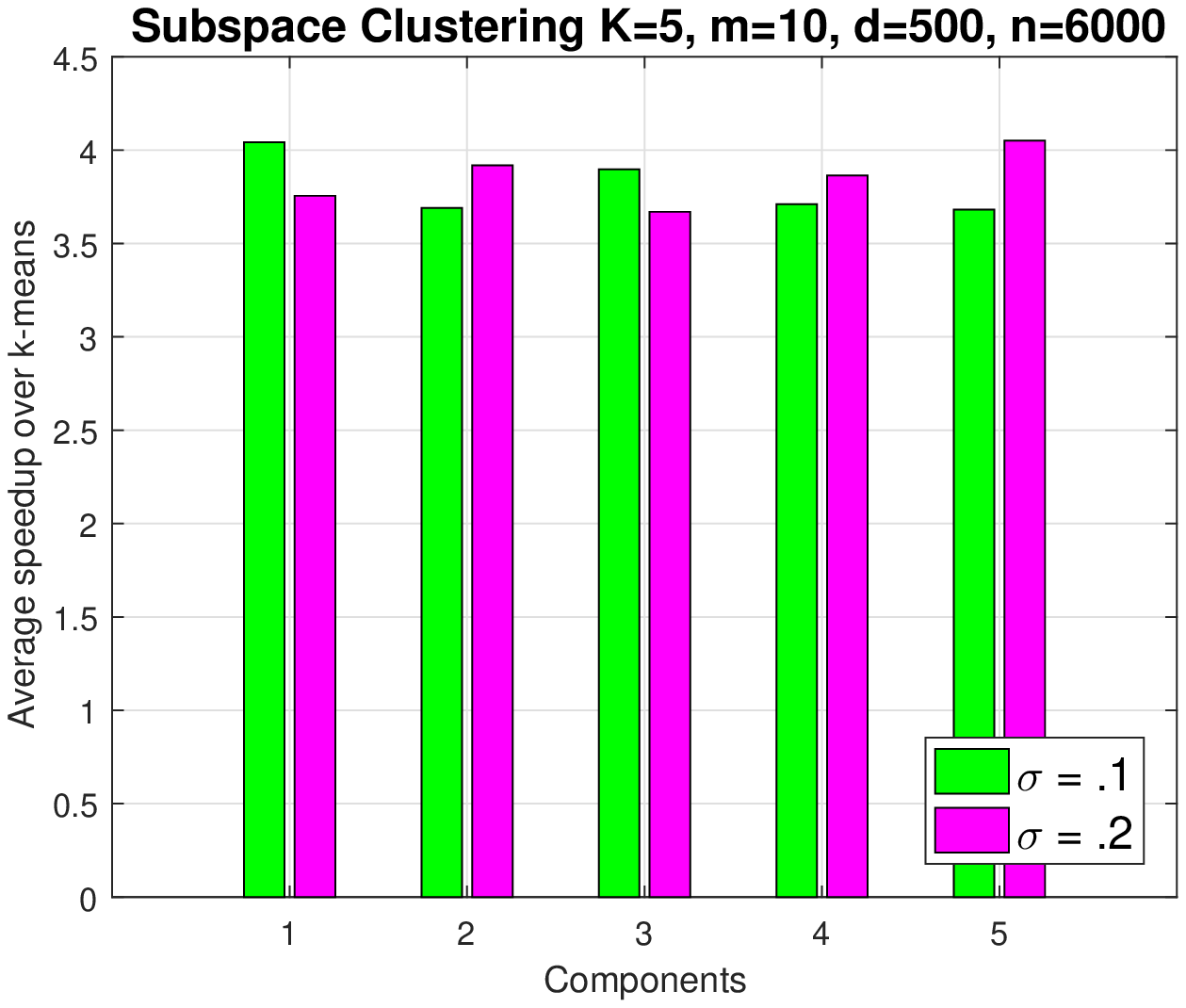}}
\subfigure[]{\includegraphics[height=1.6in,width=1.95in]{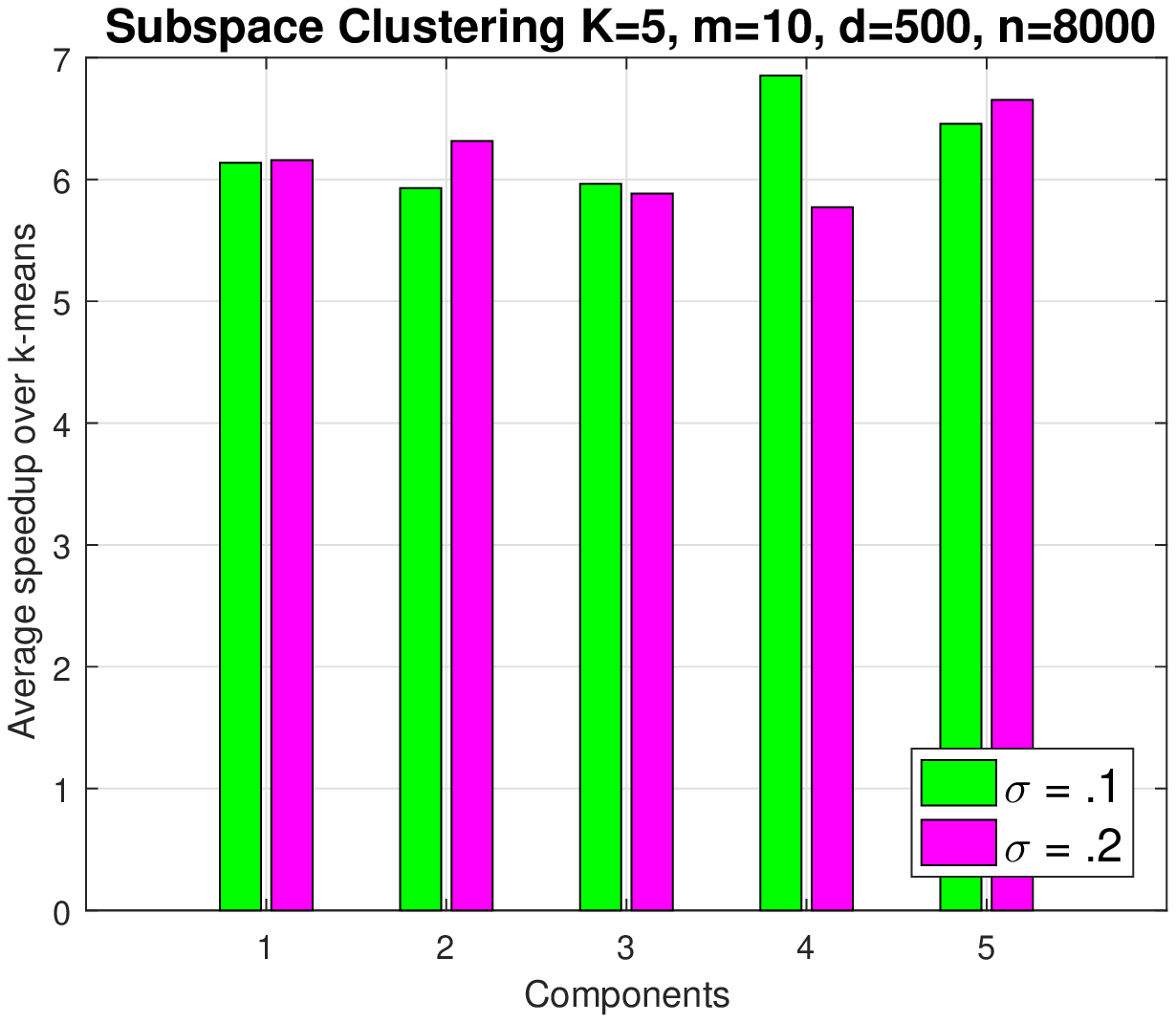}}
\subfigure[]{\includegraphics[height=1.6in,width=1.95in]{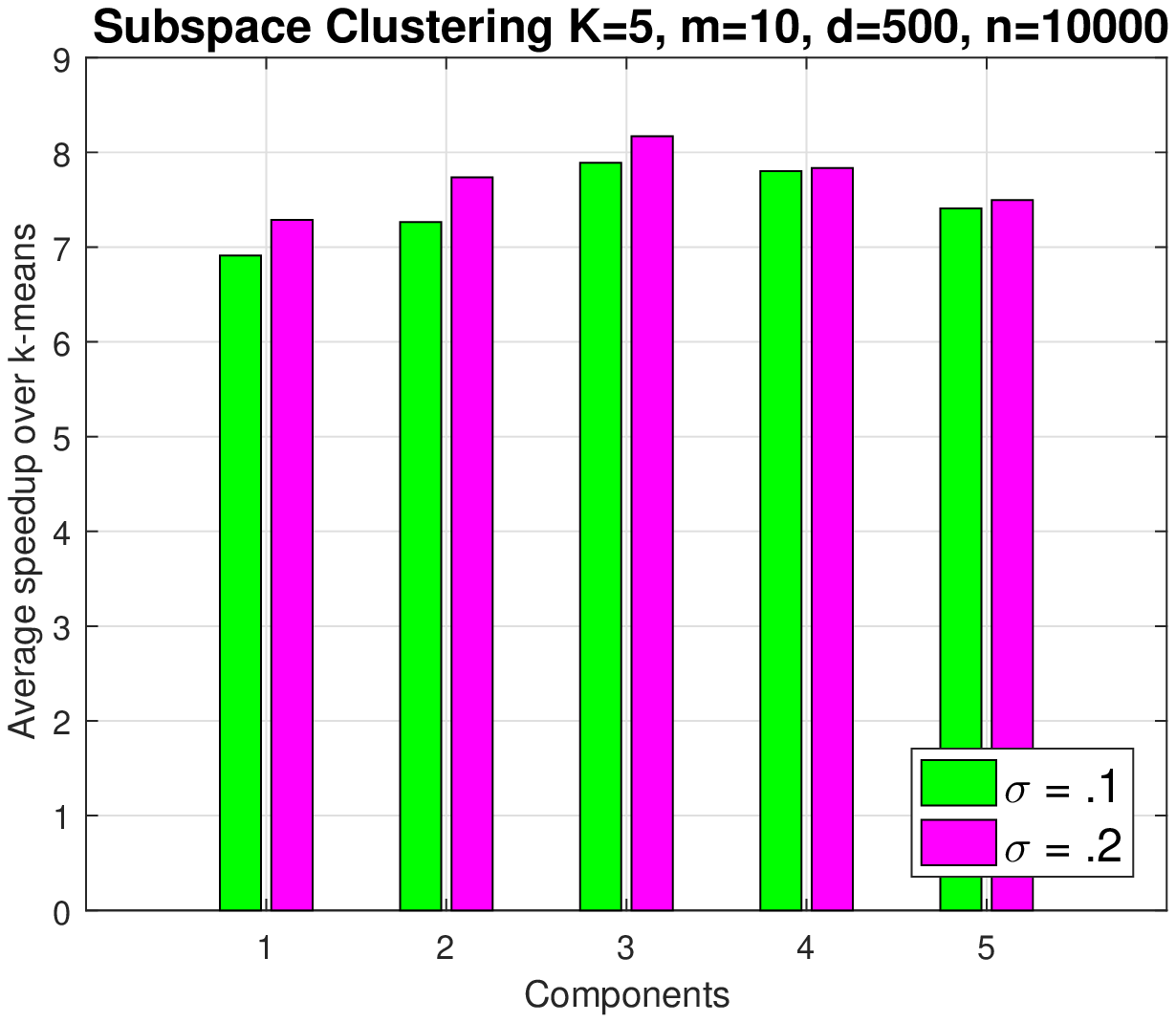}}
\caption{Figure showing the average speedup of our subspace search algorithm (Algorithm \ref{alg:subspace}) over k-means, for $5$ components of increasing size, in a subspace clustering model with $k=5,m=10,d=500,\sigma \in \{.1,.2\},$ and three different sample complexities (a) $n=6000$ (b) $n=8000$ (c) $n=10000.$ Our subspace clustering algorithm shows high speedup over k-means.}
\label{fig:Subspace_speedup}
\end{figure}

\subsection{Real Data Sets} \label{sec:real_data_experiments}

{\bf Topic Modeling:} In this section we compare the performance of Whitening algorithm with a recent non-negative matrix factorization based topic modeling algorithm by \cite{AroGeHalMimMoi:12} (we refer this as NMF algorithm), and also the semi-supervised version of this NMF algorithm (we refer to this as SS-NMF). We test on two real large data sets; (a) New York Times news article data set [\citealt{NYTimesDataset:08}] ($300,000$ articles) (b) Yelp data set of business reviews [\citealt{YelpDataset:14}] ($335,022$ reviews). We run both algorithms for $k=100$ topics. For this experiment we do not consider the TPM algorithm by \cite{AGHKT:14} since its runtime with $k=100$ topics becomes extremely large on these data sets.\footnote{To be more precise, with just $k=10$ topics, the tensor algorithm takes $908$ seconds in NY Times data set, compared to just $188$ seconds for the Whitening algorithm (using MATLAB).} In contrast, the NMF algorithm is known to be faster, and produce topics of comparable quality to more popular variational inference based algorithms [\citealt{BlNgJo:03}]. The side information for this experiment are chosen as follows. First from the set of topics produced by NMF algorithm we choose a subset of interpretable topics, then we choose labeled words representative of these topics. We test with a set of $62$ labeled words for NY Times data set and $54$ labeled words for Yelp data set. Note that given labeled word $w_l$ the whitening algorithm produces one topic distribution $\mu_1,$ but the NMF algorithm finds $k$ topics. Therefore for NMF algorithm the target topic $i$ is the one which has the highest probability of the labeled word i.e., $\mu_i(w_l).$ For the semi-supervised NMF we first compute the weighted word-word co-occurrence matrix $Q_{w}$ where we re-weigh each document by the normalized frequency of the labeled word $w_l.$ Then we apply the NMF algorithm [\citealt{AroGeHalMimMoi:12}] on this weighted matrix $Q_w.$ All three algorithms were implemented in Python.

{\em Performance metric:} We compare the quality of the topics returned by Whitening, NMF, and SS-NMF algorithms using the pointwise mutual information (PMI) score, known to be a good metric for topic coherence [\citealt{NewLauGriBald:10,Roder:15Palmetto}]. However in order to also capture the relevance of the estimated topic to the labeled word we compute PMI score for topic $i$ as,
$$
PMI(\text{topic i}) = \frac{1}{20} \sum_{w \in \mathcal{T}^i_{20}} \log \frac{p(w_l , w)}{p(w_l)p(w)} 
$$    

where $w_l$ is the labeled word, $\mathcal{T}^i_{20}$ is the set of top $20$ words in the $i$-th topic. The probabilities $p(w_l,w),p(w),p(w_l)$ are computed over a larger data set of English Wikipedia articles to reduce noise [\citealt{NewBonBun:11}]. For whitening algorithm we choose $\alpha_0 = .01.$ Note that other supervised topic modeling algorithms e.g. supervised LDA by \cite{BleiMcAuliffe:08}, labeled LDA by \cite{RamHallNallMan09} require a much stronger notion of side-information than just labeled words, hence we could not compare with them.

\begin{figure}[ht]
\centering
\subfigure[]{\includegraphics[height=1.8in,width=2.25in]{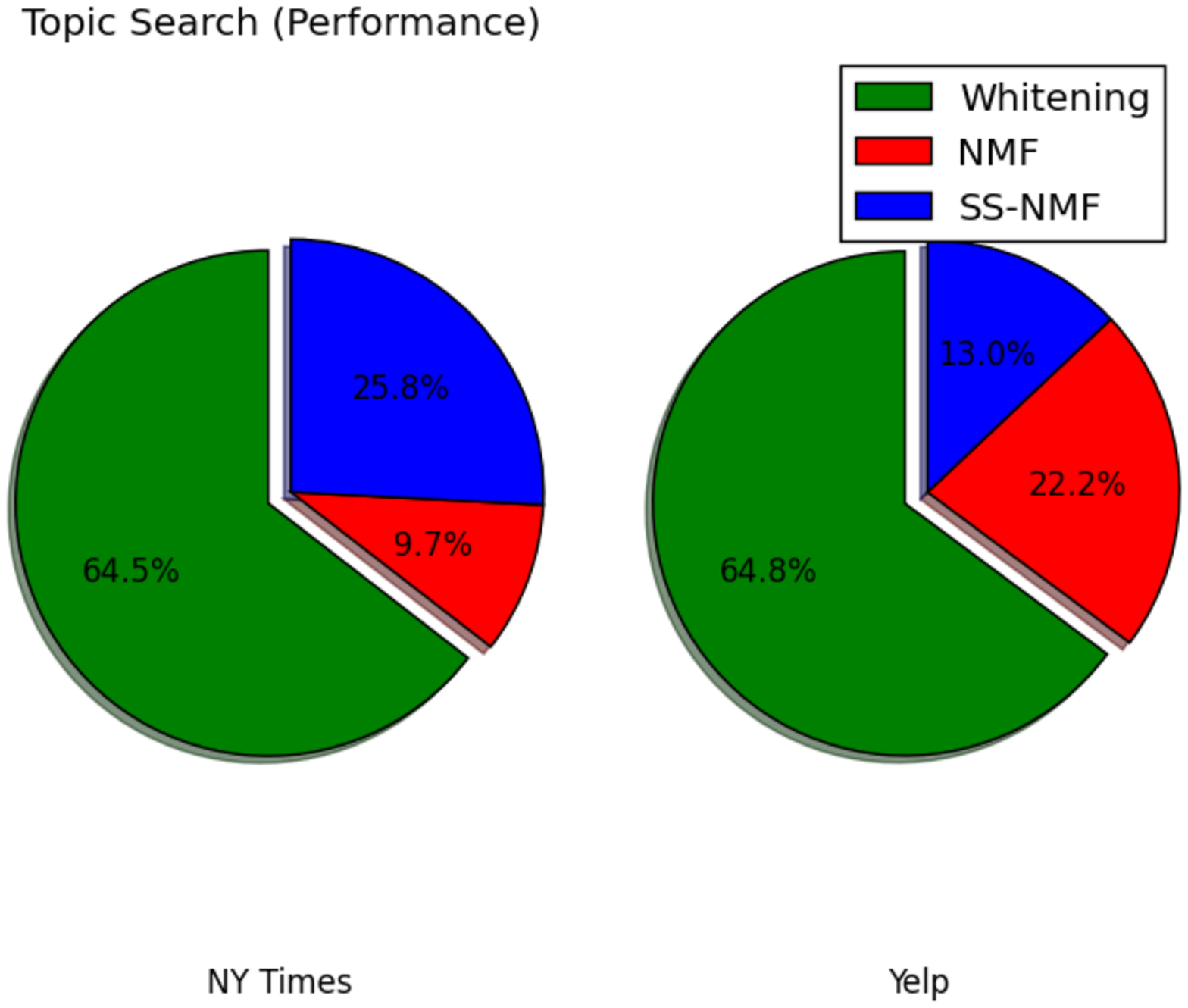}}
\subfigure[]{\includegraphics[height=1.8in,width=2.25in]{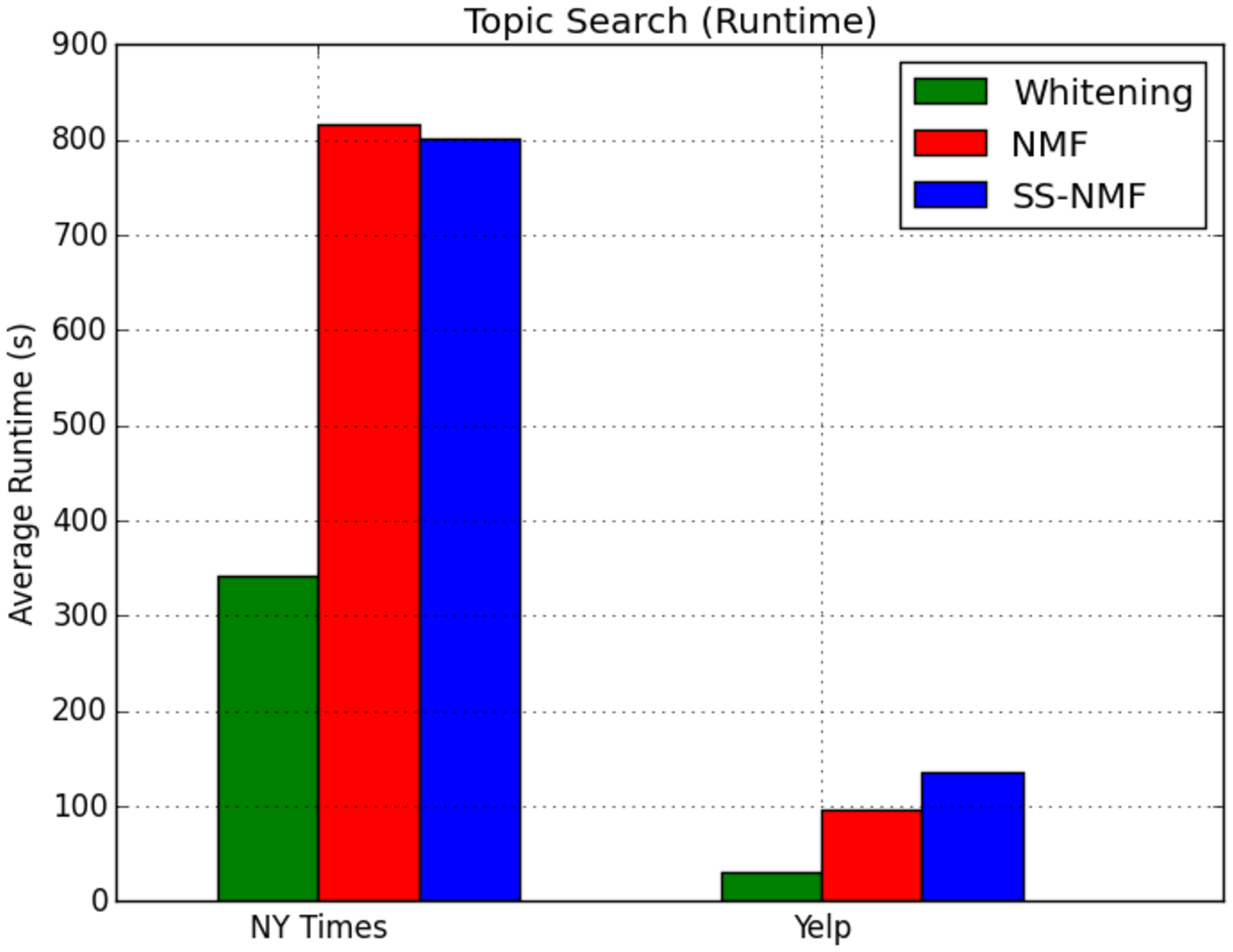}}
\caption{Figure comparing the performance of Whitening, NMF [\citealt{AroGeHalMimMoi:12}], and semi-supervised NMF (SS-NMF) algorithms on NY Times and Yelp data sets. (a) Topics estimated by Whitening algorithm have the best PMI score in $40$ out of $62$ labeled words for NY Times data set, and $35$ out of $54$ labeled words in Yelp data set. (b) Whitening shows more than 2X speedup over competing algorithm in both data sets.}
\label{fig:yelp_nytimes_results}
\end{figure}

In Figure \ref{fig:yelp_nytimes_results} (a) we plot the percentage of labeled words for which each algorithm has the best PMI score. Observe that for most labeled words ($40$ out of $62$ labeled words for NY Times data set, and $35$ out of $54$ labeled words in Yelp data set) the Whitening algorithm estimates topic with better PMI score over NMF and SS-NMF algorithms. The Whitening algorithm is also more than twice as fast as NMF and SS-NMF\footnote{For large corpus the NMF algorithm runs much faster than Gibbs sampling and variational inference based algorithms [\citealt{AroGeHalMimMoi:12}].} as shown in Figure \ref{fig:yelp_nytimes_results} (b). A complete list of topics and PMI scores returned by the algorithms for every labeled word is presented in Tables \ref{tab:nytimes}, \ref{tab:yelp} of Appendix \ref{app:nytimes_yelp_full}. Notice that the Whitening algorithm often estimates more coherent topics which are more relevant to the given labeled word than topics produced by the NMF/SS-NMF algorithm. For example in NY Times data set with the labeled word {\em student} the Whitening algorithm returns top five words in the topic as {\em student, school, teacher, percent, program}; however those returned by NMF algorithm are {\em test, school, student, ignore, export}; and those by SS-NMF algorithm are {\em student, university, shooting, shot, rampage}.\\

\noindent{\bf Parallel image segmentation:} One method to perform image segmentation is to use GMM clustering. In this experiment we demonstrate how GMM search algorithm can be used to parallelize image segmentation in vision applications. For this we consider the BSDS500 data set introduced in \cite{ArbMaiFowMal:11} and choose a subset of $70$ images having less than $4$ segments in the ground truth. Note that this data set has up to six ground truth segmentation by human users for each image. We randomly choose one pixel from each segment in ground truth as side-information $v$. We compare our Whitening algorithm with the seeded k-means clustering [\citealt{BasuBanMoo:02}] where the centers are initialized by these side-information pixels (we refer to this as s-Kmeans). The Whitening algorithm uses one pixel from the $i$-th cluster to compute $\mu_i,$ in parallel for every $i,$ and then it assigns each pixel to its closest $\mu_i.$ The segmentation quality is compared using normalized mutual information (NMI) metric [\citealt{ManningInfoRetieval:08}]. To avoid local minimum in s-Kmeans we consider the maximum NMI over $5$ initializations of side-information for each ground truth, and then we compute average NMI over all ground truths for an image.  

\begin{figure}[hptb]
\begin{center}
\includegraphics[height =1.9in]{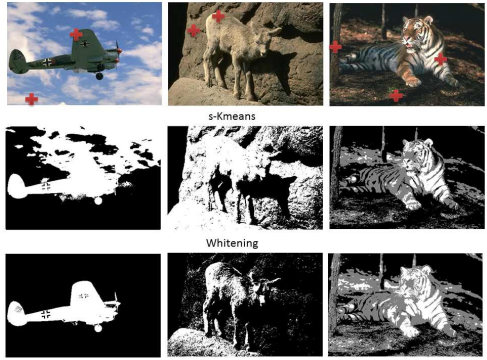}
\caption{Figure comparing the performance of image segmentation by Whitening (row $3$) and s-Kmeans (row $2$) algorithms, with images selected from the BSDS500 data set. The side-information pixels are shown in red plus in the original image (row $1$). In the segmented images (rows $2,3$) the segments are shown in different shades. Observe that the Whitening algorithm often isolates the foreground segment better than s-Kmeans.\label{fig:segmentation} }
\end{center}
\end{figure}

We summarize our result in Table \ref{tab:segmentation}. Observe that the Whitening algorithm has a slightly better NMI performance over s-Kmeans in the BSDS test data set and similar performance in BSDS train and BSDS val data sets. However the Whitening algorithm runs an order of magnitude faster than s-Kmeans. 

\begin{table}

\centering 
\begin{tabular}{|c|c|c|c|c|c|c|c|}
  \hline 
    Data set & $N$ & $N_W$ & $N_K$ & $T_W$ (s) & $T_K$ (s) & $\overline{NMI}_W$ & $\overline{NMI}_K$\\
  \hline\hline
	 BSDS test & $30$ & $17$ & $13$ & $6.7$ & $81.5$ & $0.17$ & $0.13$ \\
  \hline
	 BSDS train & $25$ & $12$ & $13$ & $8.2$ & $89.8$ & $0.15$ & $0.15$ \\
  \hline
	 BSDS val & $15$ & $8$ & $7$ & $10.6$ & $117.2$ & $0.11$ & $0.09$ \\
	\hline
\end{tabular}
\caption{Table comparing the performance of Whitening and s-Kmeans algorithm on BSDS data set. $N$ is the total number of images, $N_W$ is the number of images where segmentation produced by Whitening has a better NMI than s-Kmeans, and $N_K$ is the number of images where segmentation of s-Kmeans has a better NMI. $T_W$ is the median runtime of Whitening algorithm and $T_K$ is the median runtime of s-Kmeans. $\overline{NMI}_W$ and $\overline{NMI}_K$ are the median NMI scores for the Whitening and s-Kmeans algorithms respectively. Whitening runs much faster than s-Kmeans.} 
\label{tab:segmentation} 
\end{table}

\section{Conclusion and Discussion}

In this paper we developed a new, simple and flexible framework for incorporating side information into mixture model learning. The underlying motivation was to provide a principled way to take into account extra input (e.g. generated by human data analysts etc.). Even for cases where this input is very limited compared to the size/dimensionality of the data, we show meaningful statistical and computational performance improvement over baseline unsupervised and semi-supervised methods. More generally, developing methods which work with very limited human input is a promising research endeavor, in our opinion.

\section*{Acknowledgement}
We would like to acknowledge support from NSF grants
  CNS-1320175, 0954059, ARO grants W911NF-15-1-0227, W911NF-14-1-0387,
  W911NF-16-1-0377, and the US DoT supported D- STOP Tier 1 University Transportation Center. The authors also acknowledge the Texas Advanced Computing Center [\citealt{TACC}] at The University of Texas at Austin for providing HPC resources that have contributed to the research results reported within this paper.
  
\bibliography{search}

\begin{thebibliography}{35}
\providecommand{\natexlab}[1]{#1}
\providecommand{\url}[1]{\texttt{#1}}
\expandafter\ifx\csname urlstyle\endcsname\relax
  \providecommand{\doi}[1]{doi: #1}\else
  \providecommand{\doi}{doi: \begingroup \urlstyle{rm}\Url}\fi

\bibitem[Anandkumar et~al.(2012)Anandkumar, Foster, Hsu, Kakade, and
  Liu]{AnaFosHsuKak:12LDA}
A.~Anandkumar, D.~P. Foster, D.~J. Hsu, S.~M. Kakade, and Y.-K. Liu.
\newblock A spectral algorithm for latent dirichlet allocation.
\newblock In \emph{Advances in Neural Information Processing Systems}, pages
  917--925, 2012.

\bibitem[Anandkumar et~al.(2014)Anandkumar, Ge, Hsu, Kakade, and
  Telgarsky]{AGHKT:14}
A.~Anandkumar, R.~Ge, D.~Hsu, S.~M. Kakade, and M.~Telgarsky.
\newblock Tensor decompositions for learning latent variable models.
\newblock \emph{The Journal of Machine Learning Research}, 15\penalty0
  (1):\penalty0 2773--2832, 2014.

\bibitem[Arbelaez et~al.(2011)Arbelaez, Maire, Fowlkes, and
  Malik]{ArbMaiFowMal:11}
P.~Arbelaez, M.~Maire, C.~Fowlkes, and J.~Malik.
\newblock Contour detection and hierarchical image segmentation.
\newblock \emph{IEEE Trans. Pattern Anal. Mach. Intell.}, 33\penalty0
  (5):\penalty0 898--916, May 2011.

\bibitem[Arora et~al.(2013)Arora, Ge, Halpern, Mimno, Moitra, Sontag, Wu, and
  Zhu]{AroGeHalMimMoi:12}
S.~Arora, R.~Ge, Y.~Halpern, D.~M. Mimno, A.~Moitra, D.~Sontag, Y.~Wu, and
  M.~Zhu.
\newblock A practical algorithm for topic modeling with provable guarantees.
\newblock In \emph{Proceedings of ICML-2013}, pages 280--288, 2013.

\bibitem[Basu et~al.(2002)Basu, Banerjee, and Mooney]{BasuBanMoo:02}
S.~Basu, A.~Banerjee, and R.~Mooney.
\newblock Semi-supervised clustering by seeding.
\newblock In \emph{Proceedings of 19th International Conference on Machine
  Learning (ICML-2002}, 2002.

\bibitem[Blei et~al.(2003)Blei, Ng, and Jordan]{BlNgJo:03}
D.~M. Blei, A.~Y. Ng, and M.~I. Jordan.
\newblock Latent dirichlet allocation.
\newblock \emph{Journal of machine Learning research}, 3:\penalty0 993--1022,
  2003.

\bibitem[Chapelle et~al.(2006)Chapelle, Sch{\"o}lkopf, and Zien]{ChScZi:06}
O.~Chapelle, B.~Sch{\"o}lkopf, and A.~Zien, editors.
\newblock \emph{Semi-supervised learning}.
\newblock MIT press Cambridge, 2006.

\bibitem[Chen et~al.(2014)Chen, Yi, and Caramanis]{ChYiCa:14}
Y.~Chen, X.~Yi, and C.~Caramanis.
\newblock A convex formulation for mixed regression with two components:
  Minimax optimal rates.
\newblock In \emph{COLT}, pages 560--604, 2014.

\bibitem[Dempster et~al.(1977)Dempster, Laird, and Rubin]{DeLaRu:77}
A.~P. Dempster, N.~M. Laird, and D.~B. Rubin.
\newblock Maximum likelihood from incomplete data via the em algorithm.
\newblock \emph{Journal of the royal statistical society. Series B
  (methodological)}, pages 1--38, 1977.

\bibitem[Elhamifar and Vidal(2009)]{ElhVid:09SSC}
E.~Elhamifar and R.~Vidal.
\newblock Sparse subspace clustering.
\newblock In \emph{Computer Vision and Pattern Recognition, 2009. CVPR 2009.
  IEEE Conference on}, pages 2790--2797. IEEE, 2009.

\bibitem[Hardt and Price(2015)]{HardtPrice:14}
M.~Hardt and E.~Price.
\newblock Tight bounds for learning a mixture of two gaussians.
\newblock In \emph{Proceedings of {ACM} on Symposium on Theory of Computing,
  {STOC}}, pages 753--760, 2015.

\bibitem[Hsu and Kakade(2013)]{HsuKakade:13}
D.~Hsu and S.~M. Kakade.
\newblock Learning mixtures of spherical gaussians: moment methods and spectral
  decompositions.
\newblock In \emph{Proceedings of the 4th conference on Innovations in
  Theoretical Computer Science}, pages 11--20. ACM, 2013.

\bibitem[Huang et~al.(2015)Huang, Niranjan, Hakeem, and
  Anandkumar]{HuaNirHakAna:15online}
F.~Huang, U.~Niranjan, M.~U. Hakeem, and A.~Anandkumar.
\newblock Online tensor methods for learning latent variable models.
\newblock \emph{Journal of Machine Learning Research}, 16:\penalty0 2797--2835,
  2015.

\bibitem[Kuusela and Ocone(2004)]{KuuselaOcone:04}
P.~Kuusela and D.~Ocone.
\newblock Learning with side information: Pac learning bounds.
\newblock \emph{Journal of Computer and System Sciences}, 68\penalty0
  (3):\penalty0 521--545, 2004.

\bibitem[Lu and Zhai(2008)]{LuZhai:08}
Y.~Lu and C.~Zhai.
\newblock Opinion integration through semi-supervised topic modeling.
\newblock In \emph{Proceedings of the 17th International Conference on World
  Wide Web}, pages 121--130. ACM, 2008.

\bibitem[Manning et~al.(2008)Manning, Raghavan, Sch{\"u}tze,
  et~al.]{ManningInfoRetieval:08}
C.~D. Manning, P.~Raghavan, H.~Sch{\"u}tze, et~al.
\newblock \emph{Introduction to information retrieval}, volume~1.
\newblock Cambridge university press Cambridge, 2008.

\bibitem[Mcauliffe and Blei(2008)]{BleiMcAuliffe:08}
J.~D. Mcauliffe and D.~M. Blei.
\newblock Supervised topic models.
\newblock In \emph{Advances in neural information processing systems}, pages
  121--128, 2008.

\bibitem[Moitra and Valiant(2010)]{MoitraValiant:10}
A.~Moitra and G.~Valiant.
\newblock Settling the polynomial learnability of mixtures of gaussians.
\newblock In \emph{Foundations of Computer Science (FOCS), 2010 51st Annual
  IEEE Symposium on}, pages 93--102. IEEE, 2010.

\bibitem[Newman et~al.(2010)Newman, Lau, Grieser, and
  Baldwin]{NewLauGriBald:10}
D.~Newman, J.~H. Lau, K.~Grieser, and T.~Baldwin.
\newblock Automatic evaluation of topic coherence.
\newblock In \emph{Human Language Technologies: The 2010 Annual Conf. of the
  North American Chapter of the Association for Computational Linguistics},
  pages 100--108. Association for Computational Linguistics, 2010.

\bibitem[Newman et~al.(2011)Newman, Bonilla, and Buntine]{NewBonBun:11}
D.~Newman, E.~V. Bonilla, and W.~Buntine.
\newblock Improving topic coherence with regularized topic models.
\newblock In \emph{Advances in neural information processing systems}, pages
  496--504, 2011.

\bibitem[Park et~al.(2014)Park, Caramanis, and Sanghavi]{ParkCarSan:14greedy}
D.~Park, C.~Caramanis, and S.~Sanghavi.
\newblock Greedy subspace clustering.
\newblock In \emph{Advances in Neural Information Processing Systems}, pages
  2753--2761, 2014.

\bibitem[Pearson(1894)]{Pearson:1894}
K.~Pearson.
\newblock Contributions to the mathematical theory of evolution.
\newblock \emph{Philosophical Transactions of the Royal Society of London. A},
  pages 71--110, 1894.

\bibitem[Ramage et~al.(2009)Ramage, Hall, Nallapati, and
  Manning]{RamHallNallMan09}
D.~Ramage, D.~Hall, R.~Nallapati, and C.~D. Manning.
\newblock Labeled lda: A supervised topic model for credit attribution in
  multi-labeled corpora.
\newblock In \emph{Proc. of the 2009 Conf. on Empirical Methods in Natural
  Language Processing: Volume 1-Volume 1}, pages 248--256. Association for
  Computational Linguistics, 2009.

\bibitem[Redner and Walker(1984)]{RednerWalker:84}
R.~A. Redner and H.~F. Walker.
\newblock Mixture densities, maximum likelihood and the em algorithm.
\newblock \emph{SIAM review}, 26\penalty0 (2):\penalty0 195--239, 1984.

\bibitem[R{\"o}der et~al.(2015)R{\"o}der, Both, and
  Hinneburg]{Roder:15Palmetto}
M.~R{\"o}der, A.~Both, and A.~Hinneburg.
\newblock Exploring the space of topic coherence measures.
\newblock In \emph{Proceedings of the eighth ACM international conference on
  Web search and data mining}, pages 399--408. ACM, 2015.

\bibitem[Rosen-Zvi et~al.(2004)Rosen-Zvi, Griffiths, Steyvers, and
  Smyth]{RCSS:04}
M.~Rosen-Zvi, T.~Griffiths, M.~Steyvers, and P.~Smyth.
\newblock In \emph{Proceedings of the 20th conference on Uncertainty in
  Artificial Intelligence}, pages 487--494, 2004.

\bibitem[Sedghi et~al.(2016)Sedghi, Janzamin, and Anandkumar]{SeJaAn:14}
H.~Sedghi, M.~Janzamin, and A.~Anandkumar.
\newblock Provable tensor methods for learning mixtures of generalized linear
  models.
\newblock In \emph{Proceedings of International Conference on Artificial
  Intelligence and Statistics, {AISTATS} 2016}, pages 1223--1231, 2016.

\bibitem[Soltanolkotabi and Candes(2012)]{SolCan:12GSC}
M.~Soltanolkotabi and E.~J. Candes.
\newblock A geometric analysis of subspace clustering with outliers.
\newblock \emph{The Annals of Statistics}, pages 2195--2238, 2012.

\bibitem[TACC(2018)]{TACC}
TACC.
\newblock Texas advanced computing center, 2018.
\newblock \url{http://www.tacc.utexas.edu}.

\bibitem[Tropp(2015)]{Tropp:15}
J.~Tropp.
\newblock An introduction to matrix concentration inequalities.
\newblock \emph{arXiv preprint arXiv:1501.01571}, 2015.

\bibitem[UCI(2008)]{NYTimesDataset:08}
UCI.
\newblock {NY Times} dataset, 2008.
\newblock \url{http://mlr.cs.umass.edu/ml/machine-learning-databases/}.

\bibitem[Xing et~al.(2002)Xing, Jordan, Russell, and Ng]{XJRN:02}
E.~P. Xing, M.~I. Jordan, S.~Russell, and A.~Y. Ng.
\newblock Distance metric learning with application to clustering with
  side-information.
\newblock In \emph{Advances in neural information processing systems}, pages
  505--512, 2002.

\bibitem[Yang et~al.(2010)Yang, Jin, and Jain]{YaJiJa:10}
T.~Yang, R.~Jin, and A.~K. Jain.
\newblock Learning from noisy side information by generalized maximum entropy
  model.
\newblock In \emph{Proceedings of the 27th International Conference on Machine
  Learning (ICML-10)}, pages 1199--1206, 2010.

\bibitem[Yelp(2014)]{YelpDataset:14}
Yelp.
\newblock Yelp dataset, 2014.
\newblock \url{http://www.yelp.com/dataset_challenge/}.

\bibitem[Yi et~al.(2014)Yi, Caramanis, and Sanghavi]{YiCaSa:13}
X.~Yi, C.~Caramanis, and S.~Sanghavi.
\newblock Alternating minimization for mixed linear regression.
\newblock In \emph{Proceedings of International Conference on Machine Learning,
  {ICML} 2014}, pages 613--621, 2014.

\end{thebibliography}

\clearpage


\begin{appendices}
\section{More Experiments for Gaussian Mixture Models} \label{app:synthetic_experiments}

In Figure \ref{fig:GMM_sensitivity} we show the sensitivity of the Whitening and Cancellation algorithms in GMM with $k=20,d=500,$ all equal probability components, and two different values of $\sigma$ and $n.$ Observe that the percentage error gain of the algorithms decreases with decreasing values of $\delta = \min_{i\neq 1} \frac{\langle \mu_1,v\rangle}{\langle \mu_i ,v\rangle},$ as we would expect, and it eventually becomes negative when the performance become worse than TPM algorithm. Also here the Cancellation algorithm shows lesser sensitivity, hence better performance compared to the Whitening algorithm.   

\begin{figure}[ht]
\centering
\subfigure[]{\includegraphics[height=1.9in,width=2.45in]{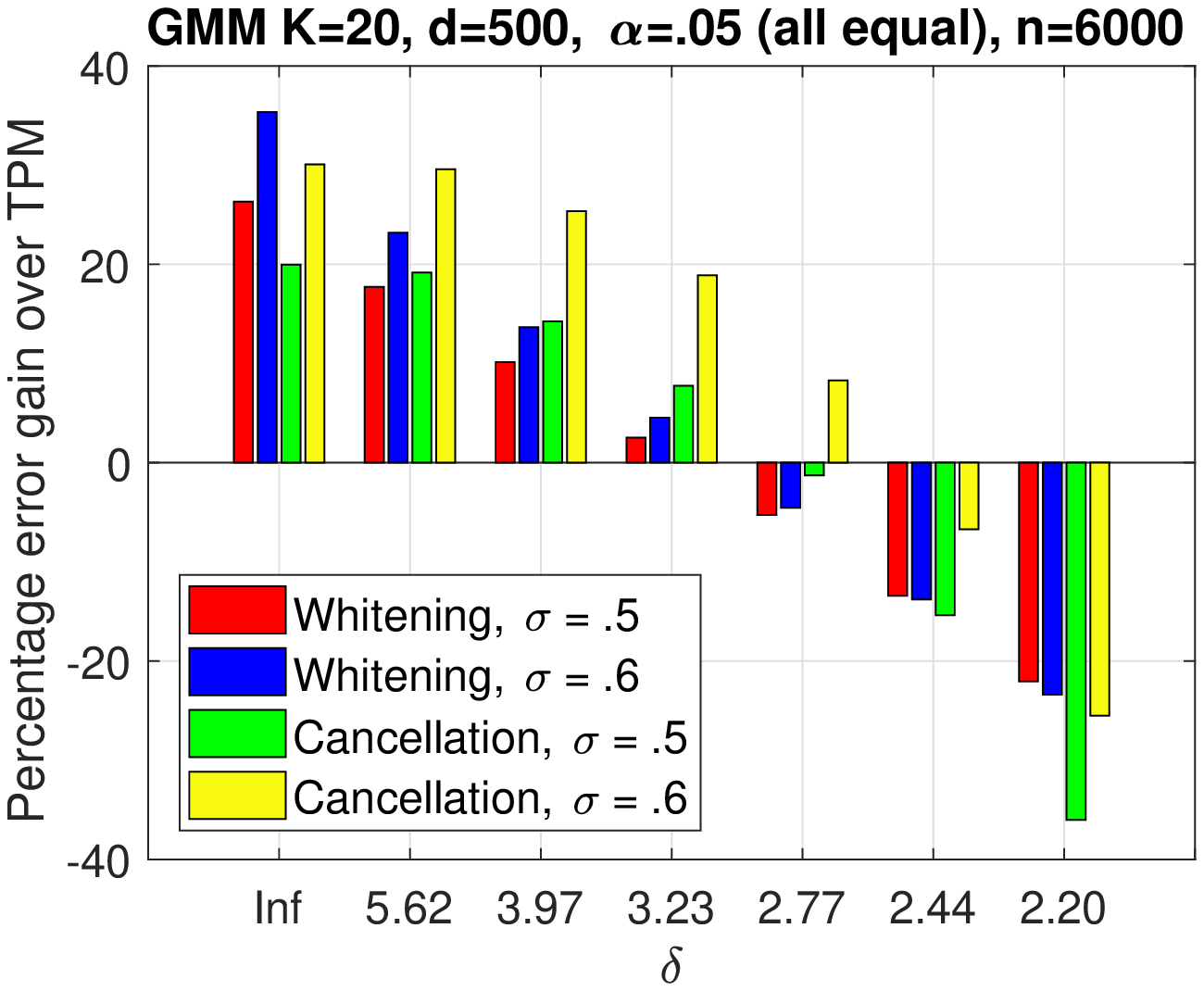}}
\subfigure[]{\includegraphics[height=1.9in,width=2.45in]{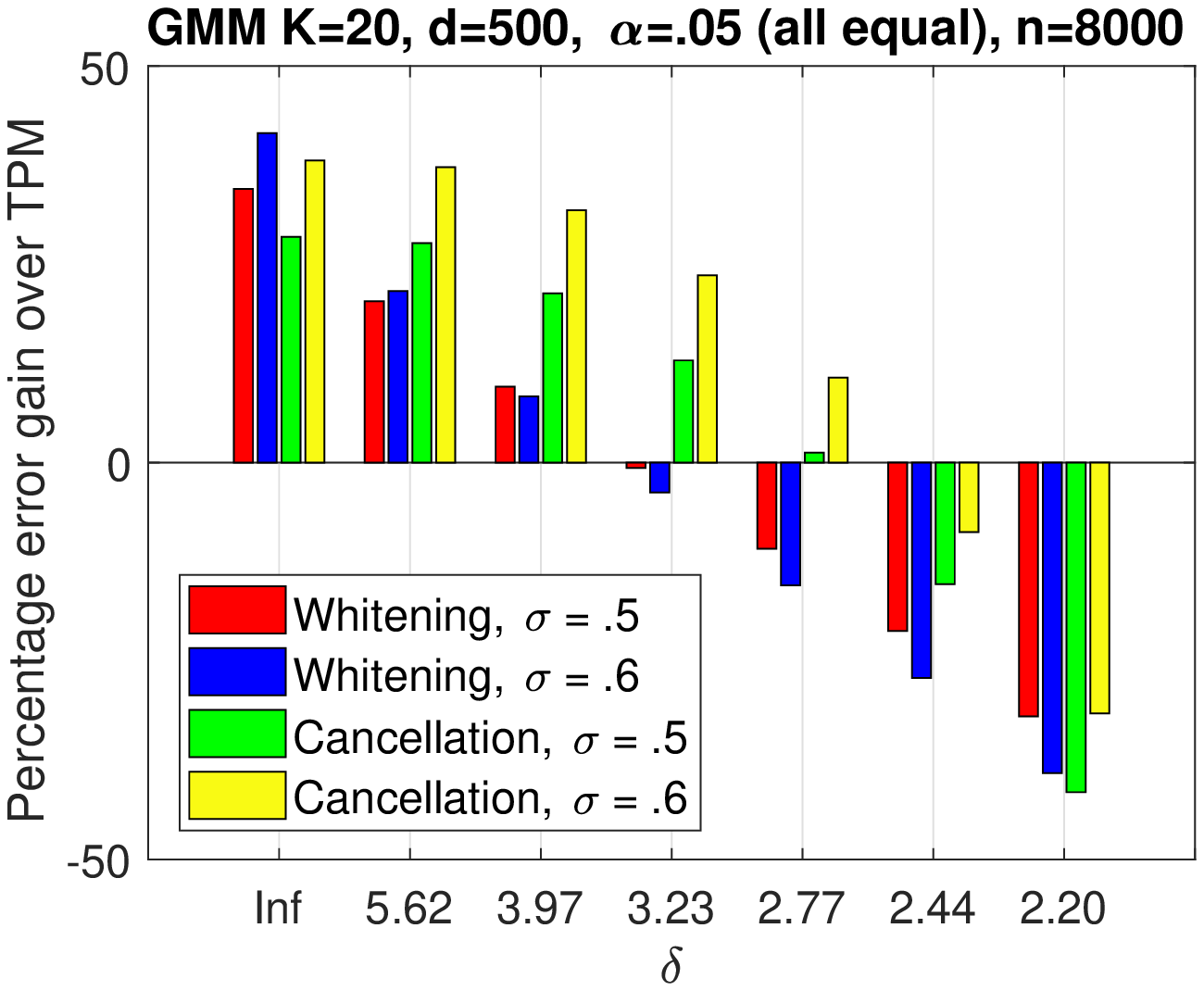}}
\caption{Sensitivity plots showing how the percentage relative error gain of the Whitening and Cancellation algorithms over the TPM algorithm decrease with decreasing values of the parameter $\delta = \min_{i\neq 1} \frac{\langle \mu_1,v\rangle}{\langle \mu_i ,v\rangle},$ in GMM with $k=20,d=500,$ all equal probability components, for different values of variance $\sigma \in \{.5,.6\},$ and two different sample complexities (a) $n=6000$ (b) $n=8000.$}
\label{fig:GMM_sensitivity}
\end{figure}

\section{Complete Results on New York Times and Yelp Data Set} \label{app:nytimes_yelp_full}

In this section we provide more detailed result of our experiments on NY Times and Yelp data sets. In Tables \ref{tab:nytimes}, \ref{tab:yelp} we show for every labeled word, the top five words in the topics computed by Whtening, NMF, and SS-NMF algorithms along with their corresponding PMI scores.

{\footnotesize
\begin{longtable}{ |p{1.5cm}|l|p{1.7cm}|p{1.7cm}|p{1.7cm}|p{1.7cm}|p{1.7cm}|l|}
\caption{Results of topic search by Whitening and NMF algorithms on NYtimes data set of $300,000$ news articles using $K=100$ topics and $62$ labeled words.}
\label{tab:nytimes}
\endfirsthead
\endhead
\hline
\multicolumn{8}{ |c| }{NY Times data set} \\
\hline
Label word & Algo & topword-1 & topword-2 & topword-3 & topword-4 & topword-5 & PMI \\ \hline
\multirow{3}{*}{passenger} & Whitening & flight & security & passenger & airport & hour & 0.1424\\
 & NMF & security & government & official & percent & bill & 0.0499\\
 & SSNMF & passenger & plane & flight & fire & crash & 0.1711\\ \hline
\multirow{3}{*}{coach} & Whitening & coach & season & job & team & head & 0.2637\\
 & NMF & team & coach & season & player & jet & 0.1740\\
 & SSNMF & coach & arrived & assistant & defenseman & ended & 0.1756\\ \hline
\multirow{3}{*}{art} & Whitening & information & question & today & eastern & daily & 0.0255\\
 & NMF & art & show & dessert & book & home & 0.0769\\
 & SSNMF & art & artist & show & painting & museum & 0.1250\\ \hline
\multirow{3}{*}{campaign} & Whitening & campaign & al gore & money & political & republican & 0.1530\\
 & NMF & al gore & campaign & george bush & president & bush & 0.1608\\
 & SSNMF & nra & florida & article & senator & presidential & 0.0926\\ \hline
\multirow{3}{*}{energy} & Whitening & corp & meeting & list & dividend & partial & 0.0815\\
 & NMF & corp & meeting & list & group & dividend & 0.0570\\
 & SSNMF & partial & energy & dividend & meeting & corp & 0.0254\\ \hline
\multirow{3}{*}{tax} & Whitening & tax & cut & taxes & percent & income & 0.2126\\
 & NMF & graf & president & bush & mail & information & 0.0722\\
 & SSNMF & tax & income & cut & taxes & site & 0.2279\\ \hline
\multirow{3}{*}{chef} & Whitening & cup & minutes & food & article & add & 0.0227\\
 & NMF & buy & panelist & flavor & thought & product & 0.0130\\
 & SSNMF & tobacco & chef & restaurant & pastry & article & 0.1495\\ \hline
\multirow{3}{*}{oil} & Whitening & oil & cup & minutes & prices & companies & 0.1460\\
 & NMF & oil & million & prices & percent & market & 0.0928\\
 & SSNMF & oil & company & listing & largest & brazil & 0.0902\\ \hline
\multirow{3}{*}{court} & Whitening & court & case & law & decision & lawyer & 0.2288\\
 & NMF & official & court & case & attack & government & 0.1285\\
 & SSNMF & chicago & court & decision & ruling & justices & 0.1834\\ \hline
\multirow{3}{*}{election} & Whitening & election & ballot & vote & voter & florida & 0.2132\\
 & NMF & election & ballot & al gore & bush & vote & 0.2155\\
 & SSNMF & gained & election & article & presidential & independence & 0.1702\\ \hline
\multirow{3}{*}{lawyer} & Whitening & case & court & lawyer & death & trial & 0.1830\\
 & NMF & official & court & case & attack & government & 0.1017\\
 & SSNMF & lawyer & rat & legal & client & jokes & 0.1314\\ \hline
\multirow{3}{*}{anthrax} & Whitening & mail & official & anthrax & attack & worker & 0.0600\\
 & NMF & anthrax & official & mail & worker & letter & 0.0156\\
 & SSNMF & anthrax & poverty & cb & show & return & -0.0776\\ \hline
\multirow{3}{*}{golf} & Whitening & tiger wood & shot & round & player & tour & 0.1288\\
 & NMF & tiger wood & shot & round & player & play & 0.1356\\
 & SSNMF & misstated & master & tee & hit & golf & 0.1356\\ \hline
\multirow{3}{*}{bacteria} & Whitening & mail & anthrax & official & test & found & -0.0763\\
 & NMF & anthrax & official & mail & worker & letter & -0.1097\\
 & SSNMF & mas & bacteria & con & una & anos & -0.2420\\ \hline
\multirow{3}{*}{film} & Whitening & film & movie & director & character & actor & 0.1906\\
 & NMF & article & misstated & new york & company & million & 0.0288\\
 & SSNMF & kiss & film & actress & article & role & 0.1295\\ \hline
\multirow{3}{*}{tourist} & Whitening & million & www & percent & building & night & 0.0481\\
 & NMF & team & tour & lance armstrong & won & race & -0.0405\\
 & SSNMF & tourist & million & visitor & official & campaign & 0.0995\\ \hline
\multirow{3}{*}{horse} & Whitening & race & won & win & run & track & 0.1129\\
 & NMF & race & won & horse & win & kentucky derby & 0.1338\\
 & SSNMF & horse & truck & road & official & killed & 0.0433\\ \hline
\multirow{3}{*}{republican} & Whitening & campaign & george bush & bush & election & republican & 0.2449\\
 & NMF & al gore & campaign & george bush & president & bush & 0.1868\\
 & SSNMF & republican & democrat & democratic & house & parties & 0.1053\\ \hline
\multirow{3}{*}{computer} & Whitening & computer & system & microsoft & program & software & 0.1904\\
 & NMF & company & computer & microsoft & system & companies & 0.1533\\
 & SSNMF & computer & chip & mail & program & buy & 0.1903\\ \hline
\multirow{3}{*}{palestinian} & Whitening & palestinian & israel & israeli & yasser arafat & peace & 0.2189\\
 & NMF & palestinian & israel & official & israeli & yasser arafat & 0.1950\\
 & SSNMF & palestinian & reformer & reform & authority & arab & 0.1519\\ \hline
\multirow{3}{*}{movie} & Whitening & film & movie & director & character & actor & 0.1492\\
 & NMF & film & show & actor & movie & thought & 0.0901\\
 & SSNMF & red sox & movie & interview & seattle & host & 0.0388\\ \hline
\multirow{3}{*}{tennis} & Whitening & player & play & won & game & women & 0.1054\\
 & NMF & game & play & player & point & andre agassi & 0.1187\\
 & SSNMF & motif & tennis & season & pros & image & 0.1480\\ \hline
\multirow{3}{*}{fight} & Whitening & won & night & fight & win & sport & 0.0566\\
 & NMF & fight & mike tyson & lennox lewis & million & round & 0.1181\\
 & SSNMF & fight & pound & fighter & beat & boxing & 0.1254\\ \hline
\multirow{3}{*}{music} & Whitening & music & song & record & album & band & 0.2298\\
 & NMF & music & company & million & companies & napster & 0.0812\\
 & SSNMF & music & mp3 & customer & digital & online & 0.0150\\ \hline
\multirow{3}{*}{tablespoon} & Whitening & cup & minutes & add & oil & tablespoon & 0.0608\\
 & NMF & cup & minutes & add & tablespoon & water & 0.0431\\
 & SSNMF & coffee & bean & tablespoon & cup & ground & -0.0765\\ \hline
\multirow{3}{*}{nuclear} & Whitening & bush & US & official & system & administration & 0.1223\\
 & NMF & official & bush & government & US & nuclear & 0.1356\\
 & SSNMF & ibm & nuclear & computer & research & fastest & -0.0253\\ \hline
\multirow{3}{*}{racing} & Whitening & race & car & driver & team & season & 0.1443\\
 & NMF & car & race & driver & team & season & 0.1319\\
 & SSNMF & sport & file & los angeles & racing & notebook & -0.0640\\ \hline
\multirow{3}{*}{war} & Whitening & military & taliban & war & afghanistan & us & 0.0916\\
 & NMF & taliban & official & afghanistan & government & us & 0.0796\\
 & SSNMF & russian & war & chechnya & army & veteran & 0.1296\\ \hline
\multirow{3}{*}{quarterback} & Whitening & yard & season & game & play & team & 0.2389\\
 & NMF & game & team & play & yard & season & 0.1773\\
 & SSNMF & effort & quarterback & ucla & heroic & alabama & 0.1472\\ \hline
\multirow{3}{*}{stock} & Whitening & stock & market & percent & company & fund & 0.1585\\
 & NMF & percent & stock & market & company & companies & 0.1338\\
 & SSNMF & stock & market & price & shares & investment & 0.0507\\ \hline
\multirow{3}{*}{ball} & Whitening & game & run & yard & play & hit & 0.1782\\
 & NMF & run & game & inning & hit & season & 0.1361\\
 & SSNMF & ball & hit & run & inning & home & 0.1708\\ \hline
\multirow{3}{*}{patient} & Whitening & patient & doctor & care & health & drug & 0.2532\\
 & NMF & official & virus & percent & new york & found & 0.1003\\
 & SSNMF & patient & study & doctor & article & brain & 0.1334\\ \hline
\multirow{3}{*}{champion} & Whitening & won & win & round & shot & tiger wood & 0.1029\\
 & NMF & fight & mike tyson & lennox lewis & million & round & 0.0955\\
 & SSNMF & olympic & champion & final & meet & medalist & 0.1177\\ \hline
\multirow{3}{*}{business} & Whitening & business & company & question & information & companies & 0.0887\\
 & NMF & information & eastern & commentary & daily & business & 0.0311\\
 & SSNMF & publication & business & send & released & businesses & 0.0996\\ \hline
\multirow{3}{*}{government} & Whitening & government & official & country & federal & political & 0.1524\\
 & NMF & graf & president & bush & mail & information & 0.0767\\
 & SSNMF & program & government & computer & local & newspaper & 0.0784\\ \hline
\multirow{3}{*}{season} & Whitening & season & team & game & games & play & 0.1799\\
 & NMF & team & game & season & play & games & 0.1406\\
 & SSNMF & season & cotton & fact & simple & variety & 0.0626\\ \hline
\multirow{3}{*}{prison} & Whitening & death & case & lawyer & court & trial & 0.1333\\
 & NMF & advise & spot & earlier & held & today & -0.0340\\
 & SSNMF & prison & inmates & security & population & bed & 0.1472\\ \hline
\multirow{3}{*}{internet} & Whitening & file & spot & internet & read & output & 0.0359\\
 & NMF & file & spot & new york & sport & los angeles & 0.0228\\
 & SSNMF & wonderful & mail & al gore & george bush & message & 0.0766\\ \hline
\multirow{3}{*}{rain} & Whitening & air & part & high & wind & rain & 0.1963\\
 & NMF & air & wind & shower & rain & storm & 0.1939\\
 & SSNMF & chicago sun times & nominated & rain & east & thought & 0.0179\\ \hline
\multirow{3}{*}{game} & Whitening & game & team & play & games & season & 0.2000\\
 & NMF & team & game & season & play & games & 0.1722\\
 & SSNMF & covering & game & tonight & coverage & celebration & 0.0531\\ \hline
\multirow{3}{*}{voter} & Whitening & election & ballot & vote & percent & voter & 0.2068\\
 & NMF & election & ballot & al gore & bush & vote & 0.1870\\
 & SSNMF & voter & poll & percent & primary & election & 0.2067\\ \hline
\multirow{3}{*}{baseball} & Whitening & player & team & season & game & sport & 0.1691\\
 & NMF & team & chicago white sox & mariner & season & player & 0.1803\\
 & SSNMF & velocity & baseball & air & shot & test & 0.0629\\ \hline
\multirow{3}{*}{student} & Whitening & student & school & teacher & percent & program & 0.2077\\
 & NMF & test & school & student & ignore & export & 0.0729\\
 & SSNMF & student & university & shooting & shot & rampage & 0.1396\\ \hline
\multirow{3}{*}{president} & Whitening & president & vice & white house & george bush & executive & 0.2116\\
 & NMF & graf & president & bush & mail & information & 0.0758\\
 & SSNMF & hedge & president & television & broadway & produced & 0.0226\\ \hline
\multirow{3}{*}{afghan} & Whitening & taliban & afghanistan & military & us & war & 0.1684\\
 & NMF & taliban & official & afghanistan & government & us & 0.1413\\
 & SSNMF & afghan & afghanistan & blanket & friend & country & 0.0577\\ \hline
\multirow{3}{*}{medal} & Whitening & team & games & won & women & american & 0.1822\\
 & NMF & team & tour & lance armstrong & won & race & 0.0348\\
 & SSNMF & endit & medal & honor & winner & newspaper & 0.0786\\ \hline
\multirow{3}{*}{teacher} & Whitening & school & student & teacher & high & program & 0.1566\\
 & NMF & test & school & student & ignore & export & 0.0388\\
 & SSNMF & teacher & program & pay & school & teaching & 0.1499\\ \hline
\multirow{3}{*}{television} & Whitening & show & home & network & television & night & 0.1721\\
 & NMF & los angeles daily new & spot & newspaper & new york & show & 0.1456\\
 & SSNMF & clinton & home & television & survived & tonight & -0.0090\\ \hline
\multirow{3}{*}{democratic} & Whitening & al gore & campaign & election & political & republican & 0.1837\\
 & NMF & al gore & campaign & george bush & president & bush & 0.1677\\
 & SSNMF & environmental & democratic & national committee & nominee & fund & 0.0813\\ \hline
\multirow{3}{*}{onion} & Whitening & cup & minutes & add & oil & tablespoon & 0.1039\\
 & NMF & cup & minutes & add & tablespoon & water & 0.1072\\
 & SSNMF & flavor & panelist & ounces & buy & onion & 0.1188\\ \hline
\multirow{3}{*}{campus} & Whitening & student & school & college & teacher & program & 0.1314\\
 & NMF & game & season & team & play & coach & -0.0595\\
 & SSNMF & campus & operation & aol & building & center & 0.0645\\ \hline
\multirow{3}{*}{car} & Whitening & car & driver & race & racing & seat & 0.2047\\
 & NMF & car & race & driver & team & season & 0.1222\\
 & SSNMF & car & team & race & driver & winston cup & 0.1516\\ \hline
\multirow{3}{*}{industry} & Whitening & companies & percent & company & business & industry & 0.1430\\
 & NMF & music & company & million & companies & napster & 0.0821\\
 & SSNMF & xxx & show & trade & software & entertainment & 0.1161\\ \hline
\multirow{3}{*}{planet} & Whitening & film & today & system & movie & team & -0.0054\\
 & NMF & wire & inadvertently & kill & mandatory & today & -0.0750\\
 & SSNMF & captor & planet & film & kill & astronomer & 0.0949\\ \hline
\multirow{3}{*}{credit} & Whitening & bill & money & member & system & number & 0.1257\\
 & NMF & bill & tax & bush & member & percent & 0.0287\\
 & SSNMF & donation & card & credit & account & voted & 0.1382\\ \hline
\multirow{3}{*}{race} & Whitening & race & car & driver & won & win & 0.1917\\
 & NMF & car & race & driver & team & season & 0.1814\\
 & SSNMF & amazing & race & show & tonight & sit & 0.0502\\ \hline
\multirow{3}{*}{wine} & Whitening & cup & minutes & food & add & oil & 0.0499\\
 & NMF & wine & wines & percent & company & million & 0.0748\\
 & SSNMF & wine & wines & bottle & bottles & age & 0.1082\\ \hline
\multirow{3}{*}{prosecutor} & Whitening & case & death & lawyer & court & trial & 0.1952\\
 & NMF & official & court & case & attack & government & 0.1363\\
 & SSNMF & prosecutor & lawyer & attorney & incorrectly & general & 0.1406\\ \hline
\multirow{3}{*}{team} & Whitening & team & season & game & player & play & 0.1654\\
 & NMF & team & game & season & play & games & 0.1558\\
 & SSNMF & team & qualify & olympic & article & member & 0.1530\\ \hline
\multirow{3}{*}{economy} & Whitening & percent & market & economy & stock & cut & 0.1528\\
 & NMF & percent & stock & market & company & companies & 0.1048\\
 & SSNMF & percent & economy & quarter & rate & recession & 0.1452\\ \hline
\multirow{3}{*}{wind} & Whitening & air & high & part & wind & rain & 0.1909\\
 & NMF & air & wind & shower & rain & storm & 0.1895\\
 & SSNMF & wash & wind & school & winter & white & 0.1902\\ \hline
\multirow{3}{*}{software} & Whitening & microsoft & computer & system & company & software & 0.1981\\
 & NMF & company & computer & microsoft & system & companies & 0.1911\\
 & SSNMF & xxx & software & industry & show & trade & 0.1222\\ \hline
\hline
\end{longtable}
}

{\footnotesize
\begin{longtable}{ |p{1.5cm}|l|p{1.7cm}|p{1.7cm}|p{1.7cm}|p{1.7cm}|p{1.7cm}|l|}
\caption{Results of topic search by Whitening and NMF algorithms on Yelp data set of $335,022$ reviews of businesses using $K=100$ topics and $54$ labeled words.}
\label{tab:yelp}
\endfirsthead
\endhead
\hline
\multicolumn{8}{ |c| }{Yelp data set} \\
\hline
Label word & Algo & topword-1 & topword-2 & topword-3 & topword-4 & topword-5 & PMI \\ \hline
\multirow{3}{*}{cheese} & Whitening & cheese & pizza & time & sandwich & back & 0.1842\\
 & NMF & bagel & coffee & bagels & cheese & sandwich & 0.1666\\
 & SSNMF & bartender & cheese & tasty & made & server & 0.0555\\ \hline
\multirow{3}{*}{salon} & Whitening & hair & salon & nails & nail & back & 0.0678\\
 & NMF & hair & absolute & cut & beautiful & salon & -0.0192\\
 & SSNMF & salon & manicure & back & nail & clean & 0.0375\\ \hline
\multirow{3}{*}{mexican} & Whitening & mexican & burrito & tacos & salsa & cheese & 0.0506\\
 & NMF & mexican & fresh & burrito & tacos & time & 0.0389\\
 & SSNMF & exit & mexican & bland & restaurants & world & -0.0720\\ \hline
\multirow{3}{*}{chinese} & Whitening & chicken & chinese & rice & hot & fast & 0.0978\\
 & NMF & chicken & chinese & fast & rice & time & 0.0717\\
 & SSNMF & chinese & area & type & lot & east & 0.0455\\ \hline
\multirow{3}{*}{tea} & Whitening & coffee & find & things & tea & starbucks & 0.1079\\
 & NMF & find & store & things & tea & oil & 0.0470\\
 & SSNMF & tea & coffee & starbucks & safeway & ice & 0.1787\\ \hline
\multirow{3}{*}{sushi} & Whitening & sushi & roll & happy & rolls & fish & 0.0330\\
 & NMF & cooks & fun & hash & browns & reasonable & -0.0441\\
 & SSNMF & 2nd & sushi & time & location & amazing & -0.1112\\ \hline
\multirow{3}{*}{nail} & Whitening & nails & nail & pedicure & salon & time & 0.1385\\
 & NMF & nails & nail & pedicure & time & salon & 0.1316\\
 & SSNMF & nail & nails & grandma & cut & make & 0.0658\\ \hline
\multirow{3}{*}{wash} & Whitening & car & wash & clean & time & job & 0.0617\\
 & NMF & car & wash & back & time & job & 0.0583\\
 & SSNMF & car & wash & feels & clean & time & 0.0290\\ \hline
\multirow{3}{*}{insurance} & Whitening & years & business & office & recommend & family & 0.0856\\
 & NMF & office & work & walk & time & insurance & 0.0189\\
 & SSNMF & insurance & years & business & steve & saved & 0.0459\\ \hline
\multirow{3}{*}{cream} & Whitening & ice & cream & chocolate & cold & wait & 0.1739\\
 & NMF & ice & cream & school & cone & kids & 0.1111\\
 & SSNMF & cream & ice & wait & stone & cold & 0.1494\\ \hline
\multirow{3}{*}{hair} & Whitening & hair & beautiful & absolute & years & salon & 0.0749\\
 & NMF & hair & absolute & cut & beautiful & salon & 0.0507\\
 & SSNMF & beautiful & hair & years & cut & time & 0.0532\\ \hline
\multirow{3}{*}{yoga} & Whitening & classes & class & yoga & studio & gym & 0.0928\\
 & NMF & yoga & classes & class & studio & time & 0.0816\\
 & SSNMF & yoga & practice & dave & feel & amazing & 0.0391\\ \hline
\multirow{3}{*}{tire} & Whitening & tire & tires & oil & car & discount & 0.0739\\
 & NMF & tire & car & tires & back & time & 0.0634\\
 & SSNMF & tire & tires & car & discount & time & 0.0274\\ \hline
\multirow{3}{*}{vietnamese} & Whitening & time & chicken & thai & rice & chinese & -0.0442\\
 & NMF & pho & chicken & rice & sauce & back & 0.0825\\
 & SSNMF & vietnamese & cake & chinese & back & fresh & -0.0105\\ \hline
\multirow{3}{*}{donuts} & Whitening & donuts & fresh & coffee & donut & chocolate & -0.0349\\
 & NMF & donuts & coffee & donut & store & location & -0.0040\\
 & SSNMF & donuts & donut & chocolate & time & selection & -0.1298\\ \hline
\multirow{3}{*}{crust} & Whitening & pizza & crust & wings & sauce & cheese & 0.0068\\
 & NMF & pizza & crust & wings & time & cheese & -0.0503\\
 & SSNMF & min & pizza & crust & hut & pretty & -0.1131\\ \hline
\multirow{3}{*}{ice} & Whitening & ice & cream & cold & chocolate & flavors & 0.1234\\
 & NMF & ice & cream & school & cone & kids & 0.0718\\
 & SSNMF & ice & cream & wait & stone & cold & 0.1312\\ \hline
\multirow{3}{*}{pharmacy} & Whitening & store & location & big & feel & kids & 0.0075\\
 & NMF & store & time & location & pharmacy & helpful & 0.0049\\
 & SSNMF & pharmacy & customer & clean & safeway & rude & -0.0127\\ \hline
\multirow{3}{*}{beer} & Whitening & bar & time & beer & wings & drinks & 0.0900\\
 & NMF & pizza & brick & pretty & bar & box & -0.0190\\
 & SSNMF & beers & beer & operated & hand & locally & 0.0817\\ \hline
\multirow{3}{*}{bike} & Whitening & bike & shop & guys & tires & back & 0.0053\\
 & NMF & bike & shop & back & bikes & time & 0.0525\\
 & SSNMF & bike & time & gun & pretty & store & -0.0293\\ \hline
\multirow{3}{*}{yogurt} & Whitening & yogurt & flavors & toppings & frozen & chocolate & 0.0659\\
 & NMF & yogurt & flavors & toppings & frozen & chocolate & 0.0420\\
 & SSNMF & yogurt & flavors & back & ice & shop & -0.1370\\ \hline
\multirow{3}{*}{korean} & Whitening & sushi & chinese & time & fresh & rice & -0.0311\\
 & NMF & magazine & market & farmer & farmers & boston & -0.0702\\
 & SSNMF & korean & chicken & pretty & fried & spicy & 0.0376\\ \hline
\multirow{3}{*}{pizza} & Whitening & pizza & crust & wings & time & cheese & 0.1491\\
 & NMF & pizza & brick & pretty & bar & box & 0.0582\\
 & SSNMF & pizza & ride & brick & long & red & 0.0518\\ \hline
\multirow{3}{*}{coffee} & Whitening & coffee & starbucks & donuts & tea & time & 0.2728\\
 & NMF & coffee & busy & starbucks & ice & cream & 0.2613\\
 & SSNMF & coffee & starbucks & drinks & latte & work & 0.0974\\ \hline
\multirow{3}{*}{sandwich} & Whitening & sandwich & subway & sandwiches & bread & time & 0.1714\\
 & NMF & sandwich & subway & fresh & bread & location & 0.1311\\
 & SSNMF & sandwich & sandwiches & ham & chips & limited & 0.0083\\ \hline
\multirow{3}{*}{pho} & Whitening & time & thai & rice & sauce & back & -0.2046\\
 & NMF & pho & chicken & rice & sauce & back & -0.1096\\
 & SSNMF & pho & rice & beef & vietnamese & sauce & -0.0911\\ \hline
\multirow{3}{*}{gym} & Whitening & classes & class & work & gym & yoga & 0.1518\\
 & NMF & link & open & isn & working & fast & -0.0304\\
 & SSNMF & gym & fitness & work & open & time & 0.1117\\ \hline
\multirow{3}{*}{park} & Whitening & dog & park & dogs & area & kids & 0.1099\\
 & NMF & park & dog & time & area & trail & 0.1023\\
 & SSNMF & park & dog & dogs & lake & area & 0.1303\\ \hline
\multirow{3}{*}{latte} & Whitening & coffee & starbucks & drink & time & make & -0.1617\\
 & NMF & coffee & busy & starbucks & ice & cream & 0.0802\\
 & SSNMF & latte & location & work & drink & drinks & -0.0539\\ \hline
\multirow{3}{*}{trail} & Whitening & park & area & phoenix & time & lot & 0.1356\\
 & NMF & park & dog & time & area & trail & 0.1049\\
 & SSNMF & trail & parking & street & major & easy & 0.0267\\ \hline
\multirow{3}{*}{dentist} & Whitening & office & years & dentist & experience & work & 0.0734\\
 & NMF & office & dentist & time & work & years & 0.1169\\
 & SSNMF & dentist & office & insurance & made & teeth & 0.0766\\ \hline
\multirow{3}{*}{starbucks} & Whitening & starbucks & drink & coffee & drinks & times & -0.0972\\
 & NMF & coffee & busy & starbucks & ice & cream & -0.0477\\
 & SSNMF & starbucks & drink & argue & smile & times & -0.1099\\ \hline
\multirow{3}{*}{taco} & Whitening & taco & bell & tacos & fast & sauce & 0.0994\\
 & NMF & mexican & fresh & burrito & tacos & time & 0.1875\\
 & SSNMF & taco & bell & ghetto & pizza & location & -0.0042\\ \hline
\multirow{3}{*}{salsa} & Whitening & mexican & burrito & tacos & salsa & fresh & 0.0887\\
 & NMF & mexican & fresh & burrito & tacos & time & 0.0267\\
 & SSNMF & salsa & fresh & tacos & baja & fish & -0.0697\\ \hline
\multirow{3}{*}{thai} & Whitening & thai & rice & chinese & hot & chicken & 0.0691\\
 & NMF & thai & chicken & rice & back & sauce & 0.1164\\
 & SSNMF & thai & pad & tea & dish & green & 0.0275\\ \hline
\multirow{3}{*}{chocolate} & Whitening & yogurt & flavors & chocolate & cream & ice & 0.1923\\
 & NMF & gelato & flavors & chocolate & ice & cream & 0.1641\\
 & SSNMF & chocolate & caramel & factory & dark & covered & 0.1943\\ \hline
\multirow{3}{*}{bar} & Whitening & bar & drinks & night & time & beer & 0.0142\\
 & NMF & pizza & brick & pretty & bar & box & -0.0143\\
 & SSNMF & bar & bit & big & seating & beer & -0.0086\\ \hline
\multirow{3}{*}{noodle} & Whitening & chicken & chinese & rice & thai & sauce & 0.2423\\
 & NMF & pho & chicken & rice & sauce & back & 0.2630\\
 & SSNMF & chicken & noodle & rice & back & sauces & 0.0910\\ \hline
\multirow{3}{*}{burrito} & Whitening & burrito & mexican & stars & tacos & salsa & 0.1320\\
 & NMF & mexican & fresh & burrito & tacos & time & 0.0638\\
 & SSNMF & stars & burrito & green & sauce & mexican & 0.0467\\ \hline
\multirow{3}{*}{salad} & Whitening & salad & chicken & fresh & sandwich & bar & 0.1780\\
 & NMF & pizza & brick & pretty & bar & box & -0.0220\\
 & SSNMF & salad & bar & salads & soup & competitors & -0.0123\\ \hline
\multirow{3}{*}{burger} & Whitening & burger & fries & burgers & fast & time & 0.1489\\
 & NMF & link & open & isn & working & fast & 0.0159\\
 & SSNMF & stale & burger & meat & bite & king & 0.0322\\ \hline
\multirow{3}{*}{hike} & Whitening & park & area & time & lot & back & 0.0572\\
 & NMF & park & dog & time & area & trail & 0.0747\\
 & SSNMF & hike & park & rock & mountain & water & 0.1255\\ \hline
\multirow{3}{*}{pedicure} & Whitening & nails & nail & pedicure & job & salon & 0.0189\\
 & NMF & nails & nail & pedicure & time & salon & 0.0158\\
 & SSNMF & pedicure & job & nail & close & home & -0.0931\\ \hline
\multirow{3}{*}{fries} & Whitening & burger & fries & burgers & fast & cheese & -0.0413\\
 & NMF & cut & wait & time & hair & manager & -0.2616\\
 & SSNMF & fries & grease & dirty & dark & slow & -0.1629\\ \hline
\multirow{3}{*}{dog} & Whitening & dog & dogs & park & pet & hot & 0.1501\\
 & NMF & dog & tony & cut & dogs & style & 0.0751\\
 & SSNMF & dog & door & tie & made & serve & 0.0080\\ \hline
\multirow{3}{*}{panda} & Whitening & chicken & fast & chinese & rice & time & -0.1488\\
 & NMF & chicken & chinese & fast & rice & time & -0.1291\\
 & SSNMF & panda & orange & rice & fried & bad & -0.1327\\ \hline
\multirow{3}{*}{beans} & Whitening & mexican & burrito & chicken & tacos & salsa & -0.0550\\
 & NMF & mexican & fresh & burrito & tacos & time & -0.1419\\
 & SSNMF & trouble & beans & rice & chicken & marinated & -0.1233\\ \hline
\multirow{3}{*}{subway} & Whitening & subway & sandwich & clean & fresh & location & -0.0074\\
 & NMF & sandwich & subway & fresh & bread & location & -0.0445\\
 & SSNMF & subway & location & clean & super & sandwich & -0.0524\\ \hline
\multirow{3}{*}{car} & Whitening & car & wash & back & time & work & 0.1064\\
 & NMF & car & wash & back & time & job & 0.0874\\
 & SSNMF & visited & car & back & job & weeks & 0.0353\\ \hline
\multirow{3}{*}{cake} & Whitening & found & cake & chocolate & shop & yogurt & 0.0754\\
 & NMF & back & time & shop & cake & found & 0.0099\\
 & SSNMF & cake & wanted & wedding & flavor & perfect & 0.0416\\ \hline
\multirow{3}{*}{steak} & Whitening & location & fast & makes & feel & quality & -0.0672\\
 & NMF & prices & selection & quality & family & helpful & -0.1569\\
 & SSNMF & difference & fast & steak & sandwiches & subs & -0.1672\\ \hline
\multirow{3}{*}{curry} & Whitening & thai & chicken & rice & chinese & hot & 0.1482\\
 & NMF & thai & chicken & rice & back & sauce & 0.1903\\
 & SSNMF & chicken & stew & brown & curry & rice & 0.0047\\ \hline
\multirow{3}{*}{massage} & Whitening & massage & back & amazing & years & spa & 0.1359\\
 & NMF & massage & time & back & amazing & hour & -0.0035\\
 & SSNMF & massage & arts & experience & amazing & hour & -0.0168\\ \hline
\multirow{3}{*}{italian} & Whitening & sandwich & pizza & time & back & bread & -0.0254\\
 & NMF & gelato & flavors & chocolate & ice & cream & 0.0241\\
 & SSNMF & ice & italian & flavors & cream & chocolate & -0.0231\\ \hline
\hline
\end{longtable}
}


\section{Computation of $A,B$ for Different Models} \label{app:moments}

This section outlines the construction of matrices $A,B$ in various models via different moment computations. First we introduce some notations which we use in Appendices \ref{app:moments}, \ref{sec:finite-whitening}, \ref{sec:finite-cancellation}, \ref{app:subspace_proof}, and \ref{app:concentration}.

\subsection{Notations}

For a vector $x,$ $\|x\|$ denotes its $\ell_2$ norm. For a matrix $X,$ $\|X\|$ represents the spectral norm of the matrix. We use the notation $\widehat{X}$ or $\widehat \E [X]$ to represent the sample estimate of a quantity $X,$ unless mentioned otherwise. For a matrix $M$ let $\sigma_k(M)$ denote the $k-$th largest singular value of $M,$ and $\tilde{\sigma}_k(M)$ denote the $k-$th largest eigenvalue. $n$ represents the number of samples used to obtain the sample estimates. Next, we introduce some basic tensor notations. Let $x,y,z \in \mathbb{R}^d$ be three $d$ dimensional vectors. Then the order-$3$ tensor $T_3=x \otimes y \otimes z$ is defined as $T_3(i,j,k) = x(i) y(j) z(k),$ for $i,j,k \in [d].$ Similarly the order-$2$ tensor $T_2 = x \otimes y$ is equivalent to the matrix outer product $T_2 = x y^T.$ Finally let $v \in \mathbb{R}^d$ be another $d$ dimensional vector, $I$ be the $d$ dimensional identity matrix. The tensor contraction $T_3(I,I,v)$ is equal to the order-$2$ tensor $T_3(I,I,v) = \langle z,v\rangle x \otimes y,$ which is again equivalent to the matrix $T_3(I,I,v) = \langle z,v\rangle xy^T.$ For order-$2$ tensors we will use the tensor and matrix notations interchangeably.

\subsection{GMM Moments} \label{app:GMM_moments}

In this section we prove how the required matrices $A,B$ can be computed in the GMM model. We restate the following useful theorem from \cite{HsuKakade:13} which computes three tensor moments for the GMM model.   

\begin{theorem}[\cite{HsuKakade:13}] \label{thm:hsu_gmm_moments}
  Consider the GMM model with means $\{\mu_1 , \hdots , \mu_k\}$ and corresponding variances $\{\sigma_1^2 , \hdots , \sigma_k^2\},$ and $\alpha_i$ denote the proportion of the $i$-th component in the mixture. Let $\sigma^2 = \sum_{i=1}^k \alpha_i \sigma_i^2$ be the smallest eigenvalue of the covariance matrix $\mathbb{E}[(x-\mathbb{E}[x])(x-\mathbb{E}[x])^T]$ (
  note that since $\sum \alpha_i \mu_i \mu_i^T$ has rank $k$,
this is the same as the $k+1$th-largest eigenvalue),
and $u$ be a unit norm eigenvector corresponding to the eigenvalue $\sigma^2.$ Define
\begin{eqnarray*}
\widetilde{m} &=& \mathbb{E}[x(u^T(x-\mathbb{E}[x]))^2], \ \ M_2 = \mathbb{E}[x \otimes x] - \sigma^2 I \\
M_3 &=& \mathbb{E}[x \otimes x \otimes x] - \sum_{i=1}^d (\widetilde{m} \otimes e_i \otimes e_i + e_i \otimes \widetilde{m} \otimes e_i + e_i \otimes e_i \otimes \widetilde{m})
\end{eqnarray*}

where $\{e_1 , \hdots , e_d\}$ form standard basis of $\mathbb{R}^d.$ Then,

\begin{equation*}
\widetilde{m} = \sum_{i=1}^k \alpha_i \sigma_i^2 \mu_i, \ \ M_2 = \sum_{i=1}^k \alpha_i \mu_i \otimes \mu_i, \ \ M_3=\sum_{i=1}^k \alpha_i \mu_i \otimes \mu_i \otimes \mu_i.  
\end{equation*}
\end{theorem}

\begin{theorem}
In the GMM model define
\begin{eqnarray*}
m &=& \mathbb{E}[x], \ \ A = \mathbb{E}[x x^T] - \sigma^2 I_d \\
B &=& \mathbb{E}[\langle x, v\rangle x x^T] - \widetilde{m} v^T - v \widetilde{m}^T - \langle \widetilde{m},v\rangle I_d
\end{eqnarray*}

Then, $m=\sum_i \alpha_i \mu_i,$ $A=\sum_{i=1}^k \alpha_i \mu_i \mu_i^T$ and $B=\sum_{i=1}^k \alpha_i \langle \mu_i,v \rangle \mu_i \mu_i^T$

\end{theorem}
\begin{proof}
The expression for $m,$ $A$ follows directly from Theorem \ref{thm:hsu_gmm_moments} by noting that $A=M_2$ and $\mu_i \otimes \mu_i = \mu_i \mu_i^T.$ To compute $B$ consider the tensor contraction $M_3(I,I,v),$ $M_3$ as in Theorem \ref{thm:hsu_gmm_moments}. Then,
\begin{eqnarray*}
M_3(I,I,v) &=& \mathbb{E}[\langle x,v \rangle x \otimes x] - \sum_{i=1}^d (v(i) \widetilde{m} \otimes e_i + v(i) e_i \otimes \widetilde{m} + \langle \widetilde{m},v\rangle e_i \otimes e_i) \\
&=& \mathbb{E}[\langle x,v \rangle x x^T] - \sum_{i=1}^d (v(i) \widetilde{m} e_i^T + v(i) e_i \widetilde{m}^T + \langle \widetilde{m},v\rangle e_i e_i^T) \\
&=& \mathbb{E}[\langle x,v \rangle x x^T] - \widetilde{m} v^T - v \widetilde{m}^T - \langle \widetilde{m},v\rangle I_d = B
\end{eqnarray*}

Also from Theorem \ref{thm:hsu_gmm_moments}, $M_3(I,I,v) = \sum_{i=1}^k \alpha_i \langle \mu_i,v \rangle \mu_i \otimes \mu_i = \sum_{i=1}^k \alpha_i \langle \mu_i,v \rangle \mu_i \mu_i^T.$ Therefore $B=\sum_{i=1}^k \alpha_i \langle \mu_i,v \rangle \mu_i \mu_i^T.$
\end{proof}

\subsection{LDA Moments}

In this section we show the $m,A,B$ computation corresponding to the LDA model. Again we restate the following theorem from \cite{AGHKT:14} which computes the first three tensor moments for LDA distribution.

\begin{theorem}[\cite{AGHKT:14}] \label{thm:anandkumar_lda_moments}
In an LDA model with parameters $\bar{\alpha}=\left(\alpha_1 , \hdots , \alpha_k\right),$ topic distributions $\mu_1 , \hdots , \mu_k.$ Let $\alpha_0 = \sum_{i=1}^k \alpha_i.$ Define

\begin{eqnarray*}
M_1 &=& \mathbb{E}[x_1], \ \ \ M_2 = \mathbb{E}[x_1 \otimes x_2] - \frac{\alpha_0}{1+\alpha_0} M_1 \otimes M_1 \\
M_3 &=& \mathbb{E}[x_1 \otimes x_2 \otimes x_3] - \frac{\alpha_0}{\alpha_0+2} \left( \mathbb{E}[x_1 \otimes x_2 \otimes M_1] + \mathbb{E}[x_1 \otimes M_1 \otimes x_3] + \mathbb{E}[M_1 \otimes x_2 \otimes x_3] \right) \\
&& + \frac{2\alpha_0^2}{(\alpha_0+1)(\alpha_0+2)} M_1 \otimes M_1 \otimes M_1
\end{eqnarray*}

Then,
\begin{eqnarray*}
M_1 &=& \sum_{i=1}^k \frac{\alpha_i}{\alpha_0} \mu_i, \ \ M_2 = \sum_{i=1}^k \frac{\alpha_i}{\alpha_0 (\alpha_0+1)}\mu_i \otimes \mu_i \\
M_3 &=& \sum_{i=1}^k \frac{2\alpha_i}{\alpha_0 (\alpha_0+1)(\alpha_0+2)} \mu_i \otimes \mu_i \otimes \mu_i
\end{eqnarray*}
\end{theorem}

\begin{theorem}
For an LDA model for any $v \in \mathbb{R}^d$ suppose $m,A,B$ be defined as
\begin{eqnarray*}
m &=& \alpha_0 \mathbb{E}[x_1] \\
A &=&  \alpha_0 (\alpha_0+1) \mathbb{E}[x_1 x_2^T] - m m^T \\
B &=& \frac{\alpha_0 (\alpha_0+1) (\alpha_0+2)}{2} \mathbb{E}[\langle x_3,v \rangle x_1 x_2^T] - \frac{\alpha_0 (\alpha_0+1)}{2} \left( \langle m,v\rangle \mathbb{E}[x_1 x_2^T] + \mathbb{E}[\langle x_3,v \rangle x_1 m^T] \right. \\
&& \left.+ \mathbb{E}[\langle x_3 , v \rangle m x_2^T] \right) + \langle m,v\rangle m m^T. 
\end{eqnarray*}
Then we can express $m,A,B$ as follows.
\begin{equation*}
m = \sum_{i=1}^k \alpha_i \mu_i, \ \ \ A = \sum_{i=1}^k \alpha_i \mu_i \mu_i^T, \ \ \ B = \sum_{i=1}^k \alpha_i \langle \mu_i,v \rangle \mu_i \mu_i^T
\end{equation*}
\end{theorem}

\begin{proof}
The expressions for $m$ and $A$ follows easily from Theorem \ref{thm:anandkumar_lda_moments} since $m = \alpha_0 M_1$ and $A=\alpha_0 (\alpha_0+1) M_2.$ To show the expression for $B$ consider the tensor contraction $M_3(I,I,v),$ $M_3$ defined as in Theorem \ref{thm:anandkumar_lda_moments}. Then we have
\begin{eqnarray*}
M_3(I,I,v) &=& \mathbb{E}[\langle x_3,v \rangle x_1 \otimes x_2] - \frac{\alpha_0}{\alpha_0+2} \left( \mathbb{E}[\langle M_1 , v \rangle x_1 \otimes x_2] + \mathbb{E}[\langle x_3,v\rangle x_1 \otimes M_1] \right. \\
&& \left. + \mathbb{E}[\langle x_3,v \rangle M_1 \otimes x_2 \otimes x_3] \right) + \frac{2\alpha_0^2}{(\alpha_0+1)(\alpha_0+2)} \langle M_1,v \rangle \otimes M_1 \otimes M_1 \\
&=& \frac{2}{\alpha_0 (\alpha_0+1) (\alpha_0+2)} B
\end{eqnarray*}

where we used $x_1 \otimes x_2$ is same as $x_1 x_2^T$ and so on. We also get from Theorem \ref{thm:anandkumar_lda_moments} 

$$
M_3(I,I,v) = \sum_{i=1}^k \frac{2\alpha_i}{\alpha_0 (\alpha_0+1)(\alpha_0+2)} \langle \mu_i , v \rangle \mu_i \otimes \mu_i
$$ 

Therefore we have 
\begin{equation*}
B= \frac{\alpha_0 (\alpha_0+1)(\alpha_0+2)}{2} M_3(I,I,v) = \sum_{i=1}^k \alpha_i \langle \mu_i,v \rangle \mu_i \mu_i^T.
\end{equation*}
\end{proof}

\subsection{Mixed Regression Moments} \label{app:mixed_regression}

Recall in mixed regression we have $y = \langle x,\mu_i\rangle + \xi$ where $x \sim \mathcal{N}(0,I)$ and $\xi \sim \mathcal{N}(0,\sigma^2).$ In the following Lemmas we compute the various moments $M_{1,1}, M_{2,2}, M_{3,1}, M_{3,3}$ and show how they are used to compute $m,A,B.$ 

\begin{lemma} \label{lem:mixed_reg_moments}
In mixed linear regression define $M_{1,1} = \mathbb{E}[yx],$ $M_{2,2} = \mathbb{E}[y^2 x x^T],$ $M_{3,1}=\mathbb{E}[y^3 x]$ and $M_{3,3}=\mathbb{E}[y^3 \langle x,v\rangle x x^T].$ Then,

\begin{align*}
  M_{1,1} &= \sum_{i=1}^k \alpha_i \mu_i  \\
  M_{2,2} &= 2 \sum_{i=1}^k \alpha_i \mu_i \mu_i^T + (\sigma^2 + \sum_{i=1}^k \alpha_i \|\mu_i\|^2 ) I \\
  M_{3,1} &= 3 \sum_{i=1}^k \alpha_i (\sigma^2 + \|\mu_i\|^2) \mu_i \\
  M_{3,3} &= 6 \sum_{i=1}^k \alpha_i \langle \mu_i, v\rangle \mu_i \mu_i^T + \left( M_{3,1} v^T + v M_{3,1}^T + \langle M_{3,1}, v\rangle I \right) 
\end{align*}

\end{lemma}

\begin{proof}

We compute the moments as shown below.

\begin{equation*}
M_{1,1} = \mathbb{E}[yx] = \sum_{i=1}^k \alpha_i \mathbb{E}[x^T \mu_i x + \xi x] = \sum_{i=1}^k \alpha_i \mu_i
\end{equation*}

\begin{eqnarray*}
M_{2,2} &=& \mathbb{E}[y^2 x x^T] = \sum_{i=1}^k \alpha_i \mathbb{E}[\langle \mu_i,x\rangle^2 x x^T] + \mathbb{E}[\xi^2] \mathbb{E}[xx^T] \\
&=& \sum_{i=1}^k \alpha_i \mathbb{E}[\langle \mu_i,x\rangle^2 x x^T] + \sigma^2 I \\
&=&  \sum_{i=1}^k \alpha_i (2 \mu_i \mu_i^T + \|\mu_i\|^2 I) + \sigma^2 I \\
&=& 2 \sum_{i=1}^k \alpha_i \mu_i \mu_i^T + \sum_{i=1}^k \alpha_i (\sigma^2 +\|\mu_i\|^2 ) I
\end{eqnarray*}

Using the fact that all odd moments of normal random variable are zero.

\begin{eqnarray*}
M_{3,1} &=& \mathbb{E}[y^3 x] = \sum_{i=1}^k \alpha_i \mathbb{E}[(\langle x,\mu_i\rangle+\xi)^3 x] \\
&=& \sum_{i=1}^k \alpha_i \mathbb{E}[\langle x,\mu_i\rangle^3 x] + 3 \sum_{i=1}^k \alpha_i \mathbb{E}[\xi^2] \mathbb{E}[\langle x,\mu_i\rangle x] \\
&=& 3 \sum_{i=1}^k \alpha_i \|\mu_i\|^2 \mu_i   + 3 \sum_{i=1}^k \alpha_i \sigma^2 \mu_i = 3 \sum_{i=1}^k \alpha_i (\sigma^2 + \|\mu_i\|^2) \mu_i
\end{eqnarray*}

We use the fact that for even $p$ the moment $\mathbb{E}[z^p] = (p-1)!!$ for a standard normal random variable $z$ and $!!$ denote the double factorial. Next we compute $M_{3,3}.$

\begin{eqnarray}
M_{3,3} &=& \mathbb{E}[y^3 \langle x,v\rangle x x^T] = \sum_{i=1}^k \alpha_i \mathbb{E}[(\langle x,\mu_i\rangle+\xi)^3 \langle x,v \rangle x x^T] \nonumber \\
&=& \sum_{i=1}^k \alpha_i \mathbb{E}[\langle x,\mu_i\rangle^3 \langle x,v\rangle x x^T] + 3 \sum_{i=1}^k \alpha_i \mathbb{E}[\xi^2] \mathbb{E}[\langle x,v\rangle \langle x,\mu_i\rangle x x^T] \nonumber \\
&=& \sum_{i=1}^k \alpha_i \mathbb{E}[\langle x,\mu_i\rangle^3 \langle x,v\rangle x x^T] + 3 \sigma^2 \sum_{i=1}^k \alpha_i \mathbb{E}[\langle x,v\rangle \langle x,\mu_i\rangle x x^T] \label{eq:mixreg1}
\end{eqnarray}

Now we compute these individual moments.

\begin{equation*}
\mathbb{E}[\langle x,v\rangle \langle x,\mu_i\rangle x x^T] = \mu_i^T v + v \mu_i^T + \langle \mu_i , v\rangle I
\end{equation*}

Using the fact that any odd combination of the variables in $x$ will be zero in expectation. Also,

\begin{equation*}
\mathbb{E}[\langle x,\mu_i\rangle^3 \langle x,v\rangle x x^T] = 6 \langle v,\mu_i \rangle \mu_i \mu_i^T + 3 \|\mu_i\|^2 [\mu_i^T v + v \mu_i^T + \langle \mu_i,v \rangle I]
\end{equation*}

Again by using the moments of standard normal variable. This can be verified by considering the $(a,b)$-th entry of the matrix on the right as a polynomial in $\mu_i(l),$ the $l$-th component of $\mu_i,$ and matching the corresponding coefficients from both sides of the equation. 

Combining with equation \eqref{eq:mixreg1} we get,

\begin{eqnarray*}
M_{3,3} &=&  \sum_{i=1}^k \alpha_i \left[ 6 \langle v,\mu_i \rangle \mu_i \mu_i^T + 3 \|\mu_i\|^2 (\mu_i^T v + v \mu_i^T + \langle \mu_i,v \rangle I) \right] \\
&& + 3\sigma^2 \sum_{i=1}^k \alpha_i [\mu_i^T v + v \mu_i^T + \langle \mu_i , v\rangle I]  \\
&=& 6 \sum_{i=1}^k \alpha_i \langle v,\mu_i \rangle \mu_i \mu_i^T + 3 \sum_{i=1}^k \alpha_i (\sigma^2 + \|\mu_i\|^2) [\mu_i^T v + v \mu_i^T + \langle \mu_i , v\rangle I] \\
&=& 6 \sum_{i=1}^k \alpha_i \langle v,\mu_i \rangle \mu_i \mu_i^T + \left( M_{3,1} v^T + v M_{3,1}^T + \langle M_{3,1}, v\rangle I \right) 
\end{eqnarray*}
\end{proof}

\begin{theorem}
Let $m,A,B$ be defined as

\begin{eqnarray*}
m &=& M_{1,1}, \ \ \ A = \frac{1}{2} (M_{2,2} - \tau^2 I), \\
B &=& \frac{1}{6} (M_{3,3} - (M_{3,1}v^T + v M_{3,1}^T + \langle M_{3,1},v\rangle I))
\end{eqnarray*}

where $\tau^2$ is the smallest singular value of $M_{2,2}.$ Then,

\begin{equation*}
m = \sum_{i=1}^k \alpha_i \mu_i, \ \ \ A = \sum_{i=1}^k \alpha_i \mu_i \mu_i^T, \ \ \ B = \sum_{i=1}^k \alpha_i \langle \mu_i,v \rangle \mu_i \mu_i^T
\end{equation*}

\end{theorem}

\begin{proof}
The proof follows directly from Lemma \ref{lem:mixed_reg_moments}. Note that since $\mu_i$-s are linearly independent the smallest singular vector $\tau^2$ of $M_{2,2}$ is equal to $\sum_{i=1}^k \alpha_i (\sigma^2 + \|\mu_i\|^2).$ Then $A=\frac{1}{2} \left( M_{2,2} - \tau^2 I\right) = \sum_{i=1}^k \alpha_i \mu_i \mu_i^T.$ Similarly the expression for $B$ holds.
\end{proof}

\subsection{Subspace Clustering Moments} \label{app:subspace_clustering}

In this section we derive the necessary moments required for subspace clustering. Recall that in the subspace clustering model we have $k$ dimension--$m$ subspaces $U_1, \dots, U_k \in \R^{d \times m}$ (matrices
$U_1, \dots, U_k$ have orthonormal columns). The data is generated as follows. We sample $y \sim \normal(0, I_d)$
and set $x = U_i U_i^T y+\xi,$ where $\xi \sim \normal(0,\sigma^2 I_d)$ is additive noise.

\begin{theorem} \label{thm:subspace_moments}
Consider the subspace clustering model. Let $M_2,A,B$ be defined as,
\begin{eqnarray*}
M_2 &:=& \E[x x^T], \ \ A:=M_2- \sigma^2 I_d \\
B &:=& \E[\langle x,v\rangle^2 x x^T] -\sigma^2 (v^T A v) I_d - \sigma^2 \|v\|^2 A - \sigma^4(\|v\|^2 I_d + v v^T) - 2 \sigma^2 (A vv^T+vv^T A)
\end{eqnarray*}
where $\sigma^2=\sigma_{mk+1}(M_2).$ Then,
\begin{eqnarray*}
A &=& \sum_{i=1}^k \alpha_i U_i U_i^T \\
B &=& \sum_{i=1}^k \alpha_i \|U_i^T v\|^2 U_i U_i^T
  + 2 \sum_{i=1}^k \alpha_i U_i U_i^T v v^T U_i U_i^T
\end{eqnarray*}

\end{theorem}
\begin{proof}
First we compute $M_2.$
\begin{equation*}
M_2 = \E(xx^T) = \sum_{i=1}^k \alpha_i \E \left[U_i U_i^T y y^T U_i U_i^T\right] + \E[\xi \xi^T] = \sum_{i=1}^k \alpha_i U_i U_i^T + \sigma^2 I_d
\end{equation*}

Using $\mathbb{E}[yy^T]=I$ as $y \sim \mathcal{N}(0,I)$ and $U_i^T U_i = I$ since the columns are orthogonal. Since $\alpha_i>0,$ the $mk+1$-th singular value of $M_2,$ $\sigma_{mk+1}(M_2)=\sigma^2.$ Therefore it follows that,

$$
A = M_2-\sigma^2 I_d = \sum_{i=1}^k \alpha_i U_i U_i^T
$$

Now we compute the moment $\E[\langle x,v\rangle^2 xx^T].$ Given a sample $x=U_i U_i^T y+\xi$ from the $i$-th subspace we have,

\begin{eqnarray*}
\langle x,v \rangle^2 &=& v^TU_iU_i^Tyy^TU_iU_i^Tv + v^T\xi \xi^T v + 2 v^T \xi v^T U_i U_i^T y \\
xx^T &=& U_iU_i^Tyy^TU_iU_i^T + U_iU_i^Ty\xi^T + \xi y^T U_iU_i^T + \xi \xi^T
\end{eqnarray*}

Then we can write,

\begin{eqnarray}
&& \E[\langle x,v\rangle^2 xx^T] \nonumber \\
&=& \sum_{i=1}^k \alpha_i \left( \E[v^TU_iU_i^Tyy^TU_iU_i^TvU_iU_i^Tyy^TU_iU_i^T] +\E[v^TU_iU_i^Tyy^TU_iU_i^Tv]\E[\xi \xi^T] \right. \nonumber \\
&& + \E[v^T\xi\xi^Tv]\E[U_iU_i^Tyy^TU_iU_i^T]+\E[v^T\xi\xi^Tv\xi\xi^T] + 2\E[(v^T\xi v^TU_iU_i^Ty)U_iU_i^Ty\xi^T] \nonumber \\
&& \left. + 2\E[(v^T\xi v^TU_iU_i^Ty)\xi y^TU_iU_i^T] \right) \nonumber \\
&=& T_1+T_2+T_3+T_4+T_5+T_6
\end{eqnarray}

where $T_1,\hdots , T_6$ are as follows. We define $v_i := U_i U_i^T v,$ we use the Gaussian moment results $\E[\langle v,z\rangle z]=\sigma^2 v,$ and $\E[\langle v,z\rangle^2 zz^T]=\sigma^4(\|v\|^2 I_d + vv^T)$ whenever $z \sim \normal(0,\sigma^2 I_d).$ 

\begin{eqnarray*}
T_1 &=& \sum_{i=1}^k \alpha_i \E\left[v^T U_i U_i^T y y^T U_i U_i^T v U_i U_i^T y y^T U_i U_i^T\right] \\
&=& \sum_{i=1}^k \alpha_i \E [\langle y,v_i\rangle^2 U_i U_i^T y y^T U_i U_i^T ] = \sum_{i=1}^k \alpha_i U_i U_i^T \E[\langle y,v_i\rangle^2 yy^T] U_i U_i^T \\
&=& \sum_{i=1}^k \alpha_i U_i U_i^T (\|v_i\|^2 I_d + 2 v_i v_i^T) U_i U_i^T \\
&=& \sum_{i=1}^n \alpha_i \|v_i\|^2 U_i U_i^T + 2 \sum_{i=1}^k \alpha_i U_i U_i^T v v^T U_i U_i^T \\
&=& \sum_{i=1}^k \alpha_i \|U_i^T v\|^2 U_i U_i^T
  + 2 \sum_{i=1}^k \alpha_i U_i U_i^T v v^T U_i U_i^T
\end{eqnarray*}

since $\|v_i\|=\|U_i U_i^T v\|=\|U_i^T v\|.$

\begin{eqnarray*}
T_2 &=& \sum_{i=1}^k \alpha_i\E[v^TU_iU_i^Tyy^TU_iU_i^Tv]\E[\xi \xi^T] =\sum_{i=1}^k \alpha_i v^T U_iU_i^Tv \times \sigma^2 I_d = \sigma^2 (v^TAv) I_d \\
T_3 &=& \sum_{i=1}^k \alpha_i \E[v^T\xi\xi^Tv]\E[U_iU_i^Tyy^TU_iU_i^T] = \sigma^2 \|v\|^2 \sum_{i=1}^k \alpha_i U_iU_i^T = \sigma^2 \|v\|^2 A \\
T_4 &=& \sum_{i=1}^k \alpha_i \E[v^T\xi\xi^Tv\xi\xi^T] = \E[\langle v,\xi \rangle^2 \xi \xi^T] = \sigma^4(\|v\|^2 I_d+2vv^T) \\
T_5 &=& \sum_{i=1}^k \alpha_i 2\E[(v^T\xi v^TU_iU_i^Ty)U_iU_i^Ty\xi^T] = 2\sum_{i=1}^k \alpha_i \E[(v^T U_i U_i^T y)U_iU_i^Ty] \E[\langle v,\xi \rangle \xi^T] \\
&=& 2\sum_{i=1}^k \alpha_i \E[(v^T U_i U_i^T y)U_iU_i^Ty] \times \sigma^2 v^T = 2\sigma^2 \sum_{i=1}^k \alpha_i \E[(v^T U_i U_i^T y)U_iU_i^Tyv^T] \\
&=& 2\sigma^2 \sum_{i=1}^k \alpha_i \E[U_iU_i^T \langle v,y \rangle yv^T] = 2\sigma^2 \sum_{i=1}^k \alpha_i U_i U_i^T v v^T = 2\sigma^2 A v v^T \\
T_6 &=& 2\sum_{i=1}^k \alpha_i \E[(v^T\xi v^TU_iU_i^Ty)\xi y^TU_iU_i^T] = 2\sum_{i=1}^k \alpha_i \E[\langle v,\xi \rangle \xi] \E[\langle v_i , y \rangle y^T U_i U_i^T ] \\
&=& 2 \sigma^2 \sum_{i=1}^k \alpha_i v v_i^T U_i U_i^T = \sigma^2 \sum_{i=1}^k \alpha_i v v^T U_i U_i^T = \sigma^2 v v^T \sum_{i=1}^k \alpha_i U_i U_i^T = 2 \sigma^2 v v^T A
\end{eqnarray*}

Therefore,

\begin{eqnarray*}
B &=& \E[\langle x,v\rangle^2 x x^T] -\sigma^2 (v^T A v) I_d - \sigma^2 \|v\|^2 A - \sigma^4(\|v\|^2 I_d + v v^T) - 2 \sigma^2 (A vv^T+vv^T A) \\
&=& \E[\langle x,v\rangle^2 x x^T] - T_2 -T_3 - T_4 -T_5 - T_6 = T_1 \\
&=& \sum_{i=1}^k \alpha_i \|U_i^T v\|^2 U_i U_i^T
  + 2 \sum_{i=1}^k \alpha_i U_i U_i^T v v^T U_i U_i^T
\end{eqnarray*}
\end{proof}

\section{Finite-sample Analysis of the Whitening Method}
\label{sec:finite-whitening}

Suppose that
\begin{align*}
A &= \sum_i \alpha_i \mu_i \mu_i^T \\
B &= \sum_i \beta_i \mu_i \mu_i^T \\
\|A - \hat A\| &\le \epsilon \\
\|B - \hat B\| &\le \epsilon,
\end{align*}
where $\sigma_k$ is the $k$th singular value of $A$.
Let $V$ be the $n \times k$ matrix whose columns are the first
$k$ singular vectors of $A$, and let $\hat V$ be the same for $\hat A$.
Let $D$ be the diagonal matrix of singular values of $A$, and let $\hat D$
be the diagonal matrix of the first $k$ singular values of $\hat A$.
Then $A = V D V^T$ and $V^T V = \hat V^T \hat V = I_k$.
This entire section is under the assumptions of Theorem~\ref{thm:meta-whitening}; in particular,
recall that $\epsilon \le \sigma_k(A)/4$.

It will be technically convenient for us to assume that $\|B\| \le \|A\| = \sigma_1(A)$. This assumption holds without loss of generality: if not, simply rescale the side information, setting $v^{\text{new}} = v \frac{\|A\|}{\|B\|}$. This has the effect of rescaling $B$, so that $\|B^{\text{new}}\| = \|A\|$; define also $\hat B^{\text{new}} = \hat B \frac{\|A\|}{\|B\|}$. Note that
\[
  \|B^{\text{new}} - \hat B^{\text{new}}\| = \|B - \hat B\| \frac{\|A\|}{\|B\|} \le \epsilon
\]
under the assumption $\|B - \hat B\| \le \epsilon$.
Now, the algorithm is homogeneous in $\hat B$: it will produce the same output given either $\hat B$ or $\hat B^{\text{new}}$; hence, it suffices to prove Theorem~\ref{thm:meta-whitening} with $v$, $B$, and $\hat B$ replaced by their new versions. Since the new versions satisfy $\|B^{\text{new}}\| \le \|A\|$, we may assume this without loss of generality. From now on, we will drop the notation $B^{\text{new}}$, and we will simply prove Theorem~\ref{thm:meta-whitening} under the assumption $\|B\| \le \|A\|$.

Our basic tool is Wedin's theorem:
\begin{theorem}
 For a matrix $A$, let $P^A_{\ge s}$ be the orthogonal projection onto the subspace
 spanned by singular vectors of $A$ with singular value at least $s$. Let $P^A_{\le s}$ be
 the orthogonal projection onto the subspace spanned by singular vectors with singular value at most $s$.
 Then for any matrices $A$ and $B$, and for any $s < t$,
 \[
  \|P^A_{\le s} P^B_{\ge t}\| \le \frac {2\|A - B\|}{t-s}.
 \]

\end{theorem}

In applying Wedin's theorem, the following geometric lemma will be useful. In what follows,
$P_E$ denotes the orthogonal projection onto $E$.
\begin{lemma}
  Let $E$ and $F$ be subspaces of $\R^n$ with $\|P_{E^\perp} P_F\| \le \delta$. Then
  $\|P_F v\|^2 \le \|P_E v\|^2 + 3 \delta \|v\|^2$ for every $v \in \R^n$.
\end{lemma}

\begin{lemma}
 If $\epsilon < \sigma_k/4$ then for any $u \in \R^k$,
 \[
  \sqrt{1-\frac{16 \epsilon^2}{\sigma_k^2}} \|u\| \le \|\hat V^T V u\| \le \|u\|.
 \]
\end{lemma}

By a simple change of variables, if we define
\[O = D^{-1/2} \hat V^T V D^{1/2}\]
then $O$ is also an almost-isometry: for every $u \in \R^k$,
 \begin{equation}\label{eq:O-almost-iso}
  \sqrt{1-\frac{16 \epsilon^2}{\sigma_k^2}} \|u\| \le \|O u\| \le \|u\|.
 \end{equation}

\begin{proof}
 First, note that $\sigma_k(\hat A) \ge \sigma_k(A) - \|A - \hat A\| \ge \sigma_k - \epsilon$.
 If $\epsilon < \sigma_k/4$, we also have $\sigma_{k+1}(\hat A) \le \sigma_{k+1}(A) + \epsilon \leq \sigma_k/4 < \sigma_k - \epsilon$,
 which implies that $\hat V \hat V^T = P^{\hat A}_{\ge \sigma_k - \epsilon}$.
 
 Let $\hat W$ be a $d \times (d-k)$ matrix whose columns form an orthonormal basis for the
 orthogonal complement of the column span of $\hat V$. Note that if $\epsilon < \sigma_k/2$ then
 the $k$th singular value of $\hat A$ is strictly larger than $\sigma_k/2$ and the $(k+1)$th singular
 value is at most $\epsilon$.
 Then $P^{\hat A}_{\le \epsilon} = \hat W \hat W^T$.
 By Wedin's theorem,
 \[
  \|\hat W \hat W^T V V^T\| = \|P^{\hat A}_{\le \epsilon} P^{A}_{\ge \sigma_k}\| \le \frac{2\epsilon}{\sigma_k - \epsilon}
  \le \frac{4\epsilon}{\sigma_k}
 \]
 Now, $\hat W^T$ and $V$ have norm $1$, and so it follows that
 \[
  \|\hat W^T V\| = \|\hat W^T (\hat W \hat W^T V V^T) V\| \le \frac{4\epsilon}{\sigma_k}.
 \]
 For any $u \in \R^k$ with $\|u\| = 1$, we have
 \[
  \|\hat V^T V u\|^2 = 1 - \|\hat W^T V u\|^2 \ge 1-16 \epsilon^2/\sigma_k^2,
 \]
 from which the claimed lower bound follows. On the other hand, $\|\hat V^T V u\| \le u$ because both $\hat V^T$ and $V$
 have norm 1.
\end{proof}

Let $M = D^{-1/2} V^T B V D^{-1/2}$ and
$\hat M = \hat D^{-1/2} \hat V^T \hat B \hat V \hat D^{-1/2}$.
Then $M$ is the infinite-sample version of $A$'s whitening matrix applied
to $B$, and $\hat M$ is the finite-sample analogue.
Recall from~\eqref{eq:O-almost-iso} that $O = D^{-1/2} \hat V^T V D^{1/2}$ is an almost-isometry of $\R^k$.

\begin{lemma}\label{lem:up-to-isometry}
 \[
  \|O M O^T - \hat M\| \le C\frac{\epsilon \sigma_1}{\sigma_k^2}.
 \]
\end{lemma}

\begin{proof}
The first step is to approximate $O M O^T$ by $D^{-1/2} \hat V^T B \hat V D^{-1/2}$. To this end, note that
\[
 O M O^T = D^{-1/2} \hat V^T V V^T B V V^T \hat V D^{-1/2}.
\]
Now, $\hat V$ is an isometry of $\R^k$ into $\R^n$; hence,
\[
  \|\hat V^T V V^T - \hat V^T\| = \|\hat V \hat V^T V V^T - \hat V \hat V^T\| = \|P^{\hat A}_{\ge \sigma_k - \epsilon} P^A_{\ge \sigma_k} - P^{\hat A}_{\ge \sigma_k - \epsilon}\| = \|P^{\hat A}_{\ge \sigma_k - \epsilon} P^A_{\le 0}\|,
\]
where the last equality used the fact that $A$ has rank exactly $k$, and hence $I - P^A_{\ge \sigma_k} = P^A_{\le 0}$.
Now, Wedin's theorem applied to the computation above implies that
\[
  \|\hat V^T V V^T - \hat V^T\| \le \frac{2\epsilon}{\sigma_k - \epsilon} \le \frac{4\epsilon}{\sigma_k}
\]
(recalling that $\epsilon \le \sigma_k/4$).

Now, for general matrices $X, Y, \tilde Y, Z$ we have
\begin{multline*}
  \|X^T Y^T Z Y X - X^T \tilde Y^T Z \tilde Y X\| \le \|X^T (Y - \tilde Y)^T Z Y X\| + \|X^T \tilde Y^T Z (Y - \tilde Y) X\| \\
  \le \|Y - \tilde Y\| \|X\|^2 \|Z\| (\|Y\| + \|\tilde Y\|).
\end{multline*}
We apply this with $X = D^{-1/2}$, $Y = \hat V$, $\tilde Y = \hat V V V^T$, and $Z = B$; since
$\|D^{-1/2}\| = \sigma_k^{-1/2}$, $\|B\| \le \sigma_1$, and $\|\hat V\|, \|V\|, \|V^T\| = 1$,
\[
  \|O M O^T - D^{-1/2} \hat V^T B \hat V D^{-1/2} \| \le \frac{8 \epsilon \sigma_1}{\sigma_k^2}
\]

Next, we will replace $B$ by $\hat B$ in the above inequality.
Since $\|\hat V\| = \|\hat V^T\| = 1$ and $\|D^{-1/2}\| = \sigma_k^{-1/2}$,
\begin{eqnarray*}
\|D^{-1/2} \hat V^T B \hat V D^{-1/2} - D^{-1/2} \hat V^T \hat B \hat V D^{-1/2}\| &=& \|D^{-1/2} \hat V^T (B - \hat B) \hat V D^{-1/2}\| \nonumber \\
&\le& \sigma_k^{-1} \|B - \hat B\| \le \frac{\epsilon}{\sigma_k}.
\end{eqnarray*}
Putting this together with the previous bound yields
\begin{equation}
\label{eq:up-to-isometry-1}
\|O M O^T - D^{-1/2} \hat V^T \hat B \hat V D^{-1/2}\| \le \frac{\epsilon}{\sigma_k} + \frac{8 \epsilon \sigma_1}{\sigma_k^2}
\end{equation}

It remains to relate $D^{-1/2} \hat V^T \hat B \hat V D^{-1/2}$ to $\hat M$ (which is the same, but with $\hat D$ instead of $D$).
Now, Weyl's inequality implies that
\[\|D^{-1/2} - \hat D^{-1/2}\| \le \sigma_k^{-1/2} - (\sigma_k - \epsilon)^{-1/2} \le \epsilon \sigma_k^{-3/2},\]
where the second inequality follows from a first-order Taylor expansion and the fact that $\epsilon \le \sigma_k/2$.
Hence,
\begin{eqnarray*}
\|D^{-1/2} \hat V^T \hat B \hat V D^{-1/2} - \hat M\| &\le& \|D^{-1/2} - \hat D^{-1/2}\| \|\hat V^T \hat B \hat V D^{-1/2}\|
 \\
&& + \|\hat D^{-1/2} \hat V^T \hat B \hat V \| \|D^{-1/2} - \hat D^{-1/2}\| \\
&\le& 4 \epsilon \sigma_1 \sigma_k^{-2}.
\end{eqnarray*}
Combining this with~\eqref{eq:up-to-isometry-1} and the triangle inequality, we have
\[
 \|O M O^T - \hat M \| = \frac{\epsilon}{\sigma_k} + 12 \frac{\epsilon \sigma_1}{\sigma_k^2} \le C \frac{\epsilon \sigma_1}{\sigma_k^2}.
\]
\end{proof}

Since $O$ is almost an isometry, it follows that there is an orthogonal matrix $\tilde O$ that is close to $O$
(for example, if $U D V^T = O$ is an SVD, let $\tilde O = U V^T$). In this way, we may find an orthogonal $\tilde O$
such that
\[
  \|O - \tilde O\| \le 1 - \sqrt{1 - \frac{16 \epsilon^2}{\sigma_k^2}} \le \frac{16 \epsilon^2}{\sigma_k^2}.
\]
Now let $u$ be the top eigenvector of $M$ and let $u_O$ be the top eigenvector of $O M O^T$.
Then $\tilde O u$ is the top eigenvector of $\tilde O M \tilde O^T$.
The triangle inequality implies that
\[
  \|O M O^T - \tilde O M \tilde O^T\| \le 2 \|M\| \|O - \tilde O\| \le \frac{32 \epsilon^2}{\sigma_k^2} \|M\|.
\]
On the other hand, $M$ was assumed to have a spectral gap of $\delta \|M\|$. By Wedin's theorem, it follows that
\[
  \|u - \tilde O^T u_O \| =
  \| \tilde O u - u_O \| \le \frac{64 \epsilon^2}{\delta \sigma_k^2}.
\]
Finally, let $\hat u$ be the top eigenvector of $\hat M$. By Lemma~\ref{lem:up-to-isometry} and Wedin's theorem,
\[
  \|\hat u - u_O\| \le \frac{C \epsilon \sigma_1}{\delta \sigma_k^2}.
\]
Then
\begin{equation}\label{eq:Ou-hatu}
  \|O u - \hat u\| \le \|O - \tilde O\| + \|\tilde O u - \hat h\|
  \le C \max\left\{
    \frac{\epsilon \sigma_1}{\delta \sigma_k^2},
    \frac{\epsilon^2}{\delta \sigma_k^2}
  \right\}
  \le \frac{C \epsilon \sigma_1}{\delta \sigma_k^2},
\end{equation}
where the last inequality follows because $\epsilon \le \sigma_k/2 \le \sigma_1/2$.

Next, we unpack $O$. Weyl's inequality implies that
\[\|D^{-1/2} - \hat D^{-1/2}\| \le \sigma_k^{-1/2} - (\sigma_k - \epsilon)^{-1/2} \le \epsilon \sigma_k^{-3/2},\]
where the second inequality follows from a first-order Taylor expansion and the fact that $\epsilon \le \sigma_k/4$.
Hence,
\[
  \|O - \hat D^{-1/2} \hat V^T V D^{1/2}\| \le \|D^{1/2}\| \|D^{-1/2} - \hat D^{-1/2}\| \le \frac{\epsilon \sqrt{\sigma_1}}{\sigma_k^{3/2}}.
\]
The right hand side is smaller than $\frac{\epsilon \sigma_1}{\sigma_k^2}$, and so we may plug it
into~\eqref{eq:Ou-hatu} to obtain
\[
  \|\hat D^{-1/2} \hat V^T V D^{1/2} u - \hat u\|
  \le \frac{C \epsilon \sigma_1}{\delta \sigma_k^2}.
\]
Finally, (again because $\epsilon \le \sigma_k/2$), $\|\hat D^{-1/2}\| \le (\sigma_k/2)^{-1/2}$, and so
\begin{equation}\label{eq:w-close}
  \|V D^{1/2} u - \hat V \hat D^{1/2} \hat u\|
  \le \frac{C \epsilon \sigma_1}{\delta \sigma_k^{5/2}}.
\end{equation}
Setting $w = V D^{1/2} u$ and $\hat w = \hat V \hat D^{1/2} \hat u$ and comparing this
to the setting of Algorithm~\ref{alg:meta-whitening},~\eqref{eq:w-close} shows that the finite-sample
algorithm gets almost the same $w$ as the infinite-sample version.

It remains to check the last few lines of Algorithm~\ref{alg:meta-whitening}; i.e., to see that
we recover the right scaling of $w$.

\begin{lemma}\label{lem:dist-sv}
 Let $M$ be a symmetric matrix of rank $k-1$ and let $E$ be the span of its columns. Then
 $\|w\| \dist(w, E) \ge \sigma_k(M + w w^T)$.
\end{lemma}

\begin{proof}
  It suffices to consider the case $\|w\| = 1$ (for a general $w$, apply the special case of the lemma to $w / \|w\|$ and $M / \|w\|^2$).
Let $P_E$ denote the orthogonal projection onto $E$,
and note that $\|w - P_E w\| = \dist(w, E)$
Let $F = \spn\{E, w\}$. Since $F$ has dimension $k$ and $y \in F^\perp$ implies $\|(M + w w^T) y\| = 0$,
it suffices to find some $y \in F$ such that $\|(M + w w^T) y\| \le \dist(w,E)\|y\|$.
Choose $y = w - P_E w$. Then $M y = 0$ and so
\[
 \|(M + w w^T) y\| = |w^T y| = \|w - P_E w\|^2 = \dist(w,E) \|y\|. 
\]
\end{proof}

\begin{lemma}\label{lem:subspace-decomp}
 Let $E$ be a subspace and take $w \not \in E$. For $x \in \spn\{E, w\}$, let $a(x) \in \R$ be the unique solution
 to $x = aw + e$, $e \in E$. Then $|a(x) - a(y)| \le \|x-y\|/\dist(w, E)$.
\end{lemma}

\begin{proof}
Given $x, y \in \spn\{E, w\}$, we can write $x-y = (a(x) - a(y)) w + e$, where $e \in E$. It follows that
\begin{eqnarray*}
 \|x - y\| &=& \|(a(x) - a(y)) w + e\| \ge \inf_{e \in E} \|(a(x) - a(y))w + e\| \\
&=& |a(x) - a(y)| \dist(w, E). 
\end{eqnarray*}
\end{proof}

Finally, we apply the preceding two lemmas to show that $\hat \alpha_1$ is accurate in Algorithm~\ref{alg:meta-whitening}.
Together with~\eqref{eq:w-close} (whose right hand side provides the value of $\eta$ that we will use),
this completes the proof of Theorem~\ref{thm:meta-whitening}.

\begin{lemma}
 Let $m = \sum_i \alpha_i \mu_i$. If $\|\hat A - A\| \le \epsilon$,
 $\|\hat m - m\|\le \epsilon$ and $\|\hat w - \sqrt{\alpha_1} \mu_1\| \le \eta$ then
 \[
 |\hat \alpha_1 - \alpha_1|
 \le \frac{C\sqrt{\alpha_1} |\alpha_1 R + \eta|}{\sigma_k} \left(\eta + R \frac{\epsilon}{\sigma_k} + \epsilon\right),
 \]
 where $R = \max_i \|\mu_i\|$, provided that the right hand side above is at most $\alpha_1$.
\end{lemma}

\begin{proof}
By Wedin's theorem,
\[
 \|V V^T - \hat V \hat V^T\| \le \frac{2\|\hat A - A\|}{\sigma_k - \|\hat A - A\|} \le 4 \frac{\epsilon}{\sigma_k}
\]
if $\epsilon \le \sigma_k/2$. Hence, 
\begin{eqnarray*}
\|m - \hat V \hat V^T \hat m\| &=&
\|V V^T m - \hat V \hat V^T \hat m\| \\
&\le& \|(V V^T - \hat V \hat V^T) m\| + \|\hat V \hat V^T(m - \hat m)\| \\
&\le& 4\frac{\epsilon}{\sigma_k} \|m\| + \epsilon.
\end{eqnarray*}
Now, let $y = \sqrt{\alpha_1} \hat w + \hat V \hat V^T \sum_{i=2}^k \alpha_i \mu_i$. Then
\begin{eqnarray*}
\|m - y\| &\le& \sqrt{\alpha_1} \|\hat w - \sqrt{\alpha_1} \mu_1\| + \left\|\sum_{i=2}^k \alpha_i (\mu_i - \hat V \hat V^T \mu_i)
\right\| \\
&\le& \eta + \max_i \|\mu_i\| \|V V^T - \hat V \hat V^T\| \\
&\le& \eta + 4\max_i \|\mu_i\| \frac{\epsilon}{\sigma_k}.
\end{eqnarray*}
Defining $R = \max_i \|\mu_i\|$, we have
\[
 \|y - \hat V \hat V^T \hat m\| \le \eta + 8 R \frac{\epsilon}{\sigma_k} + \epsilon.
\]
Now, let $\hat E$ be the span of $\{\hat V \hat D^{1/2} v: v \in \R^k, v \perp \hat u\}$, and note that $\hat E$ may also be written as the column space of $\hat V \hat D^{1/2} (I_k - \hat u \hat u^T) \hat D^{1/2} \hat V^T = \hat V \hat D \hat V^T - \hat w \hat w^T$. Since $\hat V \hat D^{1/2}$ is injective, $\hat E$ has dimension $k-1$ and does not contain $\hat w = \hat V \hat D^{1/2} \hat u$. Hence, $y = \sqrt{\alpha_1} \hat w + e$ is the unique way to decompose
$y$ in $\spn\{\hat w\} \oplus \hat E$. If we define $a$ by the decomposition $\hat m = a \hat w + e$ then
Lemma~\ref{lem:subspace-decomp} implies
\begin{eqnarray*}
 |a - \sqrt{\alpha_1}| &\le& \|y - \hat m\|/\dist(\hat w, \hat E) \\
&\le&
 \frac{1}{\dist(\hat w, \hat E)} \left(\eta + 8 R \frac{\epsilon}{\sigma_k} + \epsilon\right).
\end{eqnarray*}
On the other hand, Lemma~\ref{lem:dist-sv} applied to $\hat V \hat D \hat V^T - \hat w \hat w^T$ and $\hat w$
implies (because the $k$th singular value of $\hat V \hat D \hat V^T \ge \sigma_k - \epsilon \ge \sigma_k/2$)
that $\|\hat w\| \dist(\hat w, \hat E) \ge \sigma_k/2$. Therefore,
\[
 |a - \sqrt{\alpha_1}|
 \le \frac{2\|\hat w\|}{\sigma_k} \left(\eta + 8 R \frac{\epsilon}{\sigma_k} + \epsilon\right)
 \le \frac{2 (\alpha_1 \|\mu_1\| + \eta)}{\sigma_k} \left(\eta + 8 R \frac{\epsilon}{\sigma_k} + \epsilon\right).
\]
Finally, note that $|\hat \alpha_1 - \alpha_1| = |a^2 - \alpha_1| = |a - \sqrt{\alpha_1}|(a + \sqrt{\alpha_1})$.
We consider two cases: if $a \le C \sqrt{\alpha_1}$ then $|\hat \alpha_1 - \alpha_1| \le (1 + C)\sqrt{\alpha_1} |a - \sqrt{\alpha_1}|$, which completes the proof. In the other case, we have
\[
  |\hat \alpha_1 - \alpha_1| \sim \hat \alpha_1 \le C \sqrt{\hat \alpha_1} |a - \sqrt{\alpha_1}|,
\]
which implies that
\[
  |\hat \alpha_1 - \alpha_1| \le C |a - \sqrt{\alpha_1}|^2
\]
for some other constant $C$. This implies
\[
  |\hat \alpha_1 - \alpha_1|
  \le C
  \left[\frac{(\alpha_1 R + \eta)}{\sigma_k} \left(\eta + R \frac{\epsilon}{\sigma_k} + \epsilon\right)\right]^2
  \le C \sqrt {\alpha_1} 
  \left[\frac{(\alpha_1 R + \eta)}{\sigma_k} \left(\eta + R \frac{\epsilon}{\sigma_k} + \epsilon\right)\right],
\]
where the second inequality comes from the assumption that the right hand side in the lemma is bounded by $\alpha_1$.
\end{proof}

As we pointed out in Section \ref{sec:basic}, spectral algorithms similar to Algorithm \ref{alg:meta-whitening} has been proposed before for GMM [\citealt{HsuKakade:13}] and LDA [\citealt{AnaFosHsuKak:12LDA}] models, the main difference being how the second matrix (equivalent to $B$) is constructed. Since the underlying whitening procedure is the same in all these algorithms, the proof approach presented above is similar to those in \cite{HsuKakade:13,AnaFosHsuKak:12LDA}. The proofs diverge when computing the perturbation of the second matrix, matrix $B$ in our algorithm, which introduces different dependence on various parameter models in the overall error bound. For example the error bound in Theorem $4.1$ of \cite{AnaFosHsuKak:12LDA} has a slightly worse dependence on $k$ and $\sigma_k$ than Theorem \ref{thm:meta-whitening}.

\section{Finite-sample Analysis of the Cancellation Method}
\label{sec:finite-cancellation}

In this section we analyze the performance of Algorithm \ref{alg:meta-cancel} when we have finite sample estimates of the matrices $A,B$ and vector $m$. For ease of exposition we replaced the quantities $V_{1:(k-1)},v_i, a_i, c_i$ in Algorithm \ref{alg:meta-cancel} with the notation representing estimate $\widehat{V}_{1:(k-1)}, \hat{v}_i, \hat{a}_i, \hat{c}_i$ respectively, since these are computed from sample estimates $\widehat{A}, \widehat{B}.$ First, we show in Lemma \ref{lem:Zlambda_err_ub} that we can have a good estimate for $\widehat{Z}_{\lambda^*}$ using good estimates for $A,B$ and $\lambda_1.$

\begin{lemma} \label{lem:Zlambda_err_ub}
Let $\widehat{Z}_\lambda = \widehat{A} - \lambda \widehat{B}, Z_\lambda = A-\lambda B.$ Suppose $\max \{\|\widehat{A}-A\|,\|\widehat{B}-B\|\} < \epsilon$ and $\lambda_1 = 1/w_1.$ Then,
\begin{equation*}
\|\widehat{Z}_\lambda - Z_{\lambda_1}\| < \epsilon \left(2+\frac{1}{w_1} \right) + \epsilon_1 \sigma_1(B)
\end{equation*}
when $|\lambda_1 - \lambda| < \epsilon_1 < 1.$
\end{lemma}
\begin{proof}
We have,
\begin{eqnarray*}
\|\widehat{Z}_\lambda - Z_{\lambda_1}\| &\leq& \|\widehat{A}-A\| + \|\lambda \widehat{B} - \lambda_1 B\| \\
&<& \|\widehat{A}-A\| + \lambda_1 \|\widehat{B}-B\|+ |\lambda_1-\lambda| \|\widehat B\| \\
&\leq& \epsilon + \lambda_1 \epsilon + \epsilon_1 (\sigma_1(B) + \epsilon) \\
&<& \epsilon (1+1/w_1+\epsilon_1) + \epsilon_1 \sigma_1(B) < \epsilon \left(2+\frac{1}{w_1} \right) + \epsilon_1 \sigma_1(B)
\end{eqnarray*}
since $\epsilon_1<1.$
\end{proof}

The following lemma will show that even with noisy estimates of $A,B,$ the estimated $\lambda^*$ is close to $\lambda_1.$

\begin{lemma} \label{lem:lambda_lambda_1_ub}
Let $\max \{\|\widehat{A}-A\|,\|\widehat{B}-B\|\} < \epsilon < \sigma_k(A)/2,$ and $\lambda_1 = 1/w_1>0.$ Then,
$$
|\lambda^* - \lambda_1| = O(\epsilon)
$$
\end{lemma}
\begin{proof}
Define $Z'_{\lambda} = VV^T A VV^T - \lambda VV^T B VV^T,$ $V$ being the $d \times k$ matrix of top $k$ eigenvectors of $A.$ The corresponding empirical estimate $\widehat{Z}'_{\lambda} = \widehat{V} \widehat{V}^T \widehat A \widehat{V} \widehat{V}^T - \lambda \widehat{V}\widehat{V}^T \widehat B \widehat{V} \widehat{V}^T.$ The main proof idea is the following. We try to find $\lambda_2, \lambda_3 >0$ such that: 
\begin{enumerate}
\item $\forall \lambda > \lambda_2,$ $\widehat{Z}'_{\lambda}$ is not PSD.
\item $\forall \lambda < \lambda_3,$ $\widehat{Z}'_{\lambda}$ is PSD.
\end{enumerate}
The above two conditions imply that the optimum $\lambda^*$ is bounded as $\lambda_3 \leq \lambda^* \leq \lambda_2.$ We then simply bound $\lambda^*-\lambda_1$ as $\lambda_3-\lambda_1 \leq \lambda^*-\lambda_1 \leq \lambda_2-\lambda_1.$ We now elaborate the above two steps. First, we bound the perturbation of empirical matrix $\widehat{Z}'_{\lambda}$ as follows. Using Wedin's theorem we have $\|\widehat{V} \widehat{V}^T - VV^T\| \leq \frac{4 \epsilon}{\sigma_k(A)}.$ Using this and the theorem assumptions we can compute the following bounds.
\begin{eqnarray*}
\|\widehat{V} \widehat{V}^T \widehat A \widehat{V} \widehat{V}^T -VV^T A VV^T\| &\leq& 13 \epsilon \\
\|\widehat{V} \widehat{V}^T \widehat B \widehat{V} \widehat{V}^T -VV^T B VV^T\| &\leq& \left(1 + \frac{12 \sigma_k(B)}{\sigma_k(A)} \right) \epsilon
\end{eqnarray*}
Combining, we have
\begin{equation}
\|\widehat{Z}'_{\lambda} - Z'_{\lambda}\| \leq \|\widehat{V} \widehat{V}^T \widehat A \widehat{V} \widehat{V}^T -VV^T A VV^T\| + \lambda \|\widehat{V} \widehat{V}^T \widehat B \widehat{V} \widehat{V}^T -VV^T B VV^T\|  \leq c_1 (1+\lambda) \epsilon  \label{eq:zhat_dash_ub}
\end{equation}
where $c_1 = \max \{13,1+\frac{12 \sigma_k(B)}{\sigma_k(A)}\}.$

{\bf Step 1:} Since matrices $A$ and $B$ share the same column and row space, $VV^T A VV^T = A,$ $VV^T B VV^T = B,$ and $Z'_\lambda = Z_{\lambda} = \sum_{i=1}^k (1-\lambda w_i) \alpha_i \mu_i \mu_i^T,$ $w_i = \langle \mu_i, v \rangle.$ Recall, $\mathcal{V} = \text{span} \{\mu_2 , \hdots , \mu_k\}$ and $\Pi$ denote the projection onto $\mathcal{V}_{\perp},$ its perpendicular space. Let $x_1 = \Pi\mu_1/\|\Pi\mu_1\|,$ and $x_1 = V \tilde{x}_1,$ $\|x_1\|=\|\tilde{x}_1\|=1.$ Consider the eigenvalues of the $k \times k$ Hermitian matrix $V^T Z_{\lambda} V.$ Using variational theorem we can write:
\begin{equation}
\tilde{\sigma}_k(V^T Z_{\lambda} V) = \min_{x \neq 0, \|x\|=1} x^T V^T Z_{\lambda} V x \leq \tilde{x}_1^T V^T Z_{\lambda} V \tilde{x}_1 = x_1^T Z_{\lambda} x_1 = (1-\lambda w_1) \alpha_1 a_1'  
\end{equation}
where $a_1'=|\langle x_1,\mu_1 \rangle|^2>0.$ Now note that the matrices $Z_{\lambda}'=VV^T Z_{\lambda} VV^T$ and $V^T Z_{\lambda} V$ have the same set of non-zero eigenvalues since $V$ forms an orthonormal basis of the row/column space of $Z_{\lambda}.$ Therefore we can write from above,
\begin{equation}
\tilde{\sigma}_k(Z'_{\lambda}) = \tilde{\sigma}_k(V^T Z_{\lambda} V) \leq (1-\lambda w_1) \alpha_1 a_1' \label{eq:zdash_sigma_k_ub}
\end{equation}  
For $\lambda=\lambda_1=1/w_1,$ $Z'_{\lambda_1}$ is a rank $k-1$ matrix, and for any $\lambda>\lambda_1,$ $Z'_{\lambda}$ has at least one negative eigenvalue. Consider $\lambda_2>\lambda_1$ such that $Z'_{\lambda_2}$ has one negative eigenvalue and $k-1$ positive eigenvalues. Since $\widehat{Z}'_{\lambda_2},Z'_{\lambda_2}$ are symmetric matrices, using Weyl's inequality we get,
\begin{eqnarray}
\tilde{\sigma}_{k}(\widehat{Z}'_{\lambda_2}) &\leq& \tilde{\sigma}_k(Z'_{\lambda_2}) + \|\widehat{Z}'_{\lambda_2}-Z'_{\lambda_2}\| \leq \tilde{\sigma}_k(Z'_{\lambda_2}) + c_1 (1+\lambda_2) \epsilon \nonumber\\
&\leq& (1-\lambda_2 w_1) \alpha_1 a_1' + c_1 (1+\lambda_2) \epsilon \nonumber\\
&\leq& a_1' [(\alpha_1 + \epsilon) - \lambda_2 (w_1 \alpha_1 - \epsilon)] \label{eq:weyl_zl2}
\end{eqnarray}
using equations \eqref{eq:zhat_dash_ub}, \eqref{eq:zdash_sigma_k_ub}, and assuming $a_1'>c_1$ (else we can simply rescale $\epsilon$). Now for any $\lambda>\lambda_2=\frac{\alpha_1+\epsilon}{\alpha_1 w_1 - \epsilon}$ we get
$$
\tilde{\sigma}_{k}(\widehat{Z}'_{\lambda}) \leq a_1' [(\alpha_1 + \epsilon) - \lambda (w_1 \alpha_1 - \epsilon)] \leq a_1'[(\alpha_1 + \epsilon) - \lambda_2 (w_1 \alpha_1 - \epsilon)] = 0
$$
Therefore, when $\lambda>\lambda_2=\frac{\alpha_1+\epsilon}{\alpha_1 w_1 - \epsilon},$ $\widehat{Z}_{\lambda}'$ is not PSD. This implies that $\lambda_2 \geq \lambda^*.$ Then,
\begin{equation}
\lambda^* - \lambda_1 \leq \lambda_2 - \lambda_1 = \frac{\alpha_1+\epsilon}{\alpha_1 w_1 - \epsilon} - \frac{1}{w_1} = \frac{\epsilon(w_1+1)}{(\alpha_1 w_1 - \epsilon)w_1} \label{eq:lambda_star_ub}
\end{equation}

{\bf Step 2:} Consider $\lambda_3<\lambda_1$ such that $Z'_{\lambda_3}$ is PSD. Then we lower bound $\tilde{\sigma}_{k}(Z'_{\lambda_3})$ as follows. Let $\tilde{v}_{k,\lambda_3}$ be the $k-$th eigenvector of $Z'_{\lambda_3}$ having eigenvalue $\tilde{\sigma}_{k}(Z'_{\lambda_3}).$ Then,
\begin{eqnarray}
\tilde{\sigma}_{k}(Z'_{\lambda_3}) &=& \tilde{v}_{k,\lambda_3}^T Z'_{\lambda_3} \tilde{v}_{k,\lambda_3} = \sum_{i=1}^k \alpha_i (1-\lambda_3 w_i) \tilde{v}_{k,\lambda_3}^T \mu_i \mu_i^T \tilde{v}_{k,\lambda_3} \nonumber\\
&\geq& (1-\lambda_3 w_1) \sum_{i=1}^k \alpha_i |\langle \tilde{v}_{k,\lambda_3},\mu_i \rangle|^2 \geq (1-\lambda_3 w_1) a_2' \label{eq:sigmak_zdash_lb}
\end{eqnarray}
since $w_1>w_i,$ $i\neq 1,$ and where $a_2' = \inf_{\lambda \geq 0} \sum_{i=1}^k \alpha_i |\langle \tilde{v}_{k,\lambda},\mu_i \rangle|^2 > 0.$
Now using the lower bound of Weyl's inequality,
\begin{eqnarray}
\tilde{\sigma}_{k}(\widehat{Z}'_{\lambda_3}) &\geq& \tilde{\sigma}_k(Z'_{\lambda_3}) - \|\widehat{Z}'_{\lambda_3}-Z'_{\lambda_3}\| \nonumber\\
&\geq& \tilde{\sigma}_k(Z'_{\lambda_3}) - c_1 (1+\lambda_3) \epsilon \nonumber\\
&\geq& (1-\lambda_3 w_1) a_2' - c_1 (1+\lambda_3) \epsilon \nonumber\\
&\geq& c_1 [(1 - \epsilon) - \lambda_3 (w_1 + \epsilon)] \label{eq:weyl_zl2_lb}
\end{eqnarray}
using equation \eqref{eq:sigmak_zdash_lb}, and assuming $c_1<a_2'$ (else we can simply rescale $\epsilon$). Then, for any $\lambda<\lambda_3=\frac{(1-\epsilon)}{(w_1 + \epsilon)}$ we have $\tilde{\sigma}_{k}(\widehat{Z}'_{\lambda})>0,$ or $\widehat{Z}'_{\lambda}$ is PSD. This implies $\lambda^*>\lambda_3.$ Therefore,
\begin{equation}
\lambda^* - \lambda_1 \geq \lambda_3 - \lambda_1 = \frac{(1-\epsilon)}{(w_1 + \epsilon)} - \frac{1}{w_1} = - \frac{(w_1+1)\epsilon}{(w_1+\epsilon)w_1}\label{eq:lambda_star_lb}
\end{equation}
Combining equations \eqref{eq:lambda_star_ub}, \eqref{eq:lambda_star_lb} we get,
$$
|\lambda^* - \lambda_1| \leq c_3 \epsilon = O(\epsilon)
$$
where $c_3 = \max \left(\frac{(w_1+1)}{(w_1+\epsilon)w_1},\frac{(w_1+1)}{(\alpha_1 w_1 - \epsilon)w_1}\right).$
\end{proof}

In Lemma \ref{lem:lambda_lambda_1_ub} we assume $w_1=\langle \mu_1,v\rangle$ is positive. When $w_1<0$, we have to modify the line search and find the smallest $\lambda<0$ such that $\widehat{Z}'_{\lambda}$ is PSD. However we can still apply similar arguments and prove that as long as the estimates of $A,B,$ are within $\epsilon$ in spectral norm, Algorithm \ref{alg:meta-cancel} can estimate $\lambda^*$ within an $O(\epsilon)$ accuracy of $\lambda_1$. Lemma \ref{lem:Zlambda_err_ub} and \ref{lem:lambda_lambda_1_ub} together implies that $\|\widehat{Z}_{\lambda^*} - Z_{\lambda_1}\|=O(\epsilon)$ as follows, which will be used to prove Theorem \ref{thm:meta-cancel}. We have,
\begin{eqnarray}
\|\widehat{Z}_{\lambda^*} - Z_{\lambda_1}\| &<& \epsilon \left(2+\frac{1}{w_1} \right) + \epsilon_1 \sigma_1(B) \nonumber \nonumber\\
&\leq& \epsilon \left(2+\frac{1}{w_1} \right) + c_3 \epsilon \sigma_1(B) \nonumber \\
&\leq& 3 \eta_3 \epsilon \label{eq:Zlambda_ub}
\end{eqnarray}
where in the last inequality we assume $\epsilon < \alpha_1 w_1/2,$ and $\eta_3 = \max \left\{2,\frac{1}{w_1},c_3 \sigma_1(B)\right\}.$

\begin{lemma} \label{lem:x1v1_ub}
Let $\|\hat{m}-m\| < \epsilon,$ $\|\widehat{Z}_{\lambda^*} - Z_{\lambda_1}\| < \epsilon_2 < \sigma_{k-1}(Z_{\lambda_1})/2$ for $\lambda_1=\alpha_1/\beta_1.$ $V_{1:(k-1)}$ denote the $d \times (k-1)$ matrix of $k-1$ largest singular vectors of $Z_{\lambda_1}$ and $\widehat{V}_{1:(k-1)}$ be the $d \times (k-1)$ matrix of $k-1$ largest singular vectors of $\widehat{Z}_{\lambda^*}.$ Then,
\begin{eqnarray*}
\|\hat{x}_1 - x_1\| &<& 2\epsilon + \frac{4\epsilon_2 R}{\sigma_{k-1}(Z_{\lambda_1})} = \epsilon_3 \\
\|\hat{v}_1 - v_1\| &<& \frac{2 \epsilon_3}{\alpha_1 a_1} = \epsilon_4
\end{eqnarray*}
where $R = \max_{i\in [k]} \|\mu_i\|.$
\end{lemma}
\begin{proof}
Since, $\|\widehat{Z}_{\lambda^*} - Z_{\lambda_1}\| < \epsilon_2 < \sigma_{k-1}(Z_{\lambda_1})/2,$ applying Wedin's theorem we get,
\begin{equation}
\|\widehat{V}_{1:(k-1)}\widehat{V}_{1:(k-1)}^T - V_{1:(k-1)}V_{1:(k-1)}^T\| \leq \frac{2 \|\widehat{Z}_{\lambda^*} - Z_{\lambda_1}\|}{\sigma_{k-1}(Z_{\lambda_1})-\|\widehat{Z}_{\lambda^*} - Z_{\lambda_1}\|} \leq \frac{4 \epsilon_2}{\sigma_{k-1}(Z_{\lambda_1})} \label{eq:V_2K_ub}
\end{equation}
since $\epsilon_2 < \sigma_{k-1}(Z_{\lambda_1})/2.$ Now,

\begin{eqnarray*}
\|\hat{x}_1 - x_1\| &=& \|\hat{m} - \widehat{V}_{1:(k-1)} \widehat{V}_{1:(k-1)}^T \hat{m} - m + V_{1:(k-1)}V_{1:(k-1)}^T m\| \\
&\leq& \|\hat{m}-m\| + \|(\widehat{V}_{1:(k-1)}\widehat{V}_{1:(k-1)} - V_{1:(k-1)}V_{1:(k-1)}^T) m\| + \|\widehat{V}_{1:(k-1)} \widehat{V}_{1:(k-1)}^T (m-\hat{m})\| \\
&<& 2\|m-\hat{m}\| + \frac{4 \epsilon_2 \|m\|}{\sigma_{k-1}(Z_{\lambda_1})} < 2\epsilon + \frac{4\epsilon_2 R}{\sigma_{k-1}(Z_{\lambda_1})} := \epsilon_3
\end{eqnarray*}

where we used equation \ref{eq:V_2K_ub} and $\|m\| \leq R.$ Recall that $x_1 = \alpha_1 \prod_{\mathcal{V}} \mu_1 = \alpha_1 a_1 v_1,$ where $\mathcal{V}=\text{span}\{\mu_2,\hdots , \mu_k\}$ and $a_1 = \langle \mu_1 , v_1 \rangle.$ To show the second bound,

\begin{eqnarray*}
\|\hat{v}_1 - v_1\| &=& \left\| \frac{\hat{x}_1}{\|\hat{x}_1\|} - \frac{x_1}{\|x_1\|} \right\| \\
&\leq& \frac{\|\hat{x}_1-x_1\|}{\|x_1\|} + \|\hat{x}_1\| \left| \frac{1}{\|x_1\|}-\frac{1}{\|\hat{x}_1\|}\right| \\
&<& \frac{\|\hat{x}_1-x_1\|}{\|x_1\|} + \frac{|\|\hat{x}_1\|-\|x_1\||}{\|x_1\|} \leq  2 \frac{\|\hat{x}_1-x_1\|}{\|x_1\|} \\
&<& \frac{2 \epsilon_3}{\alpha_1 a_1} := \epsilon_4
\end{eqnarray*}

\end{proof}

\begin{lemma} \label{lem:VVtAv1_ub}
Let $\|\widehat{A}-A\| < \epsilon, \|\hat{v}_1 - v_1\| < \epsilon_4.$ Define $d \times k$ matrices $V = [v_1 V_{1:(k-1)}]$ and $\widehat{V} = [\hat{v}_1 \widehat{V}_{1:(k-1)}].$ Then,
\begin{equation*}
\|\widehat{V}\widehat{V}^T\widehat{A}\hat{v}_1-VV^T Av_1\| < \sigma_1(A) \left(3\epsilon_4 + \frac{4 \epsilon}{\sigma_{k-1}(Z_{\lambda_1})} \right) + \epsilon (1+\epsilon_4) 
\end{equation*}
\end{lemma}
\begin{proof}
Similar to Lemma \ref{lem:x1v1_ub} we have from Wedin's theorem $\|\widehat{V}_{1:(k-1)}\widehat{V}_{1:(k-1)}^T - V_{1:(k-1)}V_{1:(k-1)}^T\| < \frac{4 \epsilon}{\sigma_{k-1}(Z_{\lambda_1})}.$ Then we can bound,

\begin{eqnarray} 
\|\widehat{V} \widehat{V}^T - VV^T\| &\leq& \|\hat{v}_1 \hat{v}_1^T - v_1 v_1^T\| + \|\widehat{V}_{1:(k-1)}\widehat{V}_{1:(k-1)} - V_{1:(k-1)}V_{1:(k-1)}^T\| \nonumber \\
&<& 2 \|\hat{v}_1 - v_1\| +   \frac{4 \epsilon}{\sigma_{k-1}(Z_{\lambda_1})} \nonumber\\
&<& 2 \epsilon_4 +  \frac{4 \epsilon}{\sigma_{k-1}(Z_{\lambda_1})} \label{eq:VVt_ub}
\end{eqnarray}

Now,

\begin{eqnarray*}
\|\widehat{V}\widehat{V}^T\widehat{A}\hat{v}_1-VV^T Av_1\| &\leq& \|(\widehat{V} \widehat{V}^T - VV^T) A v_1\| + \|\widehat{V}\widehat{V}^T(A-\widehat{A}) v_1\| \\
&& + \|\widehat{V}\widehat{V}^T \widehat{A}(v_1 - \hat{v}_1)\| \\
&\leq& \|\widehat{V} \widehat{V}^T - VV^T\| \|A\| + \|A-\widehat{A}\| + \|\widehat{A}\| \|v_1 - \hat{v}_1\| \\
&<& \sigma_1(A) \left( 2 \epsilon_4 +  \frac{4 \epsilon}{\sigma_{k-1}(Z_{\lambda_1})} \right) + \epsilon + (\sigma_1(A)+\epsilon) \epsilon_4 
\end{eqnarray*}

where we use inequality \eqref{eq:VVt_ub}, $\|A v_1\| \leq \sigma_1(A)$ as $v_1$ is unit norm, $\|\widehat{V} \widehat{V}^T\|<1$ since $\widehat{V}$ is orthonormal, and $\|\widehat{A}\| < \|A\|+\epsilon.$ Combining,

\begin{equation*}
\|\widehat{V}\widehat{V}^T\widehat{A}\hat{v}_1-VV^T Av_1\| < \sigma_1(A) \left(3\epsilon_4 + \frac{4 \epsilon}{\sigma_{k-1}(Z_{\lambda_1})} \right) + \epsilon (1+\epsilon_4) 
\end{equation*} 
\end{proof}

\begin{lemma} \label{lem:a1_err_ub}
Let $\|\widehat{A}-A\| < \epsilon,$ $\|\hat{x}_1 - x_1\| < \epsilon_3 < \frac{\alpha_1 a_1}{2},$ and $\|\hat{v}_1 - v_1\| < \epsilon_4.$ Then,
\begin{equation*}
|\hat{a}_1-a_1| <  \frac{\alpha_1 a_1 \left( 2 \sigma_1(A) \epsilon_4 + \epsilon (1+\epsilon_4) \right) + 2(\sigma_1(A)+\epsilon) \epsilon_3}{\alpha_1^2 a_1^2}
\end{equation*}
\end{lemma}
\begin{proof}
We first compute,
\begin{eqnarray}
|\hat{v}_1^T \widehat{A} \hat{v}_1 - v_1^T A v_1| &\leq& |(v_1^T-\hat{v}_1^T)Av_1| + |\hat{v}_1^T(A-\widehat{A})v_1| + |\hat{v}_1^T \widehat{A}(v_1-\hat{v}_1)| \nonumber \\
&\leq& \|v_1^T-\hat{v}_1^T\| \sigma_1(A) + \|A-\widehat{A}\| + \sigma_1(\widehat{A}) \|v_1-\hat{v}_1\| \nonumber \\
&<& \sigma_1(A) \epsilon_4 + \epsilon + (\sigma_1(A)+\epsilon) \epsilon_4 = 2 \sigma_1(A) \epsilon_4 + \epsilon (1+\epsilon_4) \label{eq:v1tAv1_ub}
\end{eqnarray}
using the fact that $v_1, \hat{v}_1$ have unit norms. Now we can bound the error $|\hat{a}_1-a_1|$ as follows.
\begin{eqnarray*}
|\hat{a}_1-a_1| &=& \left| \frac{\hat{v}_1^T \widehat{A} \hat{v}_1}{\|\hat{x}_1\|} - \frac{v_1^T A v_1}{\|x_1\|}\right| \\
&\leq& \frac{1}{\|x_1\|} |\hat{v}_1^T \widehat{A} \hat{v}_1 - v_1^T A v_1| + |\hat{v}_1^T \widehat{A} \hat{v}_1| \frac{|\|x_1\|-\|\hat{x}_1\||}{\|x_1\| \|\hat{x}_1\|}
\end{eqnarray*}
From equation \eqref{eq:v1tAv1_ub} and using $|\|x_1\|-\|\hat{x}_1\|| < \|\hat{x}_1-x_1\| < \epsilon_3, \|x_1\| = \alpha_1 a_1$ we get,

\begin{eqnarray*}
|\hat{a}_1-a_1| &<& \frac{2 \sigma_1(A) \epsilon_4 + \epsilon (1+\epsilon_4)}{\alpha_1 a_1} + \frac{(\sigma_1(A)+\epsilon)\epsilon_3}{\alpha_1 a_1 (\alpha_1 a_1 - \epsilon_3)} \\
&<& \frac{\alpha_1 a_1 \left( 2 \sigma_1(A) \epsilon_4 + \epsilon (1+\epsilon_4) \right) + 2(\sigma_1(A)+\epsilon) \epsilon_3}{\alpha_1^2 a_1^2}
\end{eqnarray*}

since $\epsilon_3 < \frac{\alpha_1 a_1}{2}.$
\end{proof}

Note that from Lemma \ref{lem:x1v1_ub} taking $\frac{2 \epsilon_3}{\alpha_1 a_1} = \epsilon_4$ the above bound becomes $|\hat{a}_1-a_1| < \frac{6 \sigma_1(A) \epsilon_3 + \epsilon \alpha_1 a_1  + 4 \epsilon \epsilon_3}{\alpha_1^2 a_1^2}.$





\subsection{Proof of Theorem \ref{thm:meta-cancel}}

We now proof Theorem \ref{thm:meta-cancel}. Assume $\|\widehat{Z}_{\lambda^*}-Z_{\lambda_1}\|\leq \epsilon_2.$ Under the assumptions we have using Lemma \ref{lem:x1v1_ub} $\|\hat{x}_1 - x_1\| < \epsilon_3 = 2\epsilon + \frac{4\epsilon_2 R}{\sigma_{k-1}(Z_{\lambda_1})},$ $\|\hat{v}_1 - v_1\| < \epsilon_4 = \frac{2\epsilon_3}{\alpha_1 a_1}.$ Also from Lemma \ref{lem:VVtAv1_ub} we have $\|\widehat{V}\widehat{V}^T\widehat{A}\hat{v}_1-VV^T Av_1\| < \sigma_1(A) \left(3\epsilon_4 + \frac{4 \epsilon}{\sigma_{k-1}(Z_{\lambda_1})} \right) + \epsilon (1+\epsilon_4).$ Using these we compute the first bound as follows.
\begin{eqnarray*}
\|\hat{\mu}_1 - \mu_1\| &=& \left\| \frac{\widehat{V}\widehat{V}^T\widehat{A}\hat{v}_1}{\|\hat{x}_1\|} - \frac{VV^TAv_1}{\|x_1\|}\right\| \\
&\leq& \|\widehat{V}\widehat{V}^T\widehat{A}\hat{v}_1\| \left| \frac{1}{\|\hat{x}_1\|} - \frac{1}{\|x_1\|} \right| + \frac{1}{\|x_1\|} \|\widehat{V}\widehat{V}^T\widehat{A}\hat{v}_1 - VV^T Av_1\| \\
&\leq& \|\widehat{A}\| \frac{\|\hat{x}_1-x_1\|}{\|\hat{x}_1\| \|x_1\|} + \frac{1}{\|x_1\|} \|\widehat{V}\widehat{V}^T\widehat{A}\hat{v}_1 - VV^T Av_1\| 
\end{eqnarray*}

Now using bounds from Lemma \ref{lem:x1v1_ub}, \ref{lem:VVtAv1_ub} we get,
\begin{eqnarray*}
\|\hat{\mu}_1 - \mu_1\| &<& \frac{(\sigma_1(A)+\epsilon) \epsilon_3}{\alpha_1 a_1 (\alpha_1 a_1 - \epsilon_3)}  + \frac{\sigma_1(A) \left(3\epsilon_4 + \frac{4 \epsilon}{\sigma_{k-1}(Z_{\lambda_1})} \right) + \epsilon (1+\epsilon_4)}{\alpha_1 a_1} \\
&<& \frac{2}{\alpha_1^2 a_1^2}\left[\left(\sigma_1(A)+\epsilon\right) \epsilon_3 + \alpha_1 a_1 \left( (3\sigma_1(A)+\epsilon) \epsilon_4 \right. \right. \\
&& \left. \left. + \epsilon \left(1+4\sigma_1(A)/\sigma_{k-1}(Z_{\lambda_1})\right)\right)\right] \\
&<& \frac{2}{\alpha_1^2 a_1^2} \left[ \left(\sigma_1(A)+\epsilon\right) \epsilon_3 + 2\left(3\sigma_1(A) +\epsilon\right)\epsilon_3 \right. \\ 
&& \left. + \alpha_1 a_1 \epsilon \left(1+4\sigma_1(A)/\sigma_{k-1}(Z_{\lambda_1}) \right) \right]\\
&\leq& 2\frac{10 \sigma_1(A) \epsilon_3 + 5 \alpha_1 a_1 \epsilon \frac{\sigma_1(A)}{\sigma_{k-1}(Z_{\lambda_1})}}{\alpha_1^2 a_1^2}
\end{eqnarray*}

assuming $\epsilon_3 \leq \frac{\alpha_1 a_1}{2},$ $\sigma_1(A) \geq \epsilon,$ and $\sigma_1(A) > \sigma_{k-1}(Z_{\lambda_1}).$ Now expanding $\epsilon_3$ and rearranging terms we have,
\begin{eqnarray}
\|\hat{\mu}_1 - \mu_1\|&<& \frac{1}{\alpha_1^2 a_1^2} \left( \left(40+10\frac{\alpha_1 a_1}{\sigma_{k-1}(Z_{\lambda_1})}\right) \sigma_1(A)\epsilon + 80\frac{\sigma_1(A) R \epsilon_2}{\sigma_{k-1}(Z_{\lambda_1})}\right) \nonumber \\
&<& \frac{80}{\alpha_1^2 a_1^2} \left( \sigma_1(A) \epsilon \left( 1 + \frac{\alpha_1 a_1}{\sigma_{k-1}(Z_{\lambda_1})}\right) + \frac{\sigma_1(A) \epsilon_2 R}{\sigma_{k-1}(Z_{\lambda_1})}\right) \label{eq:muhat_ub}
\end{eqnarray}

To prove the second bound from Lemma \ref{lem:a1_err_ub} and assuming $\epsilon < \sigma_1(A)$ we have $|\hat{a}_1-a_1| \leq \frac{10 \sigma_1(A) \epsilon_3 + \alpha_1 a_1 \epsilon}{\alpha_1^2 a_1^2}.$ Then,

\begin{eqnarray*}
\hat{a}_1(\alpha_1 - \hat{\alpha}_1) &=& \hat{a}_1 \alpha_1 - \hat{a}_1 \hat{\alpha}_1 \\
&=& a_1 \alpha_1 - \hat{a}_1 \hat{\alpha}_1 + \hat{a}_1 \alpha_1 - a_1 \alpha_1 \\
\hat{a}_1 |\alpha_1 - \hat{\alpha}_1| &\leq& |a_1 \alpha_1 - \hat{a}_1 \hat{\alpha}_1| + \alpha_1 |\hat{a}_1-a_1| \\
|\alpha_1 - \hat{\alpha}_1| &\leq& \frac{1}{\hat{a}_1} \left( \|x_1 - \hat{x}_1\| + \alpha_1 |\hat{a}_1-a_1|\right) \\
&<& \frac{\epsilon_3 + \alpha_1 |\hat{a}_1-a_1|}{a_1 - |\hat{a}_1-a_1|} \\
&\leq& 2 \frac{\epsilon_3 + \frac{(10 \sigma_1(A) \epsilon_3 + \alpha_1 a_1 \epsilon)}{\alpha_1 a_1^2}}{a_1}
\end{eqnarray*}

using $|\hat{a}_1-a_1| < \frac{a_1}{2}.$ We have,
\begin{eqnarray}
|\alpha_1 - \hat{\alpha}_1| &\leq& 2\frac{\alpha_1 a_1^2 \epsilon_3 + 10 \sigma_1(A) \epsilon_3 + \alpha_1 a_1 \epsilon}{\alpha_1 a_1^3} \nonumber \\
&<& \frac{2}{\alpha_1 a_1^3} \left( \left( \alpha_1 a_1^2 + 10 \sigma_1(A)\right)\left( 2 \epsilon + 4R \epsilon_2/\sigma_{k-1}(Z_{\lambda_1})\right) +\alpha_1 a_1 \epsilon \right) \nonumber  \\
&\leq& \frac{4 \sigma_1(A)}{\alpha_1 a_1^3} \left( \eta_1 \epsilon + \frac{\eta_2 R \epsilon_2}{\sigma_{k-1}(Z_{\lambda_1})}\right) \label{eq:ahat_ub}
\end{eqnarray}

where $\eta_1 := \max\{\alpha_1 a_1 (2 a_1 + 1),20\},$ and $\eta_2 := \max\{\alpha_1 a_1^2,  10\}.$

Finally using equation \eqref{eq:Zlambda_ub} we can bound $\|\widehat{Z}_{\lambda^*}-Z_{\lambda_1}\|\leq \epsilon_2 \leq 3\eta_3 \epsilon,$ where $\eta_3 = \max \left\{1,\frac{1}{w_1},c_3 \sigma_1(B)\right\}.$ Using this in equations \eqref{eq:muhat_ub} and \eqref{eq:ahat_ub} proves the theorem.

\subsection{Related Lemmas} \label{sec:cancel_useful_proof}

In this section we prove a supporting lemma for Lemma \ref{lem:coord_algo_exact_stat}.

\begin{lemma} \label{lem:cancel_exact_supporting}
Let $\{\mu_2,\hdots , \mu_k\}$ be linearly independent. Suppose matrix $Z_{\lambda^*}$ be expressed as,
\begin{equation}
Z_{\lambda^*}=\sum_{i=2}^k \alpha_i(1-\lambda^* w_i) \mu_i \mu_i^T=V_{1:(k-1)} \Sigma_{1:(k-1)} V_{1:(k-1)}^T=\sum_{i=2}^k \sigma_{i-1}(Z_{\lambda^*})v_i v_i^T, \label{eq:z_lambda}
\end{equation} 
where $w_i = \langle \mu_i,v \rangle,$ $V_{1:(k-1)}=[v_2,\hdots , v_k]$ the matrix of $k-1$ singular vectors, and $\Sigma_{1:(k-1)}$ is a diagonal matrix of singular values of $Z_{\lambda^*}.$ Then $\{v_2,\hdots , v_k\}$ forms a basis of $span\{\mu_2,\hdots , \mu_k\}.$
\end{lemma}
\begin{proof}
Define $\mathcal{V}_{Z_{\lambda^*}}$ as the {\bf column space} of matrix $Z_{\lambda^*}.$ First observe that from equation \eqref{eq:z_lambda} each column of $Z_{\lambda^*}$ can be written as a linear combination of $\{\mu_2,\hdots , \mu_k\}.$ Therefore any vector in the column space $\mathcal{V}_{Z_{\lambda^*}}$ can be written as a linear combination of $\{\mu_2,\hdots , \mu_k\}.$ this implies,
\begin{equation}
\mathcal{V}_{Z_{\lambda^*}} \subseteq span\{\mu_2,\hdots , \mu_k\} \label{eq:temp_c1}
\end{equation}
Now any vector $y \in \mathcal{V}_{Z_{\lambda^*}}$ can be written as $y = Z_{\lambda^*} x = \sum_{i=2}^k \sigma_{i-1}(Z_{\lambda^*})\langle v_i,x \rangle v_i$ using equation \eqref{eq:z_lambda}. This implies,
\begin{equation}
\mathcal{V}_{Z_{\lambda^*}} \subseteq span\{v_2,\hdots , v_k\} \label{eq:temp_c2}
\end{equation}
Conversely any vector $s \in span\{v_2,\hdots , v_k\}$ can be written as $s = V_{1:(k-1)}r = Z_{\lambda^*}V_{1:(k-1)}\Sigma_{1:(k-1)}^{-1}r = Z_{\lambda^*} r',$ using equation \eqref{eq:z_lambda}, where $r'=V_{1:(k-1)}\Sigma_{1:(k-1)}^{-1}r.$ This implies,
\begin{equation}
span\{v_2,\hdots , v_k\} \subseteq \mathcal{V}_{Z_{\lambda^*}} \label{eq:temp_c3}
\end{equation}
Therefore combining equations \eqref{eq:temp_c1},\eqref{eq:temp_c2},\eqref{eq:temp_c3} we get,
\begin{equation}
span\{v_2,\hdots , v_k\} = \mathcal{V}_{Z_{\lambda^*}} \subseteq span\{\mu_2,\hdots , \mu_k\} \label{eq:temp_c4}
\end{equation}
Note that both the vector spaces $span\{v_2,\hdots , v_k\}$ and $span\{\mu_2,\hdots , \mu_k\}$ have rank $k-1$ since $\{v_2,\hdots , v_k\}$ are orthonormal, and $\{\mu_2,\hdots , \mu_k\}$ are linearly independent. Then from this rank constraint and equation \eqref{eq:temp_c4} we must have:
$$
span\{v_2,\hdots , v_k\} = span\{\mu_2,\hdots , \mu_k\}
$$
This implies $\{v_2,\hdots , v_k\}$ forms a basis of $span\{\mu_2,\hdots , \mu_k\}.$
\end{proof}

\section{Subspace Clustering Proofs} \label{app:subspace_proof}

In this section we prove Theorem \ref{thm:subspace_perturbation} and the necessary lemmas. 
The main point is the following infinite-sample analysis, which shows that the top $m$ eigenvectors
of the whitened matrix $B$ can be used to recover the subspace $\mathcal{U}_1$.

\begin{theorem}\label{thm:subspace_strong}
  Suppose that there is some $\delta > 0$ such that $\|U_i v\|^2 \le (1/3 - \delta) \|U_1 v\|^2$ for all $i \ne 1$.
Let $Y=[u_1,...,u_m]$ be the matrix of top $m$ eigenvectors of $R=D^{-1/2}V^TBVD^{-1/2}$ and $Z=VD^{1/2}Y.$ Let $\mathcal{Z}$ be the subspace spanned by columns of $Z.$ Then,

\begin{enumerate}
\item $\mathcal{Z}=\mathcal{U}_1$
\item $\sigma_m(R)-\sigma_{m+1}(R)\geq 3\delta \|U_1 v\|^2$

\end{enumerate}

\end{theorem}
\begin{proof}
  Define $w_i = \|U_i U_i^T v\| = \|U_i^Tv\|,$ and $\tilde{U}_i := \sqrt{\alpha_i} D^{-1/2}V^T U_i$; note that $\sum_{i=1}^k \tilde U_i \tilde U_i^T$ is the $(km) \times (km)$ identity matrix, which implies that each $\tilde U_i$ has orthonormal columns. Consider the whitened $B$ matrix. Using Theorem \ref{thm:subspace_moments},  
\begin{eqnarray*}
D^{-1/2} V^T B V D^{-1/2} &=& \sum_{i=1}^k w_i^2 \tilde{U}_i \tilde{U}_i^T + 2\sum_{i=1}^k \tilde{U}_i U_i^T v v^T U_i\tilde{U}_i^T \\
&=& \sum_{i=1}^k w_i^2 \tilde{U}_i \tilde{U}_i^T + 2\sum_{i=1}^k \tilde{v}_i \tilde{v}_i^T = \sum_{i=1}^k (w_i^2 \tilde{U}_i \tilde{U}_i^T + 2\tilde{v}_i \tilde{v}_i^T)
\end{eqnarray*}
where $\tilde{v}_i = \tilde{U}_i U_i^T v.$ Note that $\tilde{v}_i$ are orthogonal to each other and each $\tilde{v}_i$ is in the space $\tilde{\mathcal{U}}_i,$ the span of corresponding $\tilde{U}_i.$ Moreover, $\|\tilde v_i\| = w_i$. Now for each $i$ consider a different orthonormal basis $\tilde{V}_i$ of $\tilde{\mathcal{U}}_i$ such that in this basis the first unit vector is aligned along $\tilde{v}_i.$ Define a rotation $R_i$ such that $\tilde{V}_i = \tilde{U}_i R_i.$ Then $\tilde{V}_i \tilde{V}_i^T = \tilde{U}_i \tilde{U}_i^T.$ Therefore we can write the above equation as
\begin{equation}\label{eq:subspace_strong_2}
R = D^{-1/2} V^T B V D^{-1/2} = \sum_{i=1}^k \tilde{V}_i \tilde{D}_i \tilde{V}_i^T
\end{equation}
where each $\tilde{D}_i$ is a diagonal matrix with one maximum value of $3 w_i^2$ and all other values $w_i^2,$ and also the matrices $\tilde{V}_i$ are orthogonal. Under the assumption that $w_i^2 \le (1/3 - \delta) w_1^2$, it follows that the top $m$ eigenvectors
of $R$ are the columns of $\tilde V_i$, and that the corresponding eigenvalues are $3 w_1^2$ and then $w_1^2$ repeated $m-1$ times.
Therefore we can write $Y=\tilde{U}_iO,$ where $O$ is an $m \times m$ orthogonal matrix. Then,
$$
Z =  VD^{1/2}Y=VD^{1/2}\tilde{U}_iO=\sqrt{\alpha_1}U_1O
$$
This proves the first statement that $\mathcal{Z},$ the span of the columns of $Z$, is the subspace $\mathcal{U}_1,$ the span of columns of $U_1.$ The second statement follows from equation \eqref{eq:subspace_strong_2} since the maximum value of the $m+1$-th eigenvalue is $3 w_i^2$ for some $i \ne 1$. Hence,
$$
\sigma_m(R)-\sigma_{m+1}(R) \geq w_1^2- 3\max_{i \ne 1} w_i^2 \geq 3 \delta w_1^2 = 3 \delta \|U_1 v\|^2.
$$
\end{proof}

\begin{lemma} \label{lem:hatWAhatWHalf_ub}
Let $\|\hat{A}-A\|<\epsilon<\sigma_{mk}(A)/4.$ $A=VDV^T$ and $\hat{A}=\hat{V}\hat{D}\hat{V}^T$ be the eigen decompositions of $A,\hat{A}$. Let $\hat{W}=\hat{V}\hat{D}^{-1/2}$ be the whitening matrix. Then,

$$
\|I_k-(\hat{W}^TA\hat{W})^{-1/2}\| \leq \frac{4 \epsilon}{\sigma_{mk}(A)}
$$

\end{lemma}
\begin{proof}
We prove this along the lines in \cite{HsuKakade:13}. The matrix $\hat{W}$ whitens $\hat{A}$ since,

$$
\hat{W}^T \hat{A} \hat{W} = \hat{D}^{-1/2}\hat{V}^T\hat{A}\hat{V}\hat{D}^{-1/2}=I_k
$$

Also $\epsilon<\sigma_{mk}(A)/2,$ hence using Weyl's inequality $\sigma_{mk}(\hat{A})\geq \sigma_{mk}(A)/2.$ This implies

\begin{eqnarray*}
\|I_k-\hat{W}^TA\hat{W}\| &=& \|\hat{W}^T(\hat{A}-A)\hat{W}\| \leq \|\hat{W}\|^2 \|\hat{A}-A\| \\
&<& \frac{2\epsilon}{\sigma_{mk}(A)}
\end{eqnarray*}

Therefore all eigenvalues of the matrix $\hat{W}^TA\hat{W}$ lie in the interval $\left(1-2\epsilon/\sigma_{mk}(A),1+2\epsilon/\sigma_{mk}(A)\right).$ This implies the eigenvalues of $(\hat{W}^TA\hat{W})^{-1}$ lie in the interval $\left(1/(1+2\epsilon/\sigma_{mk}(A)),1/(1-2\epsilon/\sigma_{mk}(A)\right)).$ Then,

\begin{eqnarray*}
(I_k-(\hat{W}^TA\hat{W})^{-1/2})(I_k+(\hat{W}^TA\hat{W})^{-1/2}) &=& I_k-(\hat{W}^TA\hat{W})^{-1} \\
I_k-(\hat{W}^TA\hat{W})^{-1/2} &=& \left(I_k-(\hat{W}^TA\hat{W})^{-1}\right)(I_k+(\hat{W}^TA\hat{W})^{-1/2})^{-1} \\
\|I_k-(\hat{W}^TA\hat{W})^{-1/2}\| &\leq& \|I_k-(\hat{W}^TA\hat{W})^{-1}\| \\
&\leq& \frac{1}{1-2\epsilon/\sigma_{mk}(A)}-1 \leq \frac{4 \epsilon}{\sigma_{mk}(A)}
\end{eqnarray*}

\end{proof}

\begin{lemma}[Whitening matrix perturbation] \label{lem:W_perturbation}
Assume $\|\hat{A}-A\|<\epsilon<\sigma_{mk}(A)/4.$ Let $\hat{W}=\hat{V}\hat{D}^{-1/2}$ be the whitening matrix. Define $W:=\hat{W}(\hat{W}^TA\hat{W})^{-1/2}$ . Then,

$$
\|\hat{W}-W\| \leq \frac{8\epsilon}{\sigma_{mk}(A)^{3/2}}
$$

\end{lemma}
\begin{proof}
We note that the matrix $W$ whitens the matrix $A,$ since

$$
W^TAW = (\hat{W}^TA\hat{W})^{-1/2}\hat{W}^TA\hat{W} (\hat{W}^TA\hat{W})^{-1/2} = I_k
$$

We can bound the perturbation as follows.

\begin{eqnarray*}
\|\hat{W}-W\| &=& \|\hat{W}(I_k-(\hat{W}^TA\hat{W})^{-1/2})\| \\
&\leq& \|\hat{W}\| \|I_k-(\hat{W}^TA\hat{W})^{-1/2}\| \\
&\leq& \frac{2}{\sqrt{\sigma_{mk}(A)}} \frac{4\epsilon}{\sigma_{mk}(A)}=\frac{8\epsilon}{\sigma_{mk}(A)^{3/2}}
\end{eqnarray*}
where the last inequality follows from Lemma \ref{lem:hatWAhatWHalf_ub}.
\end{proof}

\begin{lemma} \label{lem:R_perturbation}
Let $\max \{\|\hat{A}-A\|,\|\hat{B}-B\|\}<\epsilon,$ and also let $\epsilon<\min\{\sigma_1(B)/2,\frac{\sigma_{mk}(A)}{16}\}.$ $W=\hat{W}(\hat{W}^TA\hat{W})^{-1/2}$ be the whitening matrix. Define $R=W^T B W$ as the whitened $B$ matrix, and $\hat{R}=\hat{W}^T\hat{B}\hat{W}$ is its estimate. Then,

$$
\|\hat{R}-R\| < \frac{51 \sigma_1(B) \epsilon}{\sigma_{mk}(A)^2} :=\epsilon_1
$$

\end{lemma}
\begin{proof}
From Lemma \ref{lem:W_perturbation} we have $\|\hat{W}-W\|\leq \frac{8\epsilon}{\sigma_{mk}(A)^{3/2}}<\|\hat{W}\|/2.$ Also we know $\|\hat{W}\|\leq \sqrt{2/\sigma_{mk}(A)}.$ We obtain the required bound as follows.

\begin{eqnarray*}
\|\hat{R}-R\| &=& \|\hat{W}^T\hat{B}\hat{W}-W^TBW\| \\
&\leq& \|(\hat{W}-W)^T\hat{B}\hat{W}\|+\|W^T(\hat{B}-B)\hat{W}\| + \|W^T B (\hat{W}-W)\| \\
&\leq& \frac{3}{2} \|\hat{W}-W\| \|B\|\|\hat{W}\| + \frac{3}{2} \|\hat{W}\|^2 \|\hat{B}-B\| + \frac{3}{2}\|\hat{W}^T\|\|B\| \|\hat{W}-W\| \\
&=& 3 \|\hat{W}-W\| \|B\|\|\hat{W}\| + \frac{3}{2} \|\hat{W}\|^2 \|\hat{B}-B\| \\
&<& 48\frac{\sigma_1(B)\epsilon}{\sigma_{mk}(A)^2} + \frac{3 \epsilon}{\sigma_{mk}(A)} < \frac{51 \sigma_1(B) \epsilon}{\sigma_{mk}(A)^2}
\end{eqnarray*}

\end{proof}

\begin{lemma} \label{lem:subspace_zzhat_ub}
Suppose $Y=[u_1,\hdots,u_m]$ be the matrix of $m$ largest eigenvectors of $R=W^TBW,$ and $\hat{Y}$ be that of $\hat{R}=\hat{W}^T\hat{B}\hat{W}.$ Let $\hat{Z}=\hat{V}\hat{D}^{1/2}\hat{Y}.$ Then,

$$
\|\hat{Z}\hat{Z}^T-ZZ^T\| \leq C_1 \frac{\sigma_1(A)\sigma_1(B)\epsilon}{(\sigma_{m}(R)-\sigma_{m+1}(R))\sigma_{mk}(A)^2}
$$
where $Z$ satisfies $Y=W^TZ,$ and $C_1$ is a constant.
\end{lemma}
\begin{proof}
First using Wedin's theorem for the matrix $A$ and $\hat{A}$ we get

\begin{equation}
\|\hat{V}\hat{V}^T-VV^T\| < \frac{4 \epsilon}{\sigma_{mk}(A)}. \label{eq:VVT_wedin}
\end{equation}

From Lemma \ref{lem:R_perturbation} we have $\|\hat{R}-R\| < \frac{51 \sigma_1(B) \epsilon}{\sigma_{mk}(A)^2}=\epsilon_1.$ Therefore we can again use Wedin's theorem on the matrices $R,\hat{R}$ to bound the perturbation of the subspace spanned by $Y.$ 
\begin{eqnarray}
\|\hat{Y}\hat{Y}^T-YY^T\| &\leq&  \frac{4\|\hat{R}-R\|}{\sigma_{m}(R)-\sigma_{m+1}(R)} \nonumber \\
&=& \frac{4\epsilon_1}{\sigma_{m}(R)-\sigma_{m+1}(R)}. \label{eq:ZZhat_0}
\end{eqnarray}

We now bound the following term.

\begin{eqnarray}
\|\hat{V}\hat{D}^{1/2}W^T-\hat{V}\hat{V}^T\| &=& \|\hat{V}\hat{D}^{1/2}(\hat{W}^TA\hat{W})^{-1/2}\hat{W}^T-\hat{V}\hat{V}^T\| \nonumber \\
&=& \|\hat{V}\hat{D}^{1/2}(\hat{W}^TA\hat{W})^{-1/2}\hat{D}^{-1/2}\hat{V}^T-\hat{V}\hat{V}^T\| \nonumber \\
&\leq& \|\hat{D}^{1/2}(\hat{W}^TA\hat{W})^{-1/2}\hat{D}^{-1/2}-I_k\| \nonumber \\
&\leq& \|\hat{D}^{1/2}\| \|(\hat{W}^TA\hat{W})^{-1/2}-I_k\| \|\hat{D}^{-1/2}\| \nonumber \\
&\leq& \sqrt{\frac{\sigma_1(\hat{A})}{\sigma_{mk}(\hat{A})}} \frac{4\epsilon}{\sigma_{mk}(A)} \leq \frac{8\sigma_1(A)^{1/2}\epsilon}{\sigma_{mk}(A)^{3/2}} \label{eq:ZZhat_1}
\end{eqnarray}
where the second to last inequality follows from Lemma \ref{lem:hatWAhatWHalf_ub}. Next we show that $\hat{Z}\hat{Z}^T$ is close to the projection of $ZZ^T$ onto the subspace $\hat{V}\hat{V}^T.$

\begin{eqnarray}
&& \|\hat{Z}\hat{Z}^T-\hat{V}\hat{V}^TZZ^T\hat{V}\hat{V}^T\| \nonumber \\
&=& \|\hat{V}\hat{D}^{1/2}\hat{Y}\hat{Y}^T\hat{D}^{1/2}\hat{V}^T-\hat{V}\hat{V}^TZZ^T\hat{V}\hat{V}^T\| \nonumber\\
&\leq& \|\hat{V}\hat{D}^{1/2}(\hat{Y}\hat{Y}^T-YY^T)\hat{D}^{1/2}\hat{V}^T\|+\|\hat{V}\hat{D}^{1/2}YY^T\hat{D}^{1/2}\hat{V}^T-\hat{V}\hat{V}^TZZ^T\hat{V}\hat{V}^T\| \nonumber\\
&\leq& \sigma_1(\hat{A})\|\hat{Y}\hat{Y}^T-YY^T\|+\|\hat{V}\hat{D}^{1/2}W^TZZ^TW\hat{D}^{1/2}\hat{V}^T-\hat{V}\hat{V}^TZZ^T\hat{V}\hat{V}^T\| \label{eq:subspace_zzhat_ub_1}
\end{eqnarray}

We bound the second term as follows. Observe that the matrix $D^{-1/2}V^T$ also whitens the matrix $A.$ Therefore $Z$ can be expressed as $Z=VD^{1/2}U'$ where $U'$ is a matrix with orthonormal columns. This implies $\|ZZ^T\|=\|VD^{1/2}U'U'^TD^{1/2}V^T\|\leq \sigma_1(A).$

\begin{eqnarray}
&& \|\hat{V}\hat{D}^{1/2}W^TZZ^TW\hat{D}^{1/2}\hat{V}^T-\hat{V}\hat{V}^TZZ^T\hat{V}\hat{V}^T\| \nonumber \\
&\leq& \|(\hat{V}\hat{D}^{1/2}W^T-\hat{V}\hat{V}^T)ZZ^TW\hat{D}^{1/2}\hat{V}^T\| + \|\hat{V}\hat{V}^TZZ^T(W\hat{D}^{1/2}\hat{V}^T-\hat{V}\hat{V}^T)\| \nonumber\\
&\leq& \|(\hat{V}\hat{D}^{1/2}W^T-\hat{V}\hat{V}^T)ZY^T\hat{D}^{1/2}\hat{V}^T\| + \|ZZ^T\|\|W\hat{D}^{1/2}\hat{V}^T-\hat{V}\hat{V}^T\| \nonumber \\
&\leq&  \|\hat{V}\hat{D}^{1/2}W^T-\hat{V}\hat{V}^T\|\|Z\|\|\hat{D}^{1/2}\| + \|ZZ^T\|\|W\hat{D}^{1/2}\hat{V}^T-\hat{V}\hat{V}^T\| \nonumber \\
&\leq& \frac{8\sigma_1(A)^{1/2}\epsilon}{\sigma_{mk}(A)^{3/2}} \times 2\sigma_1(A) + \sigma_1(A) \times \frac{8\sigma_1(A)^{1/2}\epsilon}{\sigma_{mk}(A)^{3/2}} \nonumber \\
&=& 24\frac{\sigma_1(A)^{3/2} \epsilon}{\sigma_{mk}(A)^{3/2}} \nonumber
\end{eqnarray}

The second to last step follows from equation \ref{eq:ZZhat_1}. Now using the above bound in equation \ref{eq:subspace_zzhat_ub_1} we get,

\begin{eqnarray}
\|\hat{Z}\hat{Z}^T-\hat{V}\hat{V}^TZZ^T\hat{V}\hat{V}^T\|&\leq& \sigma_1(\hat{A})\|\hat{Y}\hat{Y}^T-YY^T\|+24\frac{\sigma_1(A)^{3/2} \epsilon}{\sigma_{mk}(A)^{3/2}}\nonumber\\
&\leq& \frac{8\sigma_1(A)\epsilon_1}{\sigma_{m}(R)-\sigma_{m+1}(R)}+24\frac{\sigma_1(A)^{3/2} \epsilon}{\sigma_{mk}(A)^{3/2}} \label{eq:ZZhat_2}
\end{eqnarray}
where the last step follows from inequalities \eqref{eq:ZZhat_0}. We compute the required bound by combining equations \eqref{eq:VVT_wedin} and \eqref{eq:ZZhat_2} as follows.

\begin{eqnarray*}
\|\hat{Z}\hat{Z}^T-ZZ^T\| &=& \|\hat{Z}\hat{Z}^T-VV^TZZ^TVV^T\| \\
&\leq& \|\hat{Z}\hat{Z}^T-\hat{V}\hat{V}^TZZ^T\hat{V}\hat{V}^T\|+3\|VV^T-\hat{V}\hat{V}^T\| \|ZZ^T\| \\
&\leq& \frac{8\sigma_1(A)\epsilon_1}{\sigma_{m}(R)-\sigma_{m+1}(R)}+24\frac{\sigma_1(A)^{3/2} \epsilon}{\sigma_{mk}(A)^{3/2}}  + \frac{12\sigma_1(A)\epsilon}{\sigma_{mk}(A)} \\
&\leq& C_1 \frac{\sigma_1(A)\sigma_1(B)\epsilon}{(\sigma_{m}(R)-\sigma_{m+1}(R))\sigma_{mk}(A)^2}
\end{eqnarray*}
where $C_1$ is a constant.
\end{proof}

\subsection{Proof of Theorem \ref{thm:subspace_perturbation}}

The proof follows from Theorem \ref{thm:subspace_strong} and Lemma \ref{lem:subspace_zzhat_ub}. Note that the matrix $Z$ has all singular values equal to $\sqrt{\alpha_1},$ therefore $ZZ^T$ has singular values $\alpha_1.$ Under the affinity condition from Theorem \ref{thm:subspace_strong}, we have
$$
\sigma_m(R)-\sigma_{m+1}(R) \geq 3\delta \|U_1 v\|^2
$$
Combining with Lemma \ref{lem:subspace_zzhat_ub} we get
$$
\|\hat{Z}\hat{Z}^T-ZZ^T\| \leq  \frac{C_2\sigma_1(A)\sigma_1(B)\epsilon}{\delta\|U_1 v\|^2\sigma_{mk}(A)^2}
$$
where $C_2$ is a constant. Finally applying Wedin's theorem for the matrices $\hat{Z}\hat{Z}^T$ and $ZZ^T$, we have
$$
\|\hat U \hat U^T - U_1 U_1^T\| \leq \frac{C_3\sigma_1(A)\sigma_1(B)\epsilon}{\alpha_1 \delta \|U_1 v\|^2\sigma_{mk}(A)^2} \leq \frac{C\sigma_1(A)^2\epsilon}{\alpha_1 \delta\sigma_{mk}(A)^2}
$$
where $C_3=4C_2.$

\section{Sample Complexity Analysis} \label{app:concentration}

Since the basic application of our method requires the estimation
of certain covariance matrices, we need to show that one can estimate these
matrices. There is a large literature on estimating covariance matrices, but
for simplicity we will only focus on the simplest estimator: the sample
covariance matrix. By well-known matrix concentration inequalities, one
can show that the sample covariance matrix will be close to the covariance
matrix with high probability if the sample size is large enough:

\begin{theorem}\cite{Tropp:15}
\label{thm:matrix-bernstein}
 Let $A_1, \dots, A_n$ be i.i.d.\ symmetric random $d \times d$
 matrices. If $\|A_1\| \le L$ a.s. then
 \[
  \Pr\left(
  \left\| \frac 1n \sum_{i=1}^n A_i - \E A_i\right\| \ge t
  \right)
  \le 8 d \exp\left(-\frac{n t^2}{L^2}\right).
 \]
\end{theorem}

\subsection{Truncation}

Unfortunately, the matrices we will be dealing with do not
usually have almost sure bounds on their norm. Here, we develop
some straightforward truncation arguments in order to adapt
Theorem~\ref{thm:matrix-bernstein}.

\begin{theorem}\label{thm:truncated}
  Suppose that $A_1, \dots, A_n$ are i.i.d.\ symmetric
  random $d \times d$ matrices satisfying the tail bound
  \[
    \Pr(\|A_1\| \ge t) \le C e^{- ct^\alpha}
  \]
  for some $\alpha > 0$. Then for any $\epsilon, \delta > 0$,
  if $n \ge \tilde \Omega_\alpha(\epsilon^{-2} \log(d/\delta))$ then
  \[
    \Pr(\|\hat \E A - \E A\| \ge \epsilon) \le \delta,
  \]
  where $\tilde \Omega_\alpha(k)$ means
  $C(\alpha) \Omega(k \log^{C(\alpha)} k)$.
\end{theorem}

\begin{proof}
  Fix $L > 0$ (to be determined later) and define the random matrix $B_i$
  by $B_i = A_i 1_{\{\|A_i\| \le L\}}$. Then
  Theorem~\ref{thm:matrix-bernstein} applies to $B_i$: if
  $n \ge \Omega(L^2\epsilon^{-2} \log (d/\delta))$ then
  \[
    \Pr(\|\hat \E B - \E B\| \ge \epsilon) \le \delta.
  \]
  To compare this with the similar quantity involving $A$,
  we will consider $\hat \E (A - B)$ and $\E (A - B)$ separately.

  First, note that
  $\Pr(A_i \ne B_i) = \Pr(\|A\| \ge L) \le C \exp(-cL^\alpha)$.
  If $L = \Omega(\log^{1/\alpha} (n/\delta))$ then
  $\Pr(A_i \ne B_i) \le \delta/n$. By a union bound,
  \begin{equation}\label{eq:hat-E-truncation}
    \Pr(\hat \E A \ne \hat \E B) \le \delta.
  \end{equation}

  Now we fix $L = C' \log^{1/\alpha} (n/(\delta \lor \epsilon))$ and we consider
  $\|\E (A - B)\|$. By the triangle inequality,
  \[
    \|\E (A - B)\| = \| \E A 1_{\{\|A\| \ge L\}} \|
    \le \E \|A\| 1_{\{\|A\| \ge L\}}.
  \]
  On the other hand, we can bound
  \[
    \E \|A\| 1_{\{\|A\| \ge L\}} = \int_L^\infty \Pr(\|A\| \ge t)\, dt
    \le C \int_L^\infty e^{-c t^\alpha} \, dt.
  \]
  With the change of variables $t = u^{1/\alpha}$, we have
  \[
    \E \|A\| 1_{\{\|A\| \ge L\}} \le \frac{1}{\alpha} \int_{L^\alpha}^\infty
    u^{1/\alpha} e^{-cu}\, du.
  \]
  Now, if $u \ge C'' \frac 1\alpha \log \frac 1\alpha$ for large enough
  $C''$ then $u^{1/\alpha} e^{-cu} \le e^{-cu/2}$. Hence,
  if $L^\alpha \ge C'' \frac 1\alpha \log \frac 1\alpha$ then
  \[
    \E \|A\| 1_{\{\|A\| \ge L\}}
    \le \frac{1}{\alpha} \int_{L^\alpha}^\infty e^{-cu/2}\, du \le C(\alpha) e^{-c L^\alpha/2}
    \le C(\alpha) \epsilon
  \]
  where the last inequality holds if the constant $C'$ in the definition
  of $L$ is large enough compared to $c$. On the other
  hand, if $L^\alpha < C'' \frac 1\alpha \log \frac 1\alpha$ then
  we must have $\epsilon > c(\alpha)$ for some $c(\alpha) > 0$.
  In this case, $\E \|A\| 1_{\{\|A\| \ge L\}} \le C \le C(\alpha) \epsilon$
  trivially. To summarize, in every case we have
  \[
    \|\E (A - B)\| \le C(\alpha) \epsilon.
  \]
  Putting this together with~\eqref{eq:hat-E-truncation}, we have
  that
  if $n \ge \Omega(L^2 \epsilon^{-2} \log(d/\delta))$ then
  with probability at least $1-2\delta$,
  \begin{eqnarray*}
    \|\hat \E A - \E A\|
    &\le& \|\hat \E B - \E B\| + \|\hat \E A - \hat \E B\|
    + \|\E A - \E B\| \\
&\le& (1 + C(\alpha)) \epsilon.
  \end{eqnarray*}
  Finally, recalling that $L = \polylog(n, 1/\epsilon, 1/\delta)$
  (with the polynomial depending on $\alpha$), we see that
  $n = \tilde \Omega_\alpha(\epsilon^{-2} \log(d/\delta))$ suffices.
  Finally, we can absorb the constant $C(\alpha)$ into $\epsilon$.
\end{proof}

We will now show how Theorem~\ref{thm:truncated} bounds
the error in estimating the various matrices that we had
to estimate for the various different models we considered.
Essentially, we will repeatedly use the observation that
if $z$ is a standard Gaussian variable then $z^{2/\alpha}$
has a tail that decays like $e^{-c t^\alpha}$. In other words,
moments of Gaussians will naturally lead to a condition
that the one assumed in Theorem~\ref{thm:truncated}.

\subsection{Gaussian Mixture Model}

For the following theorem, we revert to the notation of the
Gaussian mixture model.

\begin{theorem}
  Fix $\epsilon, \delta > 0$. Let $\hat A = \hat \E [x x^T]$
  and $\hat B = \hat \E[\inr{x}{v} x x^T]$, where $\hat \E$ is taken
with $n$ i.i.d. samples. If $n \ge \tilde \Omega(d \epsilon^{-2} \log(d/\delta))$ then with probability at least $1-\delta$,
$\|\hat \E A - \E A\| \le \epsilon$ and $\|\hat \E B - \E B\| \le \epsilon$.
\end{theorem}

\begin{proof}
To estimate $A$, first note that $\|x x^T\| = \|x\|^2$.
Now, $\E \|x\|^2 \le R^2 + d \sigma^2$, where $R = \max_i \|\mu_i\|$,
and also $\Pr(\|x\|^2 \ge \E \|x\|^2 + t\sqrt{d}) \le C e^{-ct}$. Hence,
we may apply Theorem~\ref{thm:truncated} with $A_i = x_i x_i^T / \sqrt{d}$
and $\alpha = 1$; this yields the claimed bound on
$\|\hat \E A - \E A\|$.

To estimate $B$, note that $\|\inr{x}{v}^2 x x^T\| = \inr{x}{v}^2 \|x\|^2$.
Now, the triangle inequailty implies that
$\inr{x}{v}^2 \|x\|^2$ is stochastically dominated by
\[
  4R^4 + 4 \E [\inr{z}{v}^2 \|z\|^2]
  = 4 R^4 + 4 \E [z_1^2 \|z\|^2],
\]
where $z$ is a standard (i.e., centered) Gaussian vector.
Then $\E[ z_1^2 \|z\|^2] = 2 + d$, and $z_1^2 \|z\|^2$ has tails
of order $e^{-ct^{1/2}}$; that is it satisfies the assumptions
of Theorem~\ref{thm:truncated} with $\alpha=1/2$. Applying
Theorem~\ref{thm:truncated} with $A_i = \inr{x_i}{v}^2 x_i x_i^T/\sqrt{d}$
then yields the claimed bound on
$\|\hat \E B - \E B\|$.
\end{proof}

\subsection{LDA Topic Model}

For the following theorem, we revert to the notation of the
LDA topic model, where $d$ is the size of the dictionary.

\begin{theorem}
  Fix $\epsilon, \delta > 0$. Let $\hat A = \hat \E [x_1 x_2^T]$
  and $\hat B = \hat E [ \inr{x_3}{v} x_1 x_2^T ]$, where $\hat \E$
  is taken with $n$ i.i.d. samples. If $n \ge \Omega(\epsilon^{-2} \log(d/\delta))$ then with probability at least $1-\delta$,
  $\|\hat A - \E A\| \le \epsilon$ and $\|\hat B - \E B\| \le \epsilon$.
\end{theorem}

\begin{proof}
  We can apply Theorem~\ref{thm:matrix-bernstein} directly, since
  $\|x_1 x_2^T\| \le 1$ and $\inr{x_3}{v} x_1 x_2^T \le 1$.
\end{proof}

\subsection{Mixed Regression}

For the following theorem, we revert to the notation of the mixed
regression model.

\begin{theorem}
  Fix $\epsilon, \delta > 0$. Let $\hat A = \hat \E [y^2 x x^T]$
  and $\hat B = \hat \E [ y^3 \inr{x}{v} x x^T ]$, where $\hat \E$
  is taken with $n$ i.i.d. samples.
  Let $R = \max_i \|\mu_i\|$.
  If $n \ge \tilde \Omega((R^2 + \sigma^2) \epsilon^{-2} d \log(d/\delta))$ then with probability at least $1-\delta$,
  $\|\hat A - \E A\| \le \epsilon$ and $\|\hat B - \E B\| \le \epsilon$.
\end{theorem}

\begin{proof}
  Recalling that in cluster $i$ we have $y = \inr{x}{\mu_i} + \xi$, we
  have
  \[
    \|y^2 x x^T\| \le 2 \inr{x}{\mu_i}^2 \|x\|^2 + 2 \xi^2 \|x\|^2.
  \]
  Hence,
  $\E \|y^2 x x^T\| \le 2R^2 (2+d) + \sigma^2 d$,
  with tails that decay at the rate $e^{-ct^{1/2}}$. Applying
  Theorem~\ref{thm:truncated} implies the claimed bounds for $A$.
  The case of $B$ is analogous, except that since it involves sixth moments
  the tails will decay at the rate $e^{-c t^{1/3}}$; this only effects
  the polylogarithmic terms hidden in the $\tilde \Omega$ notation.
\end{proof}

\subsection{Subspace Clustering}

For the following theorem, we revert to the notation
of the subspace clustering model. We assume for simplicity that $\sigma$ is known, since if it isn't then it can
be easily and accurately learnt.

\begin{theorem}
  Fix $\epsilon, \delta > 0$. Let $\hat A = \hat \E[x x^T] - \sigma^2 I_d$
  and
  \[
    \hat B = \hat \E[\langle x,v\rangle^2 x x^T] -\sigma^2 (v^T \hat A v) I_d - \sigma^2 \|v\|^2 \hat A - \sigma^4(\|v\|^2 I_d + v v^T) - 2 \sigma^2 (\hat A vv^T+vv^T \hat A)
  \]
  where $\hat \E$
  is taken with respect to $n$ i.i.d. samples. If
  $n \ge \tilde \Omega(\epsilon^{-2} (1 + \sigma^2) \|v\|^2 m \log(d/\delta))$
  then with probability at least $1-\delta$,
  $\|\hat A - A\| \le \epsilon$ and $\|\hat B - B\| \le \epsilon$.
\end{theorem}

\begin{proof}
  Since $x/\sigma$ is an $m$-dimensional Gaussian vector,
  $\|x\|^2 / (\sigma^2 m)$ is concentrated around its mean ($1$) with
  tails of order $e^{-c t}$. In other words, Theorem~\ref{thm:truncated}
  (with $\alpha = 1$) implies our claim for $A$.
  The claim for $B$ is analogous, except that since it involves
  fourth moments, the tails will decay at the rate $e^{-c t^{1/2}}$.
\end{proof}

\end{appendices}

\end{document}